\documentclass{article}

\usepackage{ifthen}
\usepackage{fancyvrb}
\usepackage[dvipsnames]{xcolor}
\usepackage[most]{tcolorbox}
\usepackage{microtype}
\usepackage{graphicx}
\usepackage{subcaption}
\usepackage{booktabs}
\usepackage{longtable}
\usepackage{hyperref}
\usepackage{enumitem}
\usepackage{minted}
\usepackage{comment}

\definecolor{sectioncolor}{HTML}{19519a} 
\definecolor{lightgreen}{HTML}{90ee90}
\definecolor{lightorange}{HTML}{ffd580}
\definecolor{lightblue}{HTML}{ADD8E6}
\definecolor{darkcyan}{rgb}{0.0, 0.55, 0.55}
\definecolor{theoremcolor}{HTML}{dae2e6}
\definecolor{defcolor}{HTML}{D5EDF5}
\definecolor{deftitlecolor}{HTML}{912e32}
\definecolor{darkblue}{HTML}{00007F}
\definecolor{linkcolor}{HTML}{6B767D}
\hypersetup{colorlinks=true,linkcolor=linkcolor,urlcolor=linkcolor,citecolor=linkcolor}

\definecolor{light-gray}{HTML}{F5F5F5}
\setminted{breaklines,bgcolor=light-gray,fontsize=\small}
\setmintedinline{bgcolor={},fontsize=\normalsize}

\newtcolorbox{roundnote}[1][]{
  enhanced,
  breakable,
  colback=sectioncolor!5,
  colframe=sectioncolor,
  boxrule=0.5pt,
  arc=3mm,
  left=6pt,right=6pt,top=6pt,bottom=6pt,
  before upper=\setlength{\parskip}{\medskipamount},
  #1
}

\newboolean{isInternal}
\setboolean{isInternal}{false}

\newboolean{isOnBorisLaptop}
\IfFileExists{~/boris_machine_marker}{
  \setboolean{isOnBorisLaptop}{true}
}{%
  \setboolean{isOnBorisLaptop}{false}
}

\ifthenelse{\boolean{isInternal}}{

  \newcommand{\internalComment}[1]{\textbf{\color{red}[}{\footnotesize#1}\textbf{\color{red}]}}
}{
  \newcommand{\internalComment}[1]{}
}

\usepackage{mathstyle}

\usepackage[numbers]{natbib}
\DeclareMathOperator{\trace}{tr}

\ExplSyntaxOn
\DeclareExpandableDocumentCommand{\IfNotEmptyT}{mm}%
{%
  \tl_if_empty:nTF{#1}{}{#2}%
}
\ExplSyntaxOff

\definecolor{mplblue}{HTML}{1f77b4}
\definecolor{mplorange}{HTML}{ff7f0e}
\definecolor{mplgreen}{HTML}{2ca02c}
\definecolor{mplred}{HTML}{d62728}
\definecolor{mplpurple}{HTML}{9467bd}
\definecolor{mplbrown}{HTML}{8c564b}
\definecolor{mplpink}{HTML}{e377c2}
\definecolor{mplgray}{HTML}{7f7f7f}
\definecolor{mplolive}{HTML}{bcbd22}
\definecolor{mplcyan}{HTML}{17becf}

\newcommand{\vg}{\boldsymbol{g}}

\newcommand{\vF}{\boldsymbol{F}}

\newcommand{\btheta}{\boldsymbol{\theta}}

\newcommand{\vtheta}{\boldsymbol{\theta}}

\newcommand{\mSigma}{\boldsymbol{\Sigma}}

\newcommand\MTkillspecial[1]{
  \bgroup
  \catcode`\&=9
  \let\\\relax%
  \scantokens{#1}%
  \egroup
}
\newcommand{\DeclareCustomDelim}[3]{
  \DeclarePairedDelimiter{#1}{#2}{#3}
  \reDeclarePairedDelimiterInnerWrapper{#1}{star}{
    \mathopen{##1\vphantom{\MTkillspecial{##2}}\kern-\nulldelimiterspace\right.}
  ##2
  \mathclose{\left.\kern-\nulldelimiterspace\vphantom{\MTkillspecial{##2}}##3}
  }
}

\DeclareCustomDelim{\prn}{\lparen}{\rparen}
\DeclareCustomDelim{\crl}{\{}{\}}
\DeclareCustomDelim{\brk}{[}{]}
\DeclareCustomDelim{\norm}{\|}{\|}
\DeclareCustomDelim{\abs}{|}{|}

\NewDocumentCommand{\comp}{o m m}{{\color{blue} #1[} #2 {\color{blue}#1]_{#3}}}

\DeclarePairedDelimiterXPP\Prob[1]{\Problet}\{\}{}{
  
  #1}
\DeclarePairedDelimiterXPP\E[1]{\Elet}[]{}{
  
  #1}
\DeclarePairedDelimiterXPP\Epi[1]{\Elet_{\pi}}[]{}{
  
  #1}

\newcounter{notationcnt}


\newcommand{\protectedLink}[2]{
  {
    \hypersetup{hidelinks}
    \protect\hyperref[#1]{#2}
  }
}

\newcommand*\makeAlph[1]{\symbol{\numexpr96+#1}}
\NewDocumentCommand{\labrel}{m o }{%
  \IfNoValueTF{#2}{(\makeAlph{#1})}{\stackrel{\mathrm{(\makeAlph{#1})}}{#2}}%
}

\newcommand{\subarrsymbol}{:}
\newcommand{\subarr}{\nolinebreak\mathinner{\subarrsymbol}\nolinebreak}
\NewDocumentCommand{\range}{omm}{\brk[#1]{#2 \subarr #3}}

\newcommand{\Problet}{\ensuremath\mathbb{P}}
\newcommand{\Elet}{\ensuremath\mathbb{E}}

\newcommand{\Eletpi}{\ensuremath\Elet_{\pi}}

\DeclareMathOperator{\sign}{sign}

\newcommand{\trans}{^{\mkern-1.5mu\mathsf{T}}}  

\NewDocumentCommand{\gbar}{m m}{\bar{\boldsymbol{g}}_{#1}\IfNotEmptyT{#2}{^{(#2)}}}%

\NewDocumentCommand{\hess}{m m}{\partial_{#1} g_{#2}}
\NewDocumentCommand{\noise}{o m}{\partial_{#1}d_{#2}}
\NewDocumentCommand{\fb}{m}{\text{FB}_{#1}}
\NewDocumentCommand{\fbneps}{m}{\text{FB}^{(n,\epsilon)}_{#1}}
\NewDocumentCommand{\fbinfeps}{m}{\text{FB}^{(\infty,\epsilon)}_{#1}}
\NewDocumentCommand{\mbnone}{m}{\text{MBN}_{1, #1}}
\NewDocumentCommand{\mbnoneneps}{m}{\text{MBN}^{(n,\epsilon)}_{1,#1}}
\NewDocumentCommand{\mbnoneinfeps}{m}{\text{MBN}^{(\infty,\epsilon)}_{1,#1}}
\NewDocumentCommand{\mbntwo}{m}{\text{MBN}_{2, #1}}
\NewDocumentCommand{\mbntwoneps}{m}{\text{MBN}^{(n,\epsilon)}_{2,#1}}
\NewDocumentCommand{\mbntwoinfeps}{m}{\text{MBN}^{(\infty,\epsilon)}_{2,#1}}
\NewDocumentCommand{\mbnthree}{m}{\text{MBN}_{3, #1}}
\NewDocumentCommand{\mbnthreeneps}{m}{\text{MBN}^{(n,\epsilon)}_{3,#1}}
\NewDocumentCommand{\mbnthreeinfeps}{m}{\text{MBN}^{(\infty,\epsilon)}_{3,#1}}
\NewDocumentCommand{\mbnfour}{m}{\text{MBN}_{4, #1}}
\NewDocumentCommand{\mbnfourneps}{m}{\text{MBN}^{(n,\epsilon)}_{4,#1}}
\NewDocumentCommand{\mbnfourinfeps}{m}{\text{MBN}^{(\infty,\epsilon)}_{4,#1}}
\NewDocumentCommand{\mbnfive}{m}{\text{MBN}_{5, #1}}
\NewDocumentCommand{\mbnfiveneps}{m}{\text{MBN}^{(n,\epsilon)}_{5,#1}}
\NewDocumentCommand{\mbnfiveinfeps}{m}{\text{MBN}^{(\infty,\epsilon)}_{5,#1}}
\NewDocumentCommand{\mbnr}{m}{\text{MBN}_{r, #1}}
\NewDocumentCommand{\mbnrneps}{m}{\text{MBN}^{(n,\epsilon)}_{r,#1}}
\NewDocumentCommand{\mbnrinfeps}{m}{\text{MBN}^{(\infty,\epsilon)}_{r,#1}}

\newcommand\numberthis{\addtocounter{equation}{1}\tag{\theequation}}

\newcounter{ccnt}
\NewDocumentCommand{\nc}{m}{
  \refstepcounter{ccnt}\ensuremath{c_{\theccnt}}\IfValueT{#1}{\label{#1}}%
}

\newcounter{bigccnt}
\NewDocumentCommand{\nC}{m}{%
  \refstepcounter{bigccnt}\ensuremath{C_{\thebigccnt}}\IfValueT{#1}{\label{#1}}%
}

\NewDocumentCommand{\corr}{m}{\text{\textbf{Corr}}_{#1}}
\NewDocumentCommand{\corrsc}{m m}{\text{Corr}_{#2, #1}}

\NewDocumentCommand{\main}{m}{\text{\textbf{Main}}_{#1}}
\NewDocumentCommand{\mainsc}{m m}{\text{Main}_{#2, #1}}

\newlist{experiments}{enumerate}{1}
\setlist[experiments,1]{label=\arabic*., ref=\arabic*}

\newlist{conditions}{enumerate}{1}
\setlist[conditions]{label=(\roman*),nosep}
\crefname{conditionsi}{Condition}{Conditions}

\newlist{assertions}{enumerate}{1}
\setlist[assertions]{label=(\alph*),nosep}
\crefname{assertionsi}{Assertion}{Assertions}

\crefname{experimentsi}{exp}{exp}

\graphicspath{{pics/}}

\newcommand{\papertitle}{The Effect of Mini-Batch Noise on the Implicit Bias of Adam\internalComment{INTERNAL}}

\title{\papertitle}

\author{
  Matias D. Cattaneo\thanks{Authors are listed alphabetically by last name.} \\
  Princeton University\\
  \texttt{cattaneo@princeton.edu} \\
  \and
  Boris Shigida\footnotemark[1] \\
  Princeton University\\
  \texttt{bs1624@princeton.edu} \\
}

\allowdisplaybreaks[4]

\NewDocumentCommand{\py}{+v}{}

\newsavebox{\pycodebox}      

\NewDocumentEnvironment{pycode}{}{%
  \VerbatimEnvironment%
  \begin{lrbox}{\pycodebox}%
  \begin{minipage}{\linewidth}%
  \begin{Verbatim}%
}{
  \end{Verbatim}%
  \end{minipage}%
  \end{lrbox}%
}

\NewDocumentEnvironment{pycontext}{} {
  \VerbatimEnvironment
  \begin{lrbox}{\pycodebox}%
  \begin{minipage}{\linewidth}%
  \begin{Verbatim}
}{
  \end{Verbatim}
  \end{minipage}
  \end{lrbox}
}

\NewDocumentEnvironment{nothing}{}{}{}

\newcommand{\emphindef}[1]{\textit{#1}}

\begin{document}

\begin{pycontext}
import os
import sys
os.chdir("../adam_hyper")
sys.path.insert(0, os.getcwd())
from report import *
\end{pycontext}

\maketitle

\begin{abstract}
Adam and AdamW are standard optimizers in deep learning, but the choice of their momentum
hyperparameters $(\beta_1,\beta_2)$ is often not principled.
We study how the interaction of these hyperparameters with batch size implicitly affects loss sharpness, relevant both to overfitting concerns in multi-epoch training and to post-training performance concerns after single-epoch pretraining.
Our analysis predicts a batch-size-dependent reversal in the preferred choice of
Adam's betas. In the low-noise, large-batch regime, increasing $\beta_2$
strengthens the implicit anti-regularization induced by memory, suggesting that
choosing $\beta_1$ close to $\beta_2$ should improve generalization and model sensitivity.
In contrast,
in the high-noise, small-batch regime, mini-batch noise reverses this
monotonicity: larger $\beta_2$ can compensate for the anti-regularizing memory
term, making the default choice $\beta_1 \ll \beta_2$ theoretically well
motivated. The predicted transition occurs at a batch-size scale controlled by
the simple noise scale, closely related to the critical batch size. Experiments
both in an overfitting regime and with single-epoch pretraining support these
qualitative predictions.
\end{abstract}

\section{Introduction}

Adam \citep{kingma2014adam} and its variants, including AdamW
\citep{loshchilov2019decoupledweightdecayregularization} and AdaFactor
\citep{shazeer2018adafactor}, are standard optimizers for modern deep learning
tasks such as language-model training
\citep{brown2020fewshot,anil2023palm,touvron2023llama,dubey2024llama}.
Beyond the learning rate, training with these methods involves two consequential
choices: the batch size $b$ and Adam's momentum hyperparameters
$(\beta_1,\beta_2)$.
These quantities jointly determine the optimizer's memory and the amount of
mini-batch noise seen during training, yet their values are still largely set by
convention.

The often appearing default traces back to \citet{kingma2014adam}, who recommend
$(\beta_1,\beta_2)=(0.9,0.999)$ based on a grid search. These values remain the
defaults in widely used libraries such as
\href{https://pytorch.org/docs/stable/generated/torch.optim.Adam.html}{PyTorch}
and
\href{https://optax.readthedocs.io/en/latest/api/optimizers.html#optax.adam}{Optax},
and it is conventional wisdom that adaptive gradient methods work well with
their default hyperparameters \citep{sivaprasad2020optimizer}. At the same
time, recent large-model training practice often lowers $\beta_2$, for example
to $\beta_2=0.95$, when using AdamW
\citep{brown2020fewshot,zhang2022opt,zeng2022glm,pmlr-v202-biderman23a,touvron2023llama,dubey2024llama},
partly because this can improve training stability. These observations raise a
basic question: should the preferred relationship between $\beta_1$, $\beta_2$,
and batch size change across noise regimes?

The answer depends on the performance criterion. Large pretraining runs have
often used only one pass, or less, over available data
\citep{NEURIPS2020_1457c0d6}, where stability, optimization speed, and compute
efficiency are primary concerns. In contrast, limited high-quality data can make
multi-epoch training increasingly useful
\citep{villalobos2024position,kim2025pretraininginfinitecompute}, and parts of
the post-training pipeline are explicitly multi-epoch and prone to overfitting
\citep{xiao2025rethinkingconventionalwisdommachine}. In such regimes, including
pretraining under data constraints \citep{kim2025pretraininginfinitecompute},
generalization quantities such as the train-validation gap are of interest.

Loss sharpness can be considered a mechanistic proxy for this type of generalization:
there is a large body of work connecting flatter regions of the loss landscape
to better generalization and showing the benefits of explicit or implicit sharpness
penalization
\citep{Jiang2020Fantastic,foret2021sharpnessaware,pmlr-v139-kwon21b,zheng2021regularizing,pmlr-v162-kim22f,NEURIPS2022_948b1c9d,liu2022towardsefficient,li2023enhancing,xie2024sampa,tahmasebi2024a,li2024friendlysharpness,barrett2021implicit,smith2021on,ghosh2023implicit,rosca2023oncontin,pmlr-v235-cattaneo24a}. (We provide some discussion of the evidence in \cref{sec:related-work}.)
Importantly, recent findings also suggest that sharpness during pretraining, even
single-epoch, is closely related to downstream factors such as
quantization degradation and catastrophic forgetting
\citep{watts2026sharpnessaware,springer2025overtrained}.
This motivates asking
how Adam's memory and mini-batch noise implicitly bias training toward or away
from sharp regions of the loss landscape, irrespective of whether classical overfitting is a concern.

We theoretically investigate how $(\beta_1,\beta_2)$ and batch size $b$ affect
this implicit sharpness bias. Our analysis extends the framework of the
implicit bias of memory \citep{cattaneo2025howmemoryoptimization} to isolate
mini-batch noise effects in Adam. To our knowledge, this is the first theory
explaining the batch-size-dependent reversal in Adam beta preferences through an
implicit sharpness-bias mechanism.

Our main contributions are as follows.
\begin{enumerate}
\item In \cref{sec:overview}, we present a framework for finding and
interpreting implicit bias terms during mini-batch training with an optimizer
that has memory, meaning that the next iterate depends on the history of
previous gradients. The exposition uses SGD with momentum as an illustrative
example.

\item In \cref{sec:theory-on-adam}, we apply this approach to mini-batch Adam.
After removing memory and averaging over without-replacement mini-batch
sampling, we derive an implicit-bias correction whose dominant terms can be
interpreted through a sharpness proxy.

\item We show that mini-batch noise changes the monotonicity of the sharpness
bias as a function of $\beta_2$ when $\beta_1$ is fixed. In the high-noise
small-batch regime, larger $\beta_2$ can compensate for the anti-regularizing
memory term, making the default ordering $\beta_1 \ll \beta_2$ suitable when
overfitting or model sensitivity is a concern. In the low-noise large-batch regime, larger $\beta_2$
instead strengthens anti-regularization, suggesting that $\beta_1$ and
$\beta_2$ should be chosen close to each other. As reviewed in
\cref{sec:related-work}, this prescription is consistent with a substantial
body of empirical work.

\item We show that an analogous reversal appears when $\beta_1$ is swept with
$\beta_2$ fixed at a common value such as $0.999$. For large batches and
full-batch training, larger $\beta_1$ is predicted to be better, consistent with
\citet{pmlr-v235-cattaneo24a}; as mini-batch noise increases, this monotonicity
reverses and again suggests $\beta_1 \ll \beta_2$ in small-batch regimes.

\item The transition scale in our simplified theory is controlled by the simple
noise scale $\mathcal{B}_{\text{simple}}$, closely related to the critical batch
size. In \cref{sec:experiments}, small-scale language-model experiments in an
overfitting regime along with a larger online pretraining run
support these qualitative predictions.
\end{enumerate}

\subsection{Related Work}\label{sec:related-work}
\paragraph{Tuning hyperparameters of Adam}
\Citet{ma2022qualitativestudydynamic} investigate
theoretically and empirically
the qualitative features of full-batch Adam depending on $(\beta_1, \beta_2)$,
dividing possible training into three regimes (oscillations, spikes and divergence) and advocating for $\beta_1 = \beta_2$ for faster and smoother training. The latter prescription is
consistent with our theory although we focus on different metrics (loss landscape sharpness / flatness), and we argue that increasing mini-batch noise changes the conclusions and recommendations.
Relatedly, \citet{zhao2025deconstructing} include Adam's $(\beta_1, \beta_2)$ sweeps
and find that if $\beta_1 = \beta_2$, Adam behaves similarly to signed momentum (Signum), and the recently common setting for language models $(\beta_1, \beta_2) = (0.9, 0.95)$ is close to this.
For small batches, however,
it is empirically observed to be beneficial to increase
$\beta_2$ relative to $\beta_1$.
In particular,
\citet{zhang2025how}
recommend (based on empirical sweeps)
taking smaller $\beta_2$ relative to $\beta_1$
if batch sizes are large and higher when batch sizes are small,
exactly matching our theory-based prescription.
There are other prior works advocating for that, and they
use different principles \citep{porian2024resolving,marek2025small}.
To the best of our knowledge, we provide the first theoretical
argument based on generalization.
Other works that have a substantial focus on $(\beta_1, \beta_2)$ sweeps in Adam include \citet{schmidt2021descendingcrowdedvalley,orvieto2025in,wen2025fantasticpretrainingoptimizers,pagliardini2025the}.

\paragraph{SDE approximations}
In the context of mini-batch noise,
theoretical analysis of many optimizers often employs approximations by stochastic differential equations (SDEs).
Relatedly, \citet{jules2023charting} use controlled thermal-like Langevin noise
as a diagnostic probe of the low-loss landscape geometry of neural networks.
In contrast, our work studies the implicit bias induced by the endogenous
mini-batch noise and memory terms of Adam during training, and connects this
effect to batch-size-dependent choices of $(\beta_1,\beta_2)$.
In particular, \citet{NEURIPS2020_f3f27a32,xie2022adaptiveinertiadisentangling}
approximate Adam and SGD with (different types of)
SDEs, and use the escaping time from local minima to predict
better generalization of SGD compared to Adam.
The works \citet{malladi2022sdesscalingrules,compagnoni2025adaptive} also approximate Adam with SDEs under different assumptions and propose scaling rules for hyperparameters.
\Citet{zhou2024towardsunderstanding} focus on the advantages of decoupled weight decay for generalization. These works differ substantially from the present article in purpose, methods and assumptions; in particular, typically $\beta_1$ and $\beta_2$ are assumed to converge to $1$ at certain rates as step size goes to zero, whereas we consider them fixed. In addition, we do not assume that mini-batch noise in the gradients forms an i.\,i.\,d. random sequence
(since we consider sampling without replacement), we are agnostic to its distribution,
and we do not use distributional asymptotics.

\paragraph{Sharpness and generalization}
There has been a long history of relating
flatter minima or loss regions to better generalization
\citep{NIPS1994_01882513,keskar2017on,Jiang2020Fantastic}.
There has also been some criticism based on the sensitivity of standard sharpness measures
to rescaling the network's parameters even if it does not change the network's outputs
\citep{pmlr-v70-dinh17b}, which \citet{pmlr-v139-kwon21b} call the \textit{scale dependency problem}.
In response, different scale-invariant sharpness metrics have been introduced
\citep{yi2019positivelyscaleinvariant,pmlr-v119-tsuzuku20a,rangamani2021ascaleinvariant,pmlr-v139-kwon21b}; however, empirical evidence is still mixed \citep{pmlr-v202-andriushchenko23a}.
Numerous works have explored explicit sharpness penalization to improve generalization,
of which we can only name a few
\citep{foret2021sharpnessaware,pmlr-v139-kwon21b,zheng2021regularizing,pmlr-v162-kim22f,NEURIPS2022_948b1c9d,liu2022towardsefficient,li2023enhancing,xie2024sampa,tahmasebi2024a,li2024friendlysharpness}. We study implicit, rather than explicit, penalization but otherwise our theory-based perspective is consistent with this literature.
Although the non-adaptive sharpness metrics we find implicitly \mbox{(anti-)penalized} do have the scale dependency problem (along with most metrics in the related literature including Hessian-based ones), this does not invalidate any conclusions since penalizing or otherwise decreasing non-adaptive sharpness still often leads to better generalization.

\paragraph{Implicit bias}
A large strand of literature describes implicit biases of optimization algorithms
by proving convergence to a max-margin solution
\citep{soudry2018implicit,nacson2019convergence,nacson2019stochastic,qian2019implicit,wang2022does,gunasekar2018characterizing,ji2018risk,ji2019implicit,ji2018gradient,gunasekar2018implicit,ji2020directional,nacson2019lexicographic,lyu2019gradient,pmlr-v139-wang21q}.
The implicit bias of weight decay in AdamW is tackled in
\citet{zhang2018three,zhuang2022understanding,andriushchenko2024why,pmlr-v235-xie24e,kobayashi2024weight} and others.
Implicit regularization
by biasing towards flatter minima at convergence is studied
in \citet{NEURIPS2021_e6af401c,pmlr-v162-arora22a} besides works already listed.
Most relatedly to our work,
a large body of literature demonstrates
implicit penalization of a gradient norm,
using modified equations
for SGD with or without momentum
\citep{barrett2021implicit,miyagawa2022toward,smith2021on,farazmand2020multiscale,kovachki2021continuous,ghosh2023implicit,rosca2023oncontin},
for full-batch Adam \citep{pmlr-v235-cattaneo24a},
or using correction terms after removing memory
\citep{cattaneo2025howmemoryoptimization}. \Citet{beneventano2023trajectories} studies the difference between SGD with or without replacement with similar methods. We build on and extend this literature.


\section{Framework}\label{sec:overview}

This section introduces the framework used in the rest of the paper.
We first specify the without-replacement mini-batch sampling model and the
corresponding gradient-noise covariance. We then recall the memory-removal
principle, which replaces an optimizer with memory by a memoryless iteration
with an explicit correction term. Finally, we illustrate the method on SGD with
momentum before applying it to Adam in \cref{sec:theory-on-adam}.

\paragraph{Losses and gradients}
We assume there are $n + 1$ batches in an epoch, each consisting of $b$
samples, so that $N := (n + 1) b$ samples are used in total. The $k$th
\emphindef{mini-batch loss} is defined by
\begin{equation*}
\mathcal{L}_k(\btheta) = \frac{1}{b} \sum_{r = k b + 1}^{k b + b} \ell_{\pi(r)}(\btheta),\quad k \in \range{0}{n},
\end{equation*}
where $\crl{\ell_s}_{s = 1}^N$ are \emphindef{per-sample losses} and
$\pi\colon \range{1}{N} \to \range{1}{N}$ is a uniformly random permutation of
the samples. The vector $\vtheta \in \Theta$ collects all model parameters, and
$\Theta \subset \mathbb{R}^{\dim \vtheta}$ is the parameter domain of interest.
The \emphindef{full-batch loss} is the average of mini-batch losses:
\begin{equation*}
\mathcal{L}(\btheta) = \frac{1}{n + 1} \sum_{k = 0}^n \mathcal{L}_k(\btheta) = \frac{1}{N} \sum_{r = 1}^N \ell_r(\btheta).
\end{equation*}
We write the \emphindef{loss gradient} as
\begin{equation*}
\mathbb{R}^{\dim \vtheta} \ni \vg = (g_1, \ldots, g_{\dim \vtheta})\trans := \nabla \mathcal{L}(\vtheta), \quad g_i := \partial_i \mathcal{L}(\btheta).
\end{equation*}
As usual, we omit the dependence on $\btheta$ when the point is fixed and clear
from context.

\paragraph{Mini-batch noise, empirical covariance matrix}
We will denote for $k \in \range{0}{n}$
\begin{gather*}
  d_k := (\mathcal{L}_k - \mathcal{L})(\btheta)
\end{gather*}
the $k$th \emphindef{mini-batch noise}. Since $d_k$ is a function of
$\btheta$, its derivatives $\partial_i d_k$ and $\partial_{ij} d_k$ are the
corresponding gradient and Hessian noise. Further, we define the
\emphindef{empirical covariance matrix} $\mSigma$ of per-sample gradients:
\begin{equation*}
\Sigma_{i j} := \frac{1}{(n + 1) b} \sum_{p = 1}^{(n + 1) b} \partial_i (\ell_p - \mathcal{L})(\btheta) \partial_j (\ell_p - \mathcal{L}) (\btheta), \quad \mSigma = (\Sigma_{i j})_{i, j = 1}^{\dim \vtheta} \in \mathbb{R}^{\dim \vtheta \times \dim \vtheta}.
\end{equation*}

\paragraph{Memory Removal}
An optimization algorithm has \emphindef{memory} if its update depends on the
history of previous iterates, not only on the current iterate. Consider a
general iteration of this form:
\begin{equation}\label{eq:quzxvB}
\vtheta_{t + 1} = \underbracket{\vtheta_t - \eta \vF_t(\vtheta_t, \ldots, \vtheta_0).}_{\text{depends on the whole history $\vtheta_t, \ldots, \vtheta_0$}}
\end{equation}
The memory-removal result of \citet{cattaneo2025howmemoryoptimization}
converts it into a memoryless iteration
\begin{equation}\label{eq:HsXmgC}
  \tilde{\vtheta}_{t + 1} = \underbracket{\tilde{\vtheta}_t
    - \eta \, \main{t}(\tilde{\vtheta}_t) - \eta^2 \, \corr{t}(\tilde{\vtheta}_t)}_{\text{only depends on $\tilde{\vtheta}_t$ (no memory)}},
\end{equation}
where $\main{t}$ is the update obtained by freezing the whole history at one
point and $\corr{t}$ is the correction that compensates for this replacement.
The original and memoryless trajectories stay globally $O(\eta^2)$-close for
$O(\eta^{-1})$ iterations: for any ``physical time'' horizon $T > 0$, there is
a constant $C$ such that
\begin{equation}\label{eq:yqEMvw}
\max_{t \in \range{0}{\lfloor T / \eta \rfloor}} \norm[\big]{\vtheta_t - \tilde{\vtheta}_t}_{\infty} \leq C \eta^2,
\end{equation}
provided that the \emphindef{main term} $ \main{t}(\vtheta) \in \mathbb{R}^{\dim \vtheta}$ and the \emphindef{correction term} $\corr{t}(\vtheta) \in \mathbb{R}^{\dim \vtheta}$ are chosen as
\begin{equation}\label{eq:tetoJR}
    \main{t}(\vtheta) := \vF_t(\vtheta, \ldots, \vtheta), \quad \corrsc{r}{t}(\vtheta) := \sum_{k = 1}^t \frac{\partial F_{t, r}}{\partial \vtheta_{t - k}}(\vtheta)\trans \sum_{s = t - k}^{t - 1} \vF_s(\vtheta).
\end{equation}

Removing memory is the first step in our analysis. We then interpret the
correction terms in the resulting memoryless iteration, asking whether they
penalize or anti-penalize sharpness. We now illustrate the procedure on a simple
algorithm with memory.

\subsection{Warm-up: SGD with Momentum}
Consider mini-batch SGD with momentum, written in the form \cref{eq:quzxvB} with
$F_t(\vtheta_t, \ldots, \vtheta_0) := \sum_{k = 0}^t \beta^{t - k} \nabla
\mathcal{L}_k(\vtheta_k)$. This example introduces the main ideas behind the
framework; see \citet{cattaneo2025modified} for a fine-grained analysis of this
specific algorithm.

The memory removal technique just described (formally \Cref{th:UCVELF}) gives an approximation \eqref{eq:HsXmgC} with\internalComment{check}
\begin{equation}\label{eq:YeepIA}
  \begin{aligned}
    \main{t}(\vtheta) &= \sum_{k = 0}^t \beta^{t - k} \nabla \mathcal{L}_k(\vtheta),\\
    \corr{t}(\vtheta) &= \beta \sum_{q = 0}^{t - 1} \beta^q \sum_{l = 1}^{q + 1} \sum_{q_1 = 0}^{t - l} \beta^{q_1} \nabla^2 \mathcal{L}_{t - 1 - q}(\vtheta) \nabla \mathcal{L}_{t - l - q_1}(\vtheta).
  \end{aligned}
\end{equation}
The approximating algorithm does not have memory, so $\main{t}$ and $\corr{t}$ only depend on one point, which is already a significant simplification.
However, in this form these expressions are still very complex and their analysis appears impossible.
The next step (due to \citet{smith2021on}),
is to put $t = n$ and take the average $\Eletpi$ over all permutations of
samples. This gives the average one-epoch correction at the current point
$\btheta$, removing the accidental dependence on a particular batch order.
The average of the main term is easy to compute:
\begin{equation*}
  \Eletpi \main{n}(\vtheta)
  = \sum_{k = 0}^n \beta^{n - k} \Eletpi \nabla \mathcal{L}_k(\vtheta)
  = \sum_{k = 0}^n \beta^{n - k} \vg
  = \frac{1 - \beta^{n + 1}}{1 - \beta} \vg
  = \frac{1 + o_n(1)}{1 - \beta} \vg,
\end{equation*}
where $o_n(1)$ denotes terms that decay to zero as $n \to \infty$ (exponentially fast).
After some similar algebra, we also find the average correction:
\begin{equation*}
\Eletpi \corr{n}(\vtheta) = \frac{\beta + o_n(1)}{2 (1 - \beta)^3} \nabla \norm{\vg}^2
  + \frac{\beta + o_n(1)}{2 (1 - \beta)^2 (1 + \beta)} \nabla \prn*{\frac{\trace \mSigma}{b}}.
\end{equation*}
In other words,
\begin{multline*}
  \Eletpi \main{n}(\vtheta) + \eta \, \Eletpi \corr{n}(\vtheta)
  \\
  = \frac{1}{1 - \beta} \nabla \prn[\bigg]{(1 + o_n(1)) \mathcal{L} + \eta \frac{\beta + o_n(1)}{2 (1 - \beta)^2} \norm{\vg}^2
  + \eta \frac{\beta + o_n(1)}{2 (1 - \beta) (1 + \beta)} \frac{\trace \mSigma}{b}}.
\end{multline*}
This expression is non-random and is much easier to analyze. In the right-hand side, we see a modified loss with two correction terms:
\begin{itemize}
\item implicit regularization by memory $\eta \frac{\beta + o_n(1)}{2 (1 - \beta)^2} \norm{\vg}^2$ (present already in the full-batch case), and
\item implicit regularization by stochasticity $\eta \frac{\beta + o_n(1)}{2 (1 - \beta) (1 + \beta)} \frac{\trace \mSigma}{b}$ (appearing as a result of mini-batch noise).
\end{itemize}

The first term implicitly penalizes the squared norm of the gradient, which is a first-order approximation of (non-adaptive) $\ell_2$ sharpness \citep{foret2021sharpnessaware}: for small $\rho$,
\begin{equation*}
  \max_{\norm{\boldsymbol{\epsilon}} \leq \rho} \mathcal{L}(\vtheta + \boldsymbol{\epsilon}) - \mathcal{L}(\vtheta)
  \approx \max_{\norm{\boldsymbol{\epsilon}} \leq \rho} \vg\trans \boldsymbol{\epsilon} = \rho \norm{\vg}.
\end{equation*}

The second term implicitly penalizes gradient noise variance
\begin{equation*}
\trace \mSigma(\vtheta) = \frac{1}{(n + 1) b} \sum_{p = 1}^{(n + 1) b} \norm[\big]{\nabla (\ell_p - \mathcal{L})(\vtheta)}^2 = \frac{n b + b - 1}{n} \sum_i \Eletpi (\partial_i d_0)^2,
\end{equation*}
also related to flatness of the loss landscape and observed to be predictive of generalization \citep{Jiang2020Fantastic}.

Therefore, penalizing both terms is predictive of moving toward flatter regions
of the loss landscape and often better generalization. Following tradition, we classify
them as ``implicit regularization''.

\subsection{Summary}

Our approach interprets implicit biases of mini-batch versions of optimization
algorithms with memory using the following three steps.
\begin{enumerate}
\item \textbf{Removing memory:} use the memory-removal technique to approximate the algorithm with memory by a memoryless iteration.
\item \textbf{Calculating the average correction terms:} take expectation $\Eletpi \corr{n}(\vtheta)$ to remove dependence on a particular mini-batch order, and potentially make other simplifications without qualitatively changing the situation.
\item \textbf{Interpretation:} interpret the terms in the resulting expression, especially connecting to known sharpness/flatness or generalization measures.
\end{enumerate}

This framework gives a way to study how mini-batch noise influences, on
average, the implicit bias of memory in complex optimization algorithms used for
deep learning.

\section{Mini-Batch Noise in Adam}\label{sec:theory-on-adam}

We now apply the framework from \cref{sec:overview} to Adam. The section has
three steps. First, we remove memory and obtain a memoryless approximation with
an explicit correction term. Second, we average this correction over
without-replacement mini-batch order and decompose it into a full-batch term and
five mini-batch-noise terms. Third, we interpret the dominant terms through a
sharpness proxy and derive directional predictions for how the preferred betas
change with batch size.

We use Adam in its equivalent one-variable form, obtained by eliminating the
first- and second-moment variables.

\begin{definition}[Adam \citep{kingma2014adam}]\label{def:adam}
For numerical hyperparameters $\epsilon > 0$ and $\beta_1,\beta_2 \in (0,1)$, Adam can be written
for each coordinate $j \in \range{1}{\dim \btheta}$ as
\begin{gather*}
  \theta_{t + 1, j} = \theta_{t, j} - \eta \frac{\sum_{k = 0}^t \mu_{t, k} \partial_j \mathcal{L}_k(\btheta_k)}{\sqrt{\sum_{k = 0}^t \nu_{t, k} \abs{\partial_j \mathcal{L}_k(\btheta_k)}^2 + \epsilon}},\\
  \text{where} \quad \mu_{t, k} := \frac{\beta_1^{t - k} (1 - \beta_1)}{1 - \beta_1^{t + 1}},\quad \nu_{t, k} := \frac{\beta_2^{t - k} (1 - \beta_2)}{1 - \beta_2^{t + 1}}, \quad k \in \range{0}{t}, t \in \mathbb{Z}_{\geq 0},
\end{gather*}
with arbitrary initial parameter $\btheta_0 \in \mathbb{R}^{\dim \btheta}$.
\end{definition}



\subsection{Step 1: Removing Memory}\label{sec:removing-memory}

An application of the memory removal technique to the case of mini-batch Adam provides the following result whose full version is \cref{thm:applying-memory-removal}.

\begin{theorem}[Memory removal, simplified version]
The iteration $\crl[\big]{\vtheta_t}_{t = 0}^{\infty}$ given by Adam (\cref{def:adam}) is
$O(\eta^2)$-close for $O(\eta^{-1})$ steps in the sense of bound \eqref{eq:yqEMvw} to the iteration $\crl[\big]{\tilde{\vtheta}_t}_{t = 0}^{\infty}$ given by
\begin{equation*}
  \tilde{\vtheta}_{t + 1} = \tilde{\vtheta}_t
    - \eta \, \main{t}(\tilde{\vtheta}_t) - \eta^2 \, \corr{t}(\tilde{\vtheta}_t), \quad \tilde{\vtheta}_0 = \vtheta_0,
  \end{equation*}
  where
\begin{equation*}
\mainsc{j}{t}(\btheta) := \frac{\sum_{k = 0}^t \mu_{t, k} \partial_j \mathcal{L}_k(\btheta)}{\sqrt{\sum_{k = 0}^t \nu_{t, k} \abs{\partial_j \mathcal{L}_k(\btheta)}^2 + \epsilon}},
\end{equation*}
and the full expression for $\corr{t}(\btheta)$ is deferred to \eqref{eq:corr-term-adam-def} in \cref{sec:general-theorems} due to its length.
\end{theorem}

The proof follows from the general result in \citet{cattaneo2025howmemoryoptimization}, and is given in \cref{sec:proof-of-thm-applying-memory-removal}. The terms $\main{t}(\btheta)$ and $\corr{t}(\btheta)$ are complex and difficult to interpret. Thus, we proceed to the next step to simplify the analysis.

\subsection{Step 2: Calculating the Average Correction Terms}

In this step, we put $t = n$, expand $\corrsc{j}{n}(\btheta)$ up to degree-2
monomials in noise derivatives, and calculate the average of the result with
respect to permutations of samples. The expansion is local and second order in
mini-batch noise. It is intended to capture directional effects in regimes where
these terms dominate the omitted higher-order corrections, rather than to give
a uniformly accurate quantitative approximation for every extremely small batch
size.
Recall
that $d_k$, $\partial_i d_k$, $\partial_{i j} d_k$ denote the mini-batch noise and its partial derivatives. Accordingly, we will use the notation $O(d^p)$ to mean ``terms of order at least $p$ in (derivatives of) noise''. For example, all terms of the form $(\partial_{i j} d_k) (\partial_i d_k)$ are $O(d^2)$ and all terms of the form $(\partial_{i j l} d_k) (\partial_{i j} d_k) (\partial_l d_k)$ are $O(d^3)$.

The following is a simplified combination of
\cref{thm:finite-n-nonzero-eps,thm:limits}, proven in
\cref{sec:proof-of-prop-expect,sec:proof-of-limits-thm}. The appendix keeps the
nonzero-$\epsilon$ expressions; the main text presents the cleaner
small-$\epsilon$ form.

\begin{theorem}[Mini-batch noise expansion of the memoryless dynamics, simplified version]
\label{thm:expect}
The expectation $\Eletpi$ of the correction term with respect to the uniform law on all permutations $\range{1}{(n + 1) b} \to \range{1}{(n + 1) b}$ satisfies, up to small $\epsilon$ corrections and finite-epoch terms of order $o_n(b^{-1})$,
  \begin{equation}\label{eq:expect}
    \begin{aligned}
      \abs{g_j} \Eletpi \corrsc{j}{n}
        ={} &\fb{j}(\beta_1,\beta_2) + \mbnone{j}(\beta_1, \beta_2) + \mbntwo{j}(\beta_1, \beta_2)\\
      &+ \mbnthree{j}(\beta_1,\beta_2) + \mbnfour{j}(\beta_1,\beta_2) + \mbnfive{j}(\beta_1,\beta_2) + O(d^3),
    \end{aligned}
  \end{equation}
  where the full-batch correction is given by
\begin{equation}\label{eq:fb-simple}
\fb{j}(\beta_1,\beta_2) := \prn[\bigg]{\frac{\beta_1}{1 - \beta_1} - \frac{\beta_2}{1 - \beta_2}} \partial_j \norm{\vg}_1,
\end{equation}
and five mini-batch noise corrections are given by
\begin{equation}\label{eq:mbn-simple}
\begin{aligned}
  \mbnone{j}(\beta_1, \beta_2) &:= \frac{1}{b} C_1(\beta_1, \beta_2)
    \partial_j \norm{\vg}_1 \frac{\Sigma_{j j}}{g_j^2},\\
  \mbntwo{j}(\beta_1, \beta_2) &:= \frac{1}{b} C_2(\beta_1, \beta_2)
    \sum_i \partial_j \abs{g_i} \frac{\Sigma_{i i}}{g_i^2},\\
  \mbnthree{j}(\beta_1,\beta_2) &:= \frac{1}{b} C_3(\beta_1, \beta_2)
      \frac{\sign g_j}{\abs{g_j}} \sum_i \sign g_i \, \partial_i \Sigma_{j j},\\
  \mbnfour{j}(\beta_1,\beta_2) &:= \frac{1}{b} C_4(\beta_1, \beta_2)
    \sum_i \frac{1}{\abs{g_i}}
    \partial_j \Sigma_{i i},\\
  \mbnfive{j}(\beta_1,\beta_2) &:= \frac{1}{b} C_5(\beta_1, \beta_2)
    \frac{\sign g_j}{\abs{g_j}} \sum_i \frac{\partial_i g_j}{\abs{g_i}} \Sigma_{i j},
\end{aligned}
\end{equation}
with $\Sigma_{i j}$ denoting the $(i, j)$th component of the empirical per-sample gradient covariance matrix $\mSigma$, and the values of $\crl{C_k(\beta_1, \beta_2)}_{k = 1}^5$ deferred to \cref{eq:constants} in \cref{sec:general-theorems} due to their length.
\end{theorem}

\begin{table}[t]
\caption{Interpretation of the terms in the mini-batch-noise expansion.}
\label{tab:correction-term-roles}
\vskip 0.05in
\begin{center}
\small
\begin{tabular}{p{0.14\linewidth}p{0.34\linewidth}p{0.42\linewidth}}
\toprule
Term & Role & Treatment in the main interpretation \\
\midrule
$\fb{j}$ & Full-batch memory correction; for $\beta_1 < \beta_2$, it
anti-penalizes the non-adaptive $\ell_\infty$ sharpness proxy. & Kept; this is
                                                         the baseline anti-regularizing term that mini-batch noise competes with. \\
\addlinespace
$\mbnone{j}$ & Diagonal noise-to-signal correction proportional to
$\Sigma_{jj}/g_j^2$. & Kept; it contributes to the sharpness-bias coefficient
                  after the simple-noise-scale approximation. \\
\addlinespace
$\mbntwo{j}$ & Cross-coordinate correction involving
$\sum_i \partial_j |g_i|\,\Sigma_{ii}/g_i^2$. & Kept; after replacing
per-coordinate noise-to-signal ratios by their global average, it has the same
                                 sharpness-proxy structure as $\mbnone{j}$. \\
\addlinespace
$\mbnthree{j}$ & Covariance-derivative term involving
$\partial_i \Sigma_{jj}$. & Treated as sign-neutral for the sharpness direction;
                \cref{sec:term-c_3-neutral-generalization} gives the heuristic argument. \\
  \addlinespace

$\mbnfour{j}$, $\mbnfive{j}$ & Remaining covariance-derivative and off-diagonal
covariance terms. & Neglected in the main interpretation because their
coefficients are small in the beta ranges studied; see
\cref{lem:c4-c5-are-small}. \\
\bottomrule
\end{tabular}
\end{center}
\end{table}

For very small batches, these degree-2 monomials may not be enough for an
accurate quantitative approximation. Our predictions below are therefore
directional: they describe the sign and monotonicity of the sharpness-bias
coefficient, not the exact optimal performance
boundaries.

\subsection{Step 3: Interpretation}

We need to analyze each term in the right-hand side of \cref{eq:expect}. Since
the theorem is written after multiplying the correction by $\abs{g_j}$,
all parts of the correction term have the same preconditioning $\abs{g_j}^{-1}$ as the full-batch Adam does.
Because the sign of the effect depends on the region of hyperparameter space,
we focus on two one-dimensional sweeps. In each case, one beta is fixed at a
common value and we ask which value of the other beta is preferred under the
sharpness proxy discussed in the introduction. Specifically, we study how to set
$\beta_2$ when $\beta_1$ is fixed at $0.9$, and how to set $\beta_1$ when
$\beta_2$ is fixed at the default value $0.999$.

\paragraph{How to set $\beta_2$ if $\beta_1$ is fixed}
We start with the setting where $\beta_1$ is fixed at its default value $0.9$
and $\beta_2$ varies in the interval $[0.9, 1)$:
\begin{equation*}
\beta_1 = 0.9, \quad \text{seeking preferred }\beta_2 \in [0.9, 1).
\end{equation*}

The terms $\mbnfour{j}(\beta_1, \beta_2)$ and
$\mbnfive{j}(\beta_1, \beta_2)$ are easiest to handle:
\cref{lem:c4-c5-are-small} shows that their coefficients are small compared to
the dominant terms in this beta range. In addition, we argue in
\cref{sec:term-c_3-neutral-generalization} that
$\mbnthree{j}(\beta_1, \beta_2)$ is sign-neutral for our sharpness-direction
interpretation.

We are left with the sum of three terms:
$\fb{j}(\beta_1,\beta_2)$, $\mbnone{j}(\beta_1, \beta_2)$ and $\mbntwo{j}(\beta_1, \beta_2)$.
The full-batch term $\fb{j}(\beta_1,\beta_2)$ anti-penalizes, when
$\beta_1 < \beta_2$, the 1-norm of the gradient. This is a first-order
approximation of non-adaptive $\ell_{\infty}$-sharpness: for small $\rho$,
\begin{equation*}
  \max_{\norm{\boldsymbol{\epsilon}}_{\infty} \leq \rho} \mathcal{L}(\vtheta + \boldsymbol{\epsilon}) - \mathcal{L}(\vtheta)
  \approx \max_{\norm{\boldsymbol{\epsilon}}_{\infty} \leq \rho} \vg\trans \boldsymbol{\epsilon} = \rho \norm{\vg}_1.
\end{equation*}
Thus, we can refer to this term as anti-regularization,
same as in the setting with zero noise \citep{pmlr-v235-cattaneo24a}. The term containing $C_1(\beta_1, \beta_2)$ (it can be checked that it is positive in our setting) provides regularization:
it also penalizes the $\ell_1$ gradient norm although the magnitude of this penalization in each component $j$ depends on the per-component
noise-to-signal ratio $\Sigma_{j j} / g_j^2$.

The term $\mbntwo{j}(\beta_1, \beta_2)$ is more complicated, so we make an
explicit simple-noise-scale approximation. At the current point $\btheta$, we
replace the per-coordinate noise-to-signal ratios $\Sigma_{i i}/g_i^2$ by a
single global average, the ``simple noise scale''
$\mathcal{B}_{\text{simple}}$ from
\cite{mccandlish2018empiricalmodellarge}:
\begin{equation*}
\mathcal{B}_{\text{simple}} := \frac{\trace \mSigma}{\sum_j g_j^2} = \frac{\trace \mSigma}{\norm{\vg}^2}.
\end{equation*}
This approximation discards per-coordinate variation in $\Sigma_{i i}/g_i^2$,
but it preserves the global scale of mini-batch noise relative to the full-batch
gradient. After this replacement, the terms
$\mbnone{j}(\beta_1, \beta_2)$
and
$\mbntwo{j}(\beta_1, \beta_2)$
have the same sharpness-proxy structure up to their coefficients:
$C_1(\beta_1, \beta_2) \partial_j \norm{\vg}_1 b^{-1} \mathcal{B}_{\text{simple}}$ and $C_2(\beta_1, \beta_2) \partial_j \norm{\vg}_1 b^{-1} \mathcal{B}_{\text{simple}}$ respectively.

Under this approximation, the dominant terms in \eqref{eq:expect} provide
implicit \mbox{(anti-)penalization} of an approximate non-adaptive sharpness
measure with coefficient
$C_{\text{total}} \prn[\big]{\beta_1, \beta_2, b^{-1} \mathcal{B}_{\text{simple}}}$,
where the function $(0, + \infty) \ni \lambda \mapsto C_{\text{total}} \prn[\big]{\beta_1, \beta_2, \lambda}$ is defined by
\begin{equation}\label{eq:ctotal}
    C_{\text{total}}(\beta_1, \beta_2, \lambda) := \frac{\beta_1}{1 - \beta_1} - \frac{\beta_2}{1 - \beta_2} + \crl{C_1(\beta_1, \beta_2) + C_2(\beta_1, \beta_2)} \lambda.
\end{equation}
If $C_{\text{total}}(\beta_1, \beta_2, b^{-1} \mathcal{B}_{\text{simple}}) > 0$, this can be interpreted as regularization, otherwise as anti-regularization.

It remains to analyze how this coefficient depends on $\lambda > 0$. We use the following fact whose full version is \cref{prop:ctotal-monotonicity-b2-full}.

\begin{proposition}[Monotonicity of \(C_{\mathrm{total}}(0.9,\beta_2,\lambda)\), simplified version]
\label{lem:ctotal-monotonicity-b2-simplified.}
  If $\lambda \geq 0.5082$, the function $C_{\text{total}}(0.9, \beta_2, \lambda)$ is increasing in $\beta_2 \in [0.9, 1)$. If $0 < \lambda < 0.494$, it is decreasing in $\beta_2 \in [0.9, 1)$.
\end{proposition}

We obtain the following prediction. For fixed $\beta_1 = 0.9$ and
$\beta_2 \in [0.9, 1)$, the quantity
$b^{-1}\mathcal{B}_{\text{simple}}$ controls the monotonicity of the approximate
sharpness-bias coefficient. If this quantity is significantly below $0.5$,
equivalently if $b \gg 2\mathcal{B}_{\text{simple}}$, the coefficient decreases
with $\beta_2$: higher $\beta_2$ means weaker sharpness penalization, often
predicting worse generalization. If the quantity is significantly above $0.5$,
equivalently if $b \ll 2\mathcal{B}_{\text{simple}}$, the coefficient increases
with $\beta_2$: higher $\beta_2$ means stronger sharpness penalization, often
predicting better generalization:
\begin{align*}
  b \gg 2 \mathcal{B}_{\text{simple}} \quad &\Rightarrow \quad \text{$C_{\text{total}} \prn[\big]{0.9, \beta_2, b^{-1} \mathcal{B}_{\text{simple}}}$ $\searrow$ in $\beta_2$},\\
  b \ll 2 \mathcal{B}_{\text{simple}} \quad &\Rightarrow \quad \text{$C_{\text{total}} \prn[\big]{0.9, \beta_2, b^{-1} \mathcal{B}_{\text{simple}}}$ $\nearrow$ in $\beta_2$}.
\end{align*}
Theoretically, the transition happens very quickly around the point where the batch size is $2 \mathcal{B}_{\text{simple}}$, although simplifications that we used make the theoretical transition quicker than it is in practice.

\paragraph{How to set $\beta_1$ if $\beta_2$ is fixed}
Next, we consider the complementary one-dimensional sweep, where $\beta_2$ is
fixed at a default value and $\beta_1$ varies. This sweep is less common in
practice, but it is still useful for understanding the full picture:
\begin{equation*}
\beta_2 = 0.999, \quad \text{seeking preferred }\beta_1 \in [0.9, 1).
\end{equation*}
The following simplified variant of \cref{prop:ctotal-monotonicity-b1-full} describes this situation.

\begin{proposition}[Monotonicity of \(C_{\mathrm{total}}(\beta_1,0.999,\lambda)\), simplified version]\label{prop:ctotal-at-b2-0999}
If $\lambda \geq 1.002$, the function $C_{\text{total}}(\beta_1, 0.999, \lambda)$ is strictly decreasing in $\beta_1 \in [0.9, 1)$. If $0 < \lambda < 0.995$, it is strictly increasing in $\beta_1 \in [0.9, 1)$.
\end{proposition}

In this case, if $b^{-1}\mathcal{B}_{\text{simple}}$ is much larger than one
($b \ll \mathcal{B}_{\text{simple}}$), the approximate sharpness-bias
coefficient decreases with $\beta_1$. Thus, larger $\beta_1$ weakens the bias
toward flatter regions, often predicting worse generalization, and the
sharpness proxy favors taking $\beta_1$ as low as stability allows (for example
$0.9$). If $b^{-1}\mathcal{B}_{\text{simple}}$ is much smaller than one
($b \gg \mathcal{B}_{\text{simple}}$), the coefficient increases with
$\beta_1$. In that low-noise regime, the sharpness proxy favors moving
$\beta_1$ upward, with $\beta_1=\beta_2$ as a natural first choice when training
remains stable.

\py{fig_transformer_xl_wikitext2_adam_beta_2_sweeps_val_ppl()}
\begin{figure*}[htb!]
\centering
\begin{subfigure}[t]{0.320\textwidth}
\includegraphics[width=\linewidth]{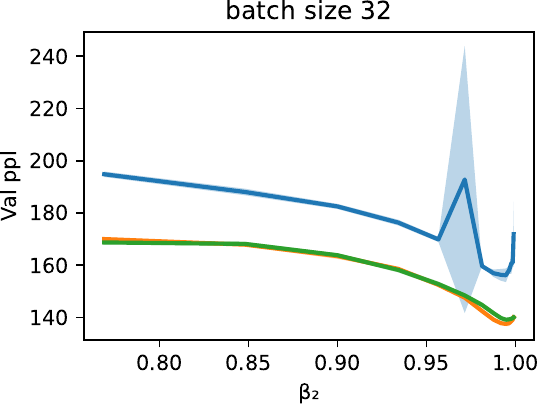}
\end{subfigure}
\hfill
\begin{subfigure}[t]{0.320\textwidth}
\includegraphics[width=\linewidth]{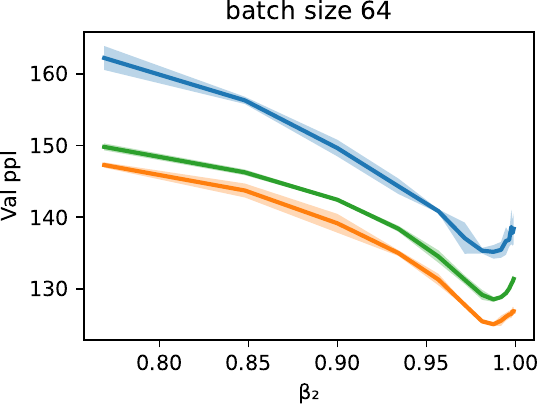}
\end{subfigure}
\hfill
\begin{subfigure}[t]{0.320\textwidth}
\includegraphics[width=\linewidth]{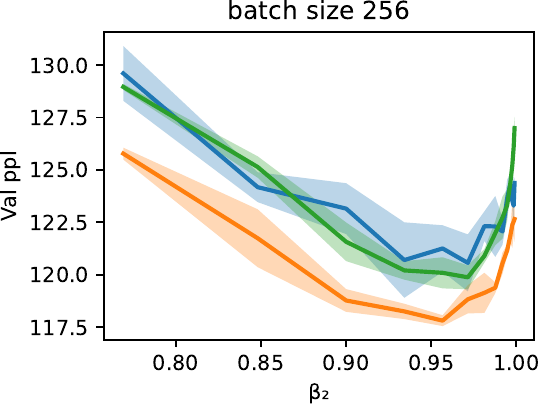}
\end{subfigure}
\bigskip
\centering
\begin{subfigure}[t]{0.320\textwidth}
\includegraphics[width=\linewidth]{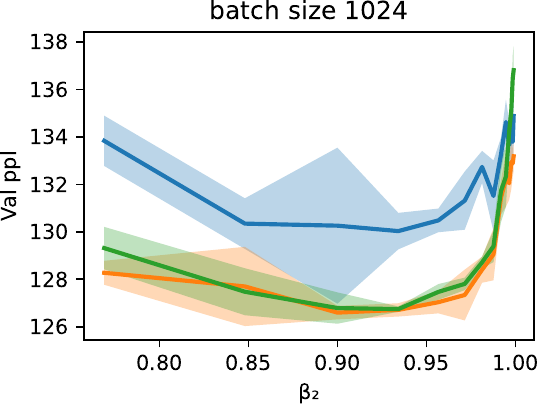}
\end{subfigure}
\hfill
\begin{subfigure}[t]{0.320\textwidth}
\includegraphics[width=\linewidth]{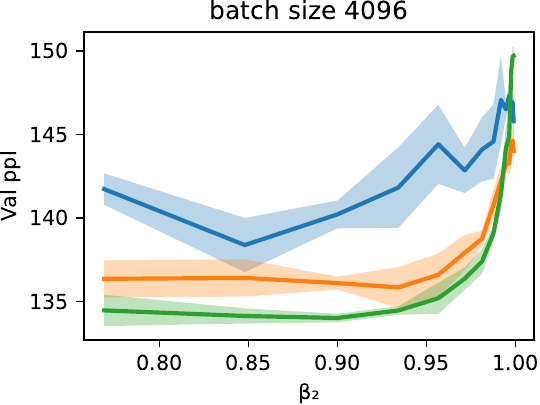}
\end{subfigure}
\hfill
\begin{subfigure}[t]{0.320\textwidth}
\includegraphics[width=\linewidth]{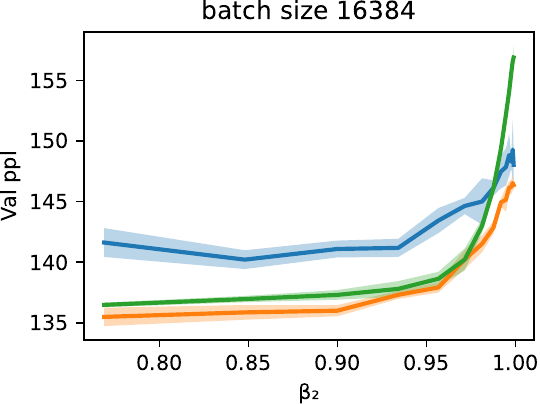}
\end{subfigure}
\caption{Minimal validation perplexity (before overfitting) of a small Transformer trained with Adam on WikiText-2 with different batch sizes, learning rates $\crl{{\color{mplblue} 10^{-3}}, {\color{mplorange} 10^{-3.5}}, {\color{mplgreen} 10^{-4}}}$, $\beta_1 = 0.9$ (averaged over three iterations).}
\label{fig:transformer_xl_wikitext2_adam_beta_2_sweeps_val_ppl}
\end{figure*}

\begin{figure*}[htb!]
\py{fig_transformer_xl_wikitext2_sharpness_snapshots()}
\centering
\begin{subfigure}[t]{0.320\textwidth}
\includegraphics[width=\linewidth]{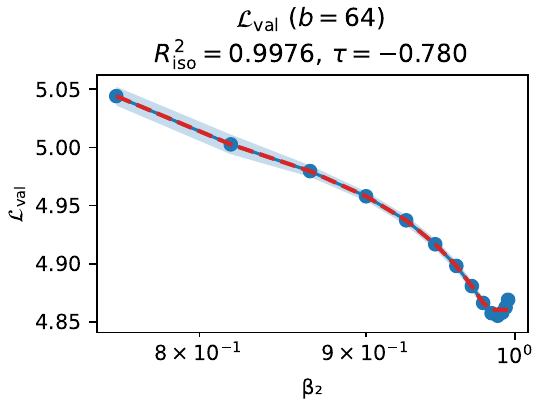}
\end{subfigure}
\hfill
\begin{subfigure}[t]{0.320\textwidth}
\includegraphics[width=\linewidth]{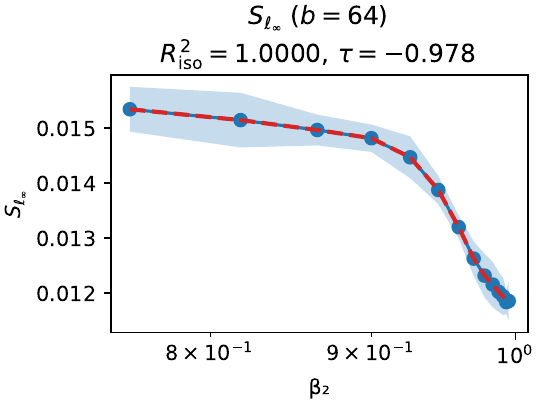}
\end{subfigure}
\hfill
\begin{subfigure}[t]{0.320\textwidth}
\includegraphics[width=\linewidth]{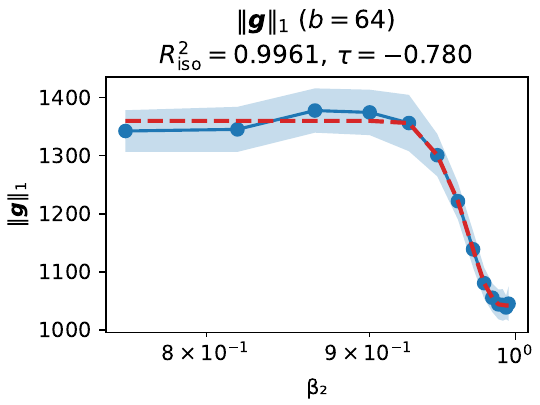}
\end{subfigure}
\bigskip
\centering
\begin{subfigure}[t]{0.320\textwidth}
\includegraphics[width=\linewidth]{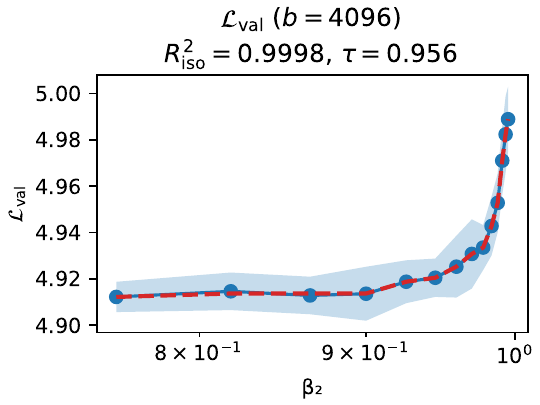}
\end{subfigure}
\hfill
\begin{subfigure}[t]{0.320\textwidth}
\includegraphics[width=\linewidth]{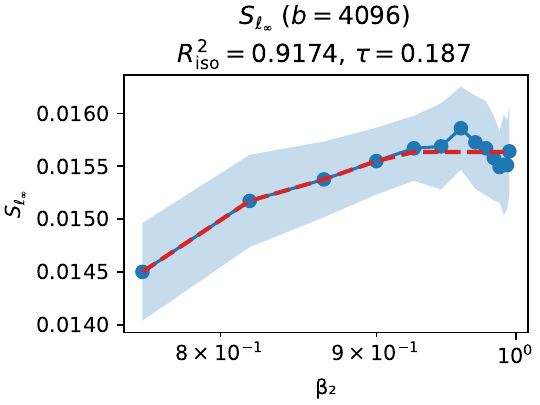}
\end{subfigure}
\hfill
\begin{subfigure}[t]{0.320\textwidth}
\includegraphics[width=\linewidth]{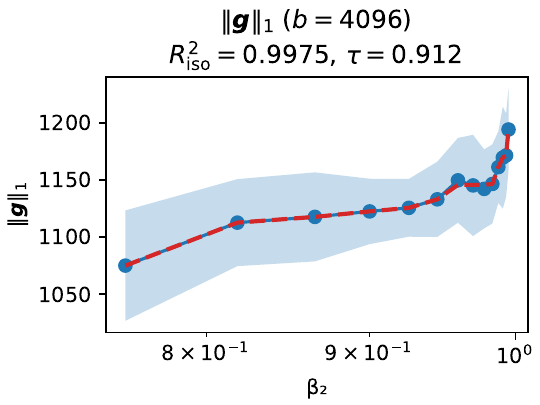}
\end{subfigure}
\caption{Validation loss, $\ell_\infty$-sharpness, $\ell_1$-norm of the gradient as a function of $\beta_2$ at the median epoch of overfitting, for a small Transformer trained with Adam on WikiText-2 at a small (top) and large (bottom) batch size (averaged across at least 16 iterations). Red dashed line denotes an isotonic regression fit; $\tau$ denotes Kendall's tau.}
\label{fig:transformer_xl_wikitext2_sharpness_snapshots}
\end{figure*}

\begin{figure*}[htb!]
\py{fig_llama3_dclm_sharpness_snapshots()}
\centering
\begin{subfigure}[t]{0.320\textwidth}
\includegraphics[width=\linewidth]{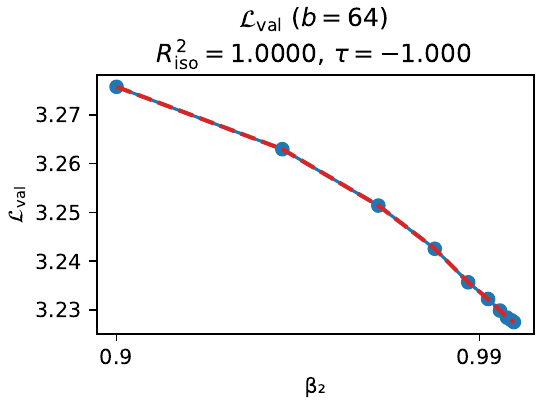}
\end{subfigure}
\hfill
\begin{subfigure}[t]{0.320\textwidth}
\includegraphics[width=\linewidth]{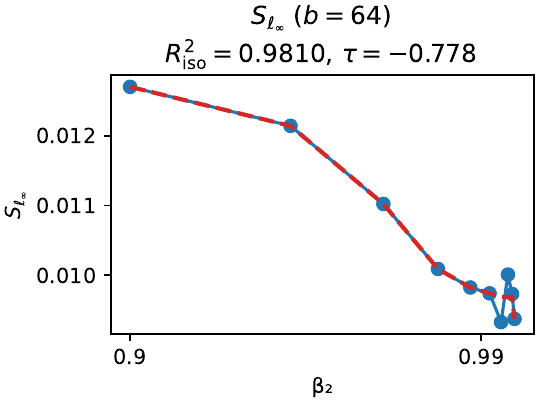}
\end{subfigure}
\hfill
\begin{subfigure}[t]{0.320\textwidth}
\includegraphics[width=\linewidth]{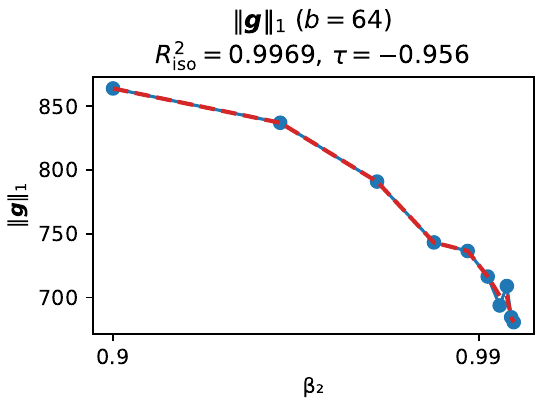}
\end{subfigure}
\bigskip
\centering
\begin{subfigure}[t]{0.320\textwidth}
\includegraphics[width=\linewidth]{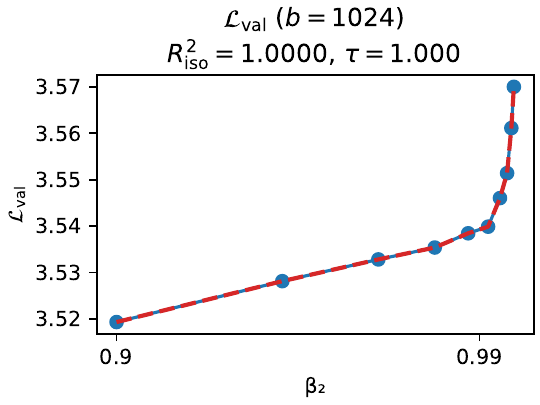}
\end{subfigure}
\hfill
\begin{subfigure}[t]{0.320\textwidth}
\includegraphics[width=\linewidth]{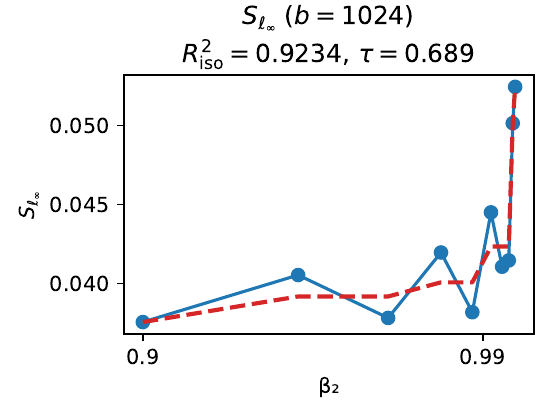}
\end{subfigure}
\hfill
\begin{subfigure}[t]{0.320\textwidth}
\includegraphics[width=\linewidth]{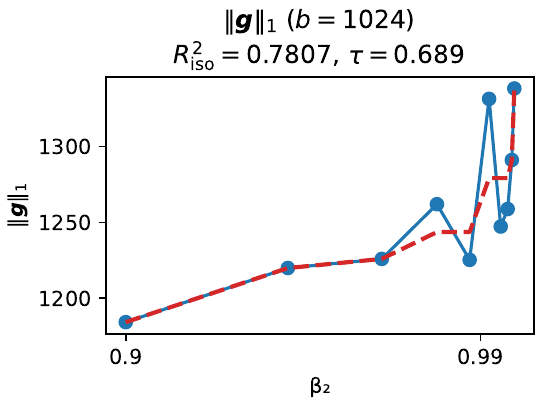}
\end{subfigure}
\caption{Validation loss, $\ell_\infty$-sharpness, $\ell_1$-norm of the gradient as a function of $\beta_2$ at the end of training, for a 280\,M-parameter Llama3 trained with AdamW on DCLM at a small (top) and large (bottom) batch size matched on total token budget. Red dashed line denotes an isotonic regression fit; $\tau$ denotes Kendall's tau.}
\label{fig:llama3_dclm_sharpness_snapshots}
\end{figure*}


\paragraph{Takeaway} If the batch size is much smaller than $\mathcal{B}_{\text{simple}}$, take $\beta_1$ much smaller than $\beta_2$  (e.\,g., the default values $\beta_1 = 0.9$, $\beta_2 = 0.999$ are a reasonable first choice). If the batch size is much larger than $2 \mathcal{B}_{\text{simple}}$, take $\beta_1 = \beta_2$ (e.\,g., $\beta_1 = \beta_2 = 0.9$ is a natural first choice).

Because the derivation
uses simplifications, the predictions are directional rather than precisely
quantitative. In particular, the main conclusion is about the scale of the
transition, controlled by $\mathcal{B}_{\text{simple}}$ (or
$2\mathcal{B}_{\text{simple}}$), which is not difficult to estimate
\citep{mccandlish2018empiricalmodellarge}. Thus, the rule of thumb can guide
practical choices of $\beta_1$ and $\beta_2$ without requiring very large grids.

\section{Experiments}\label{sec:experiments}

\paragraph{Small Transformer overfitting on a small dataset}
We train a small Transformer from \citep{Dai2019TransformerXLAL} on WikiText-2 \citep{merity2017pointer} following \citet{kunstner2023noise}.
We fix the default value $\beta_1 = 0.9$, and sweep $\beta_2$ for different batch sizes and learning rates. Running sufficiently many epochs to let the model overfit, we plot the minimal validation perplexity achieved depending on $\beta_2$.
The results in \cref{fig:transformer_xl_wikitext2_adam_beta_2_sweeps_val_ppl} show that in small-batch Adam, larger $\beta_2$ mostly helps the model generalize better (decreases minimal validation perplexity), and this behavior smoothly transitions into the opposite as the batch size increases.
To track the mechanism, we also plot in \cref{fig:transformer_xl_wikitext2_sharpness_snapshots} non-adaptive $\ell_{\infty}$ sharpness evaluated exactly, along with its approximation $\norm{\vg}_1$ for a selected learning rate and two (small and large) batch sizes. Additional figures, including a sweep of $\beta_1$ at a fixed $\beta_2 = 0.999$, are provided in \cref{sec:further-evidence-and-experiment-details}.


\paragraph{Llama-3 on DCLM}
We train a modern Llama3 \citep{dubey2024llama} model with 280\,M parameters (no weight tying) on DCLM \citep{li2024datacomplm} with sequence length 4096 for about 25\,B tokens (90 tokens per parameter).
The learning rate is warmed up
linearly for the first $10\%$ of the training run and then decayed at a cosine schedule to $10\%$ of the peak value.
The optimizer is AdamW \citep{loshchilov2019decoupledweightdecayregularization} with decoupled weight decay $\lambda = 0.1$ applied to all but the normalization parameters.
No batches of training data are repeated, and the number of iterations is chosen to match the token count.
In \cref{fig:llama3_dclm_sharpness_snapshots}, we plot the loss, $\ell_{\infty}$ sharpness and 1-norm of the full-batch gradient on a fixed set of 1024 validation sequences (full-batch calculation on the training data is infeasible), for a fixed peak learning rate and two (small and large) pretraining batch sizes, sweeping $\beta_2$.
We observe the familar monotonicity reversion.
Although train/validation performance gap is not applicable, even in single-epoch training the trends may be important for post-training or post-quantization performance \citep{watts2026sharpnessaware,springer2025overtrained}.

\section*{Acknowledgments}

Cattaneo gratefully acknowledges financial support from the National Science Foundation through DMS-2210561 and SES-2241575. We acknowledge the Princeton Research Computing resources, coordinated by the Princeton Institute for Computational Science and Engineering (PICSciE) and the Office of Information Technology's Research Computing.

\bibliography{beta}

\newpage
\appendix

\section{Further Evidence and Experiment Details}\label{sec:further-evidence-and-experiment-details}


\subsection{Transformer-XL on WikiText-2: $\beta_2$ Sweep with $\beta_1 = 0.9$ Fixed}

Transformer-XL \citep{Dai2019TransformerXLAL} with about 55\,M parameters is trained on WikiText-2 \citep{merity2017pointer} following \citet{kunstner2023noise} until after overfitting on the training set (when the validation loss starts rising).
Adam is used without weight decay, $\epsilon = 10^{-6}$.
The learning rate is constant $10^{-4}$.
For $b = 64$, 16 iterations completed and the sharpness metrics are plotted at epoch 9 (median overfitting epoch 9, mean 8.5); for $b = 4096$, 32 iterations completed and the sharpness metrics are plotted at epoch 100 (median overfitting epoch 103, mean 99.2).
We plot the validation loss, $\ell_{\infty}$ sharpness calculated by projected gradient ascent and noise metrics at the last step, along with scatterplots showing correlations; $\rho$ in the titles denotes the Pearson correlation coefficient and $\tau$ denotes Kendall's tau.
Validation loss does not use EMA but other metrics use EMA with parameter $\beta = 0.99$.

\paragraph{Sharpness Metrics at Batch Size 64}
{\setlength{\tabcolsep}{0pt}%
\begin{longtable}{ccc}
  \includegraphics[width=0.320\textwidth]{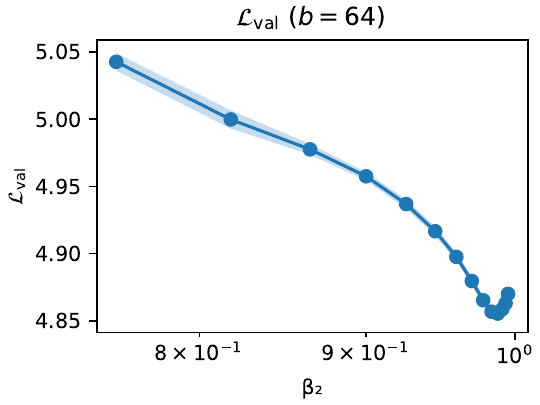}
  &
  \includegraphics[width=0.320\textwidth]{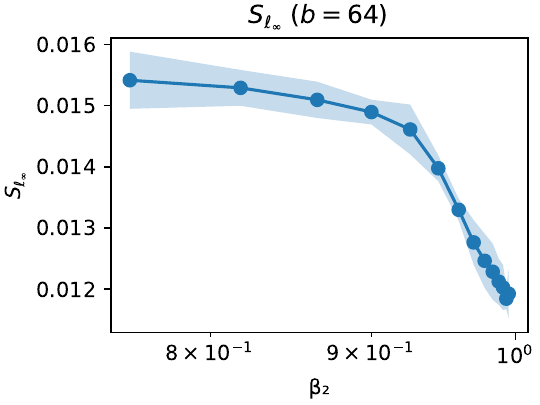}
  &
  \includegraphics[width=0.320\textwidth]{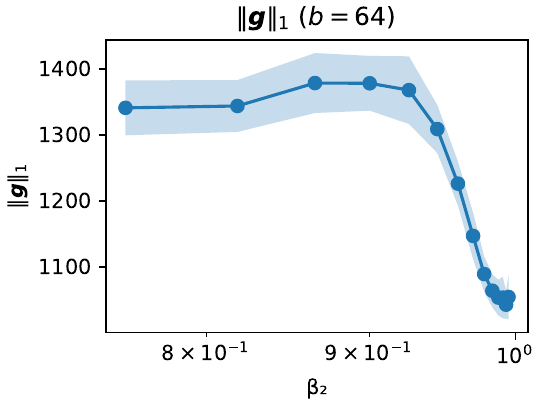}
  \\
  \includegraphics[width=0.320\textwidth]{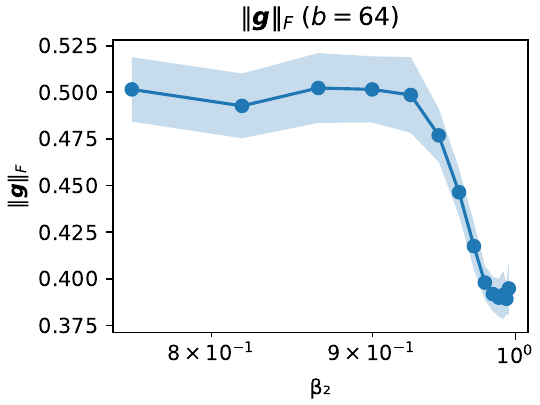}
  &
\includegraphics[width=0.320\textwidth]{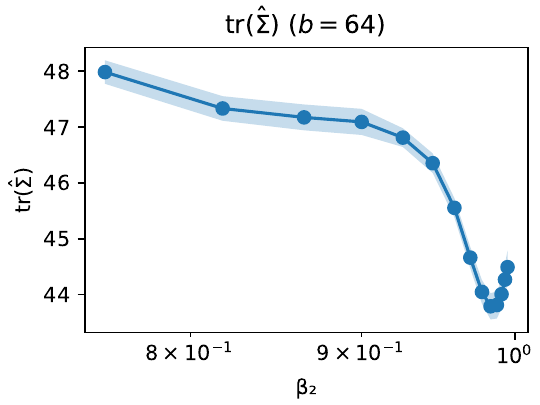}
  & \includegraphics[width=0.320\textwidth]{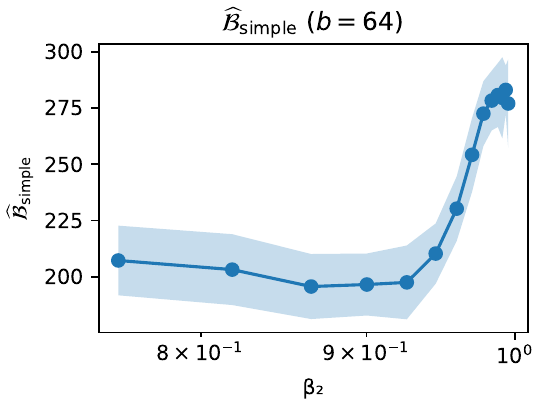}
  \\
  \includegraphics[width=0.320\textwidth]{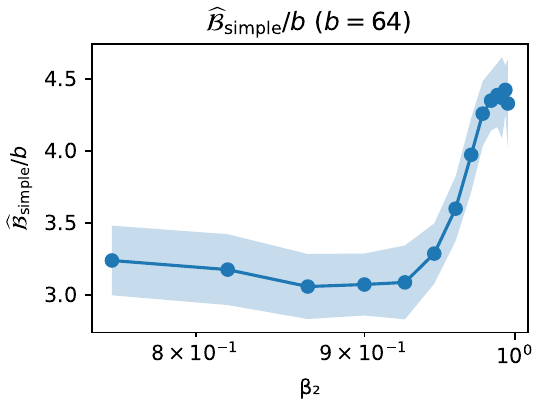}
  & \includegraphics[width=0.320\textwidth]{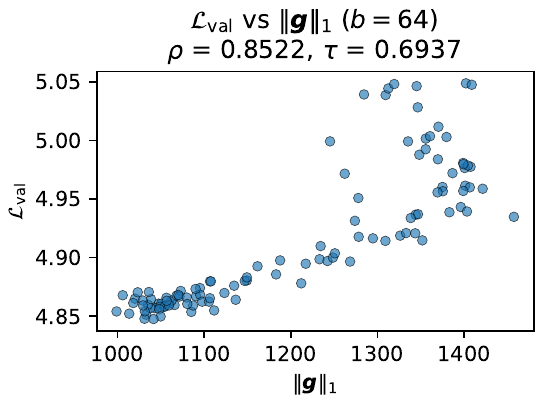}
  & \includegraphics[width=0.320\textwidth]{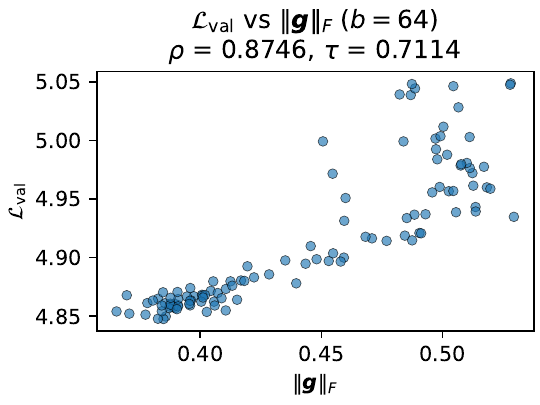}
  \\ \includegraphics[width=0.320\textwidth]{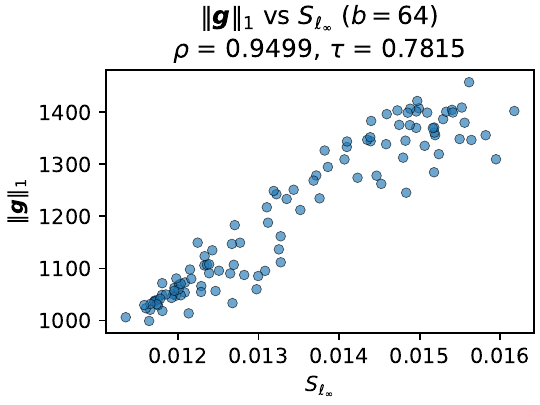}
  & \includegraphics[width=0.320\textwidth]{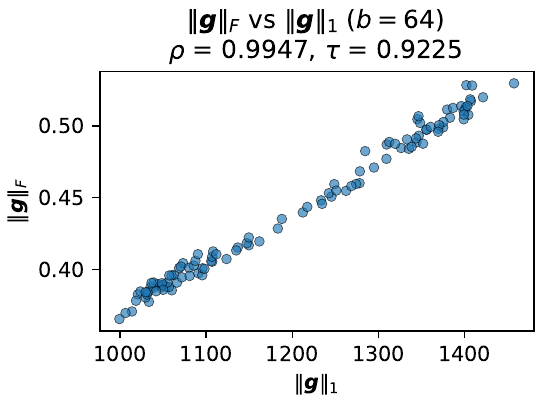}
  &
\end{longtable}}

\paragraph{Sharpness Metrics at Batch Size 4096}
{\setlength{\tabcolsep}{0pt}%
\begin{longtable}{ccc}
  \includegraphics[width=0.320\textwidth]{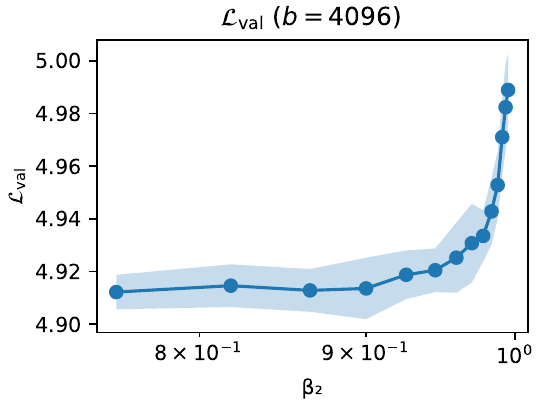}
  &
  \includegraphics[width=0.320\textwidth]{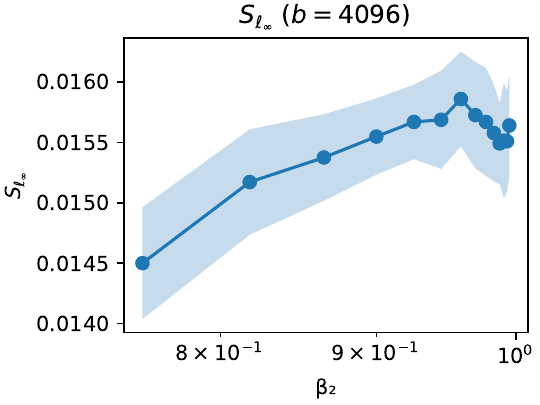}
  &
  \includegraphics[width=0.320\textwidth]{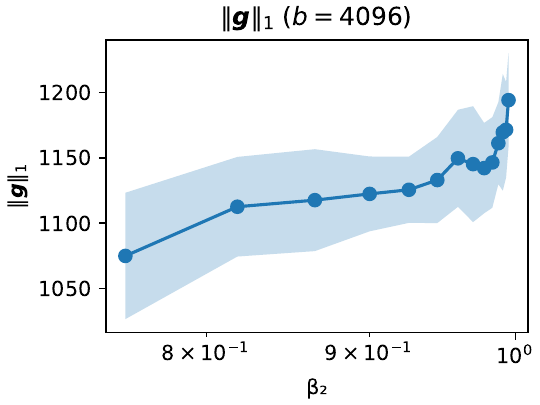}
  \\
  \includegraphics[width=0.320\textwidth]{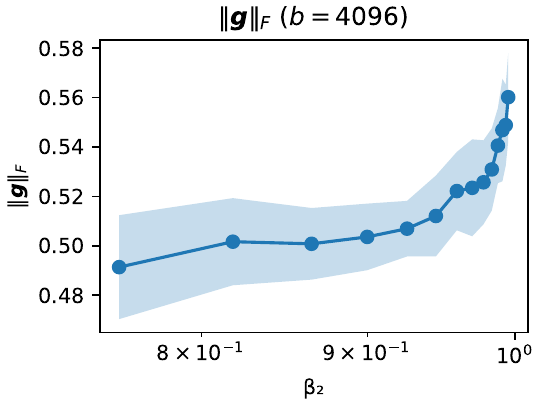}
  &
  \includegraphics[width=0.320\textwidth]{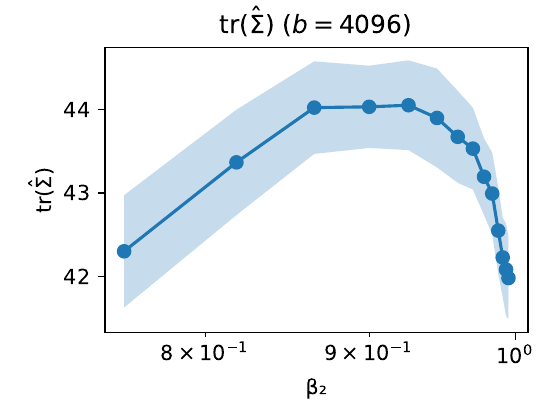}
  & \includegraphics[width=0.320\textwidth]{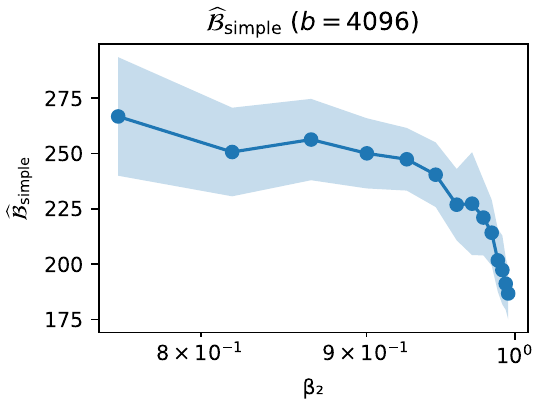}
  \\
  \includegraphics[width=0.320\textwidth]{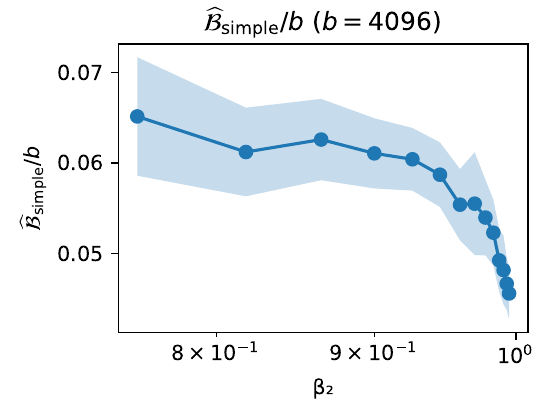}
  & \includegraphics[width=0.320\textwidth]{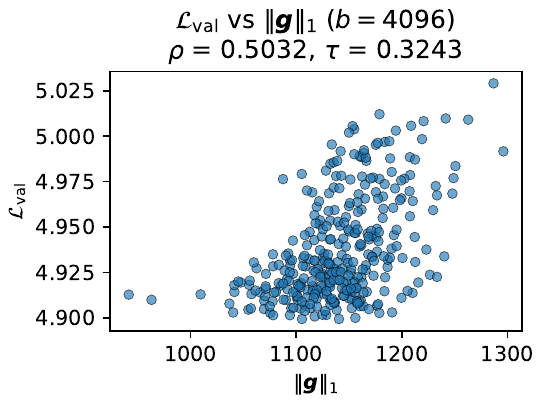}
  & \includegraphics[width=0.320\textwidth]{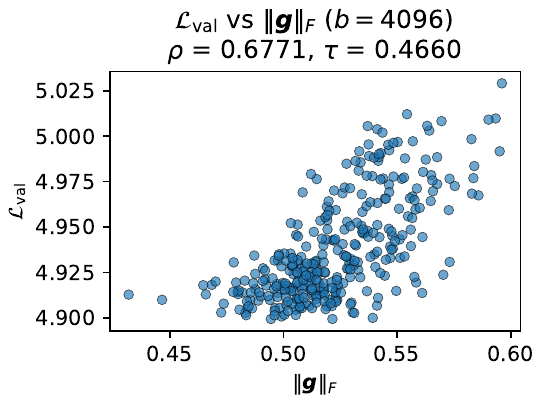}
  \\ \includegraphics[width=0.320\textwidth]{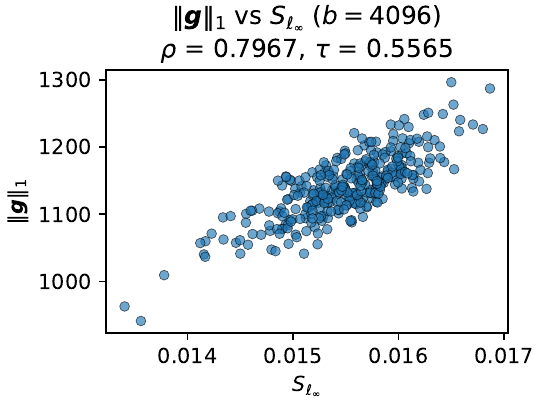}
  & \includegraphics[width=0.320\textwidth]{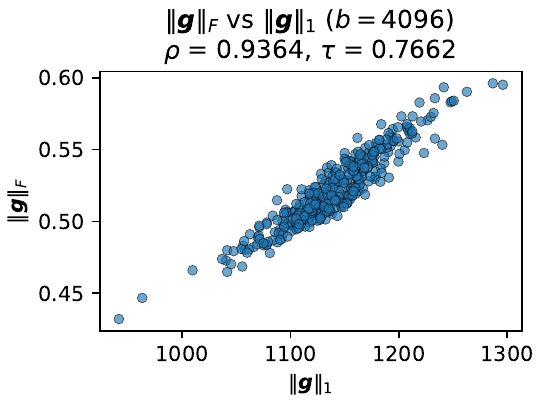}
  &
\end{longtable}}

\subsection{Transformer-XL on WikiText-2: $\beta_1$ Sweep with $\beta_2 = 0.999$ Fixed}

\subsubsection{Minimal Validation Perplexity at a Large Set of Batch Sizes}
We plot the minimal validation perplexity achieved (by definition, it is exactly at the point of overfitting), with learning rates $\crl{{\color{mplblue} 10^{-3.5}}, {\color{mplorange} 10^{-4}}, {\color{mplgreen} 10^{-4.5}}}$ and different batch sizes. The monotonicity trends as a function of $\beta_1$ largely revert as the batch size increases.

{\setlength{\tabcolsep}{0pt}%
\begin{longtable}{ccc}
  \includegraphics[width=0.320\textwidth]{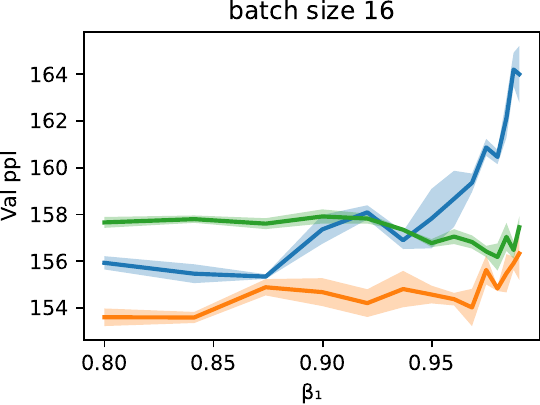}
  &
  \includegraphics[width=0.320\textwidth]{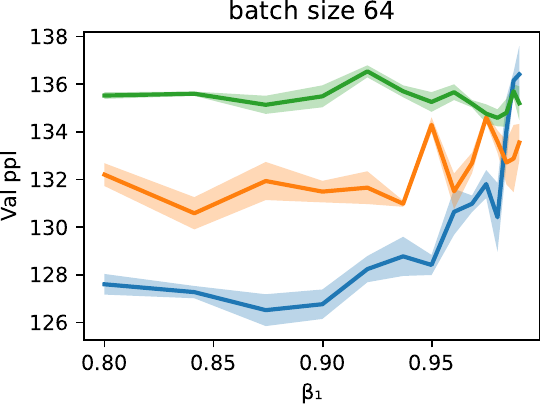}
  &
  \includegraphics[width=0.320\textwidth]{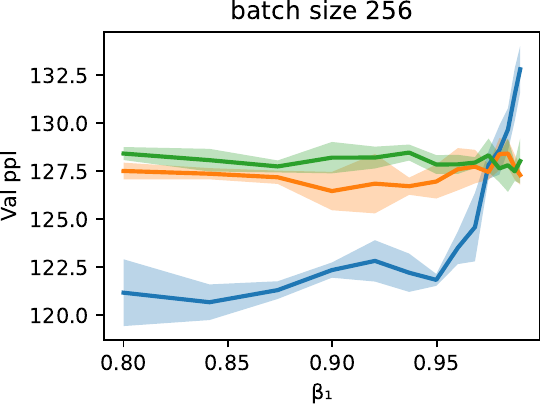}
  \\
  \includegraphics[width=0.320\textwidth]{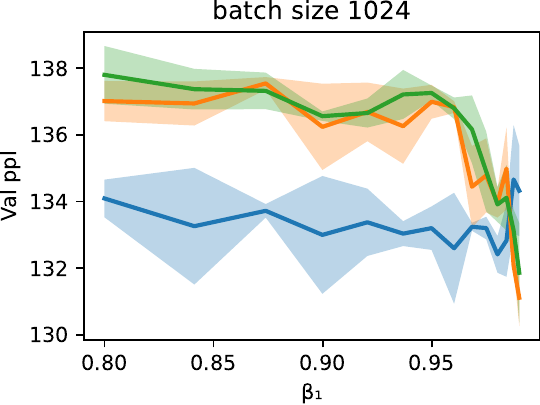}
  &
  \includegraphics[width=0.320\textwidth]{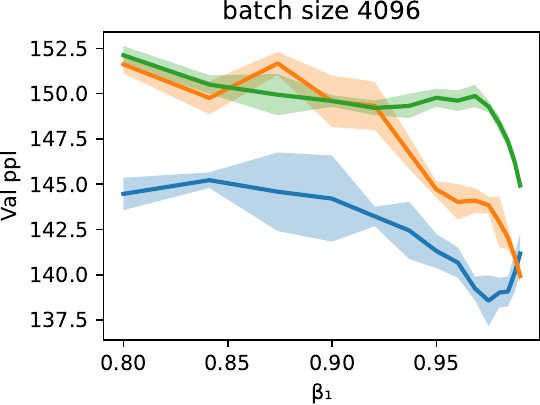}
  &
  \includegraphics[width=0.320\textwidth]{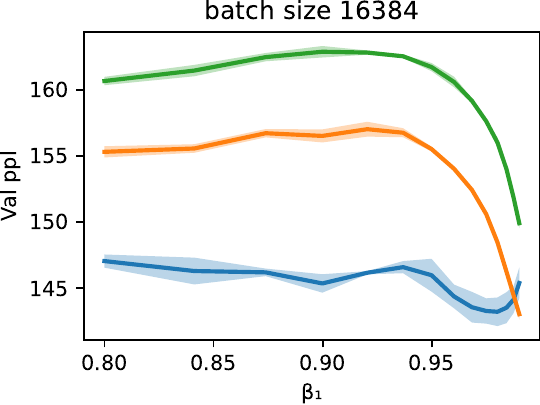}
\end{longtable}}

Similarly to the $\beta_2$ sweep, we also plot below the sharpness metrics.

\subsubsection{Sharpness Metrics at Batch Size 8}

{\setlength{\tabcolsep}{0pt}%
\begin{longtable}{ccc}
  \includegraphics[width=0.320\textwidth]{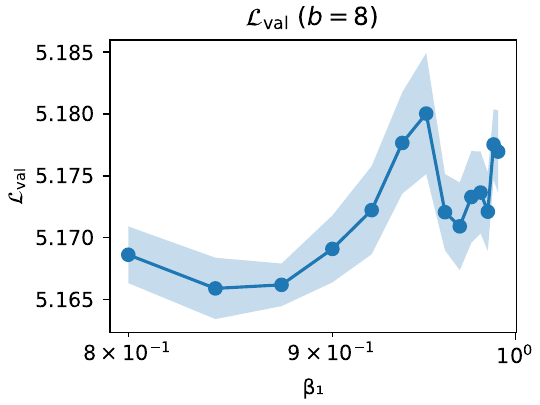}
  &
  \includegraphics[width=0.320\textwidth]{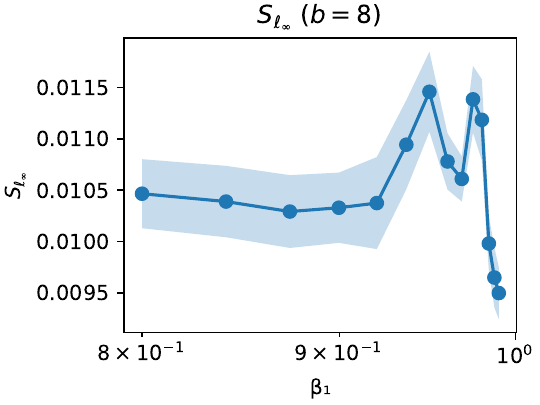}
  &
  \includegraphics[width=0.320\textwidth]{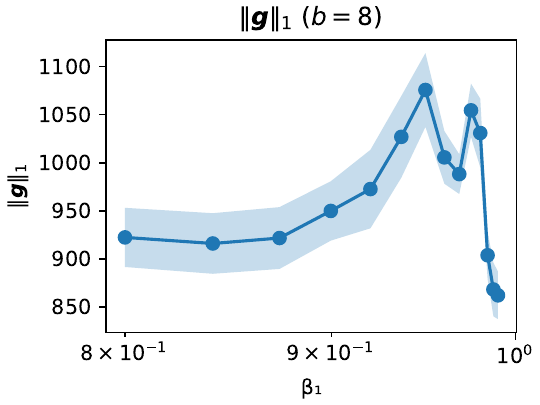}
  \\
  \includegraphics[width=0.320\textwidth]{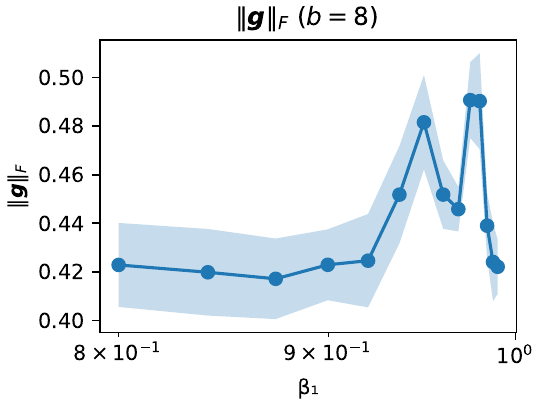}
  &
  \includegraphics[width=0.320\textwidth]{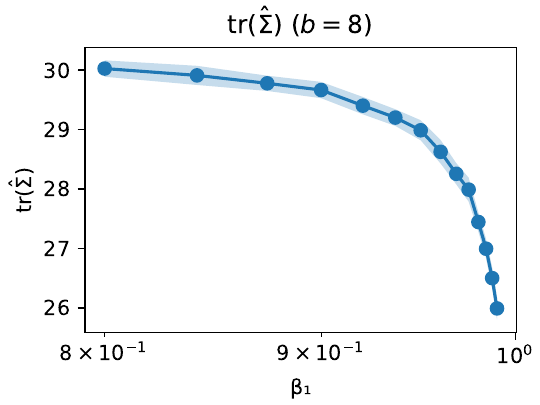}
  & \includegraphics[width=0.320\textwidth]{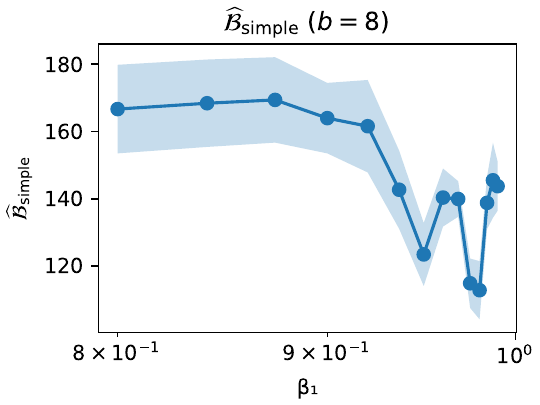}
  \\
  \includegraphics[width=0.320\textwidth]{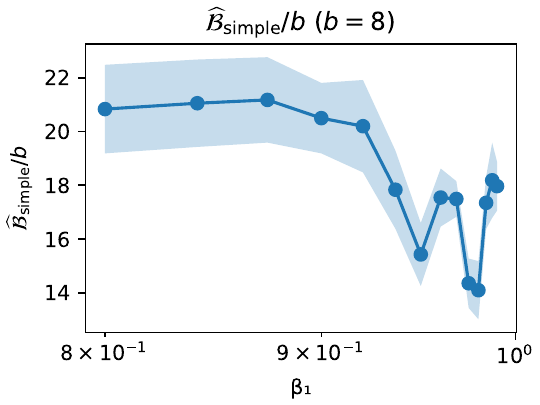}
  & \includegraphics[width=0.320\textwidth]{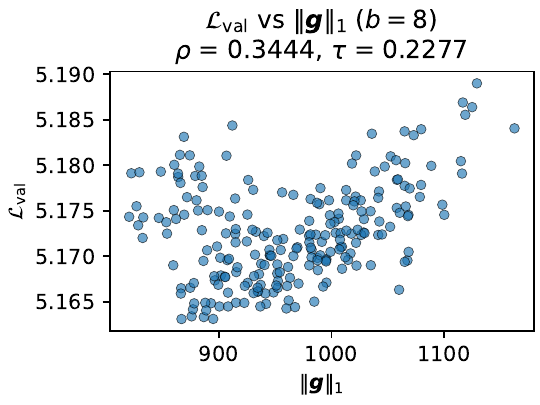}
  & \includegraphics[width=0.320\textwidth]{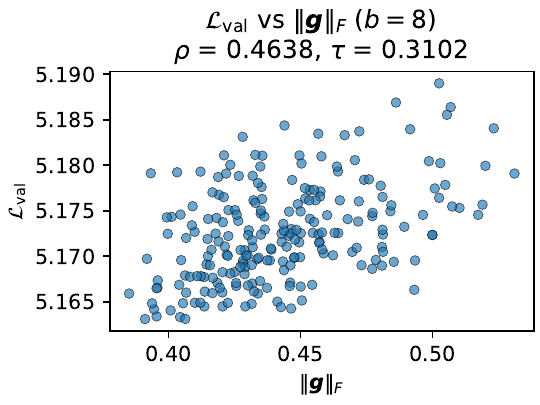}
  \\ \includegraphics[width=0.320\textwidth]{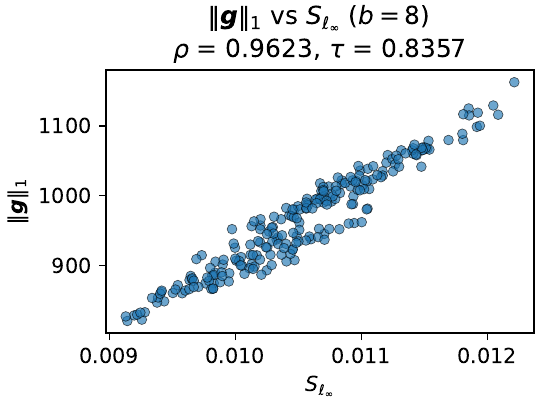}
  & \includegraphics[width=0.320\textwidth]{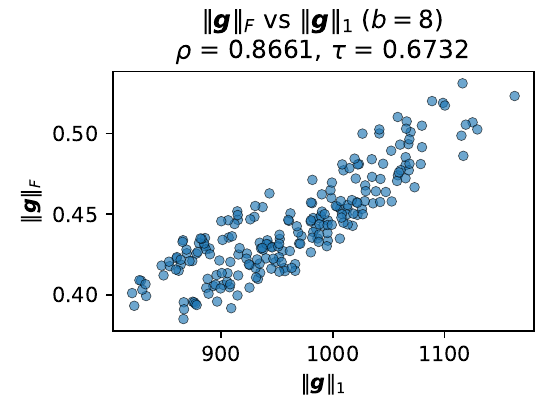}
  &
\end{longtable}}

\subsubsection{Sharpness Metrics at Batch Size 4096}

{\setlength{\tabcolsep}{0pt}%
\begin{longtable}{ccc}
  \includegraphics[width=0.320\textwidth]{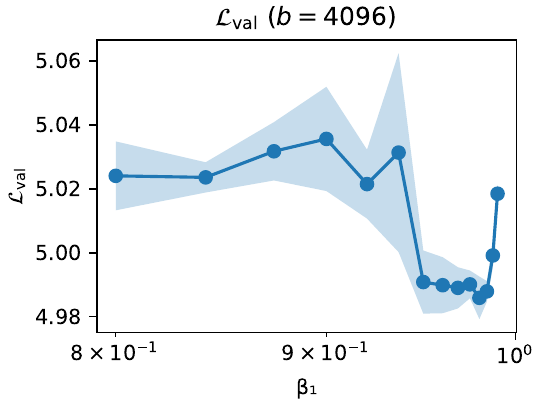}
  &
  \includegraphics[width=0.320\textwidth]{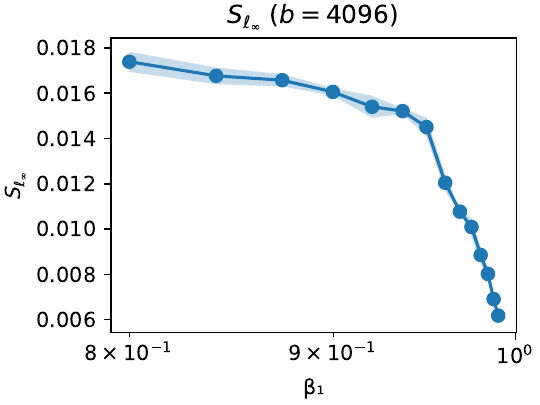}
  &
  \includegraphics[width=0.320\textwidth]{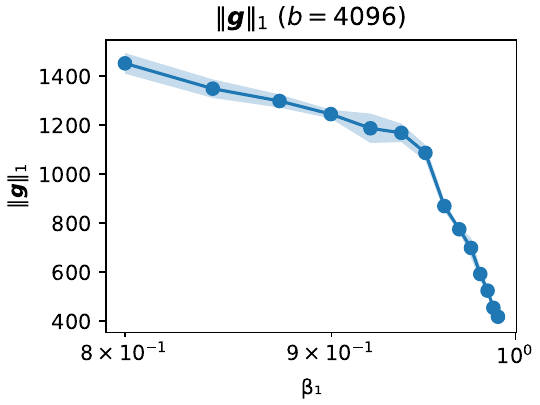}
  \\
  \includegraphics[width=0.320\textwidth]{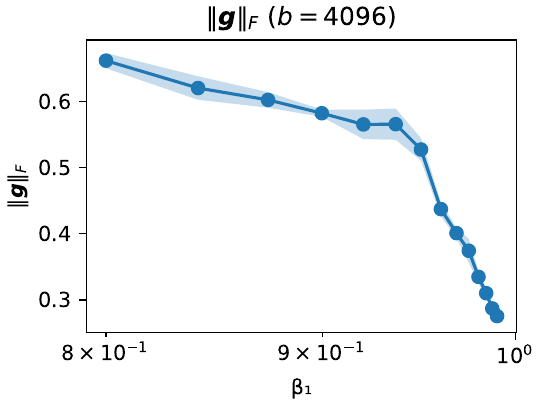}
  &
  \includegraphics[width=0.320\textwidth]{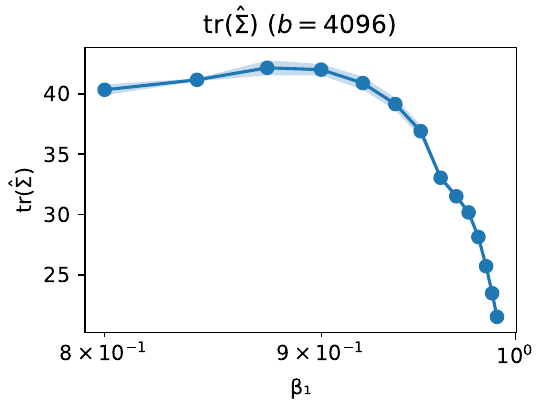}
  & \includegraphics[width=0.320\textwidth]{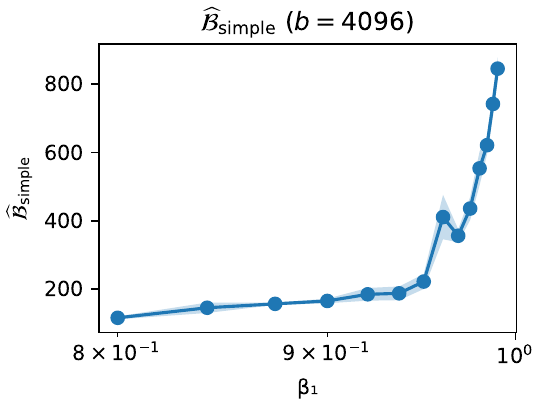}
  \\
  \includegraphics[width=0.320\textwidth]{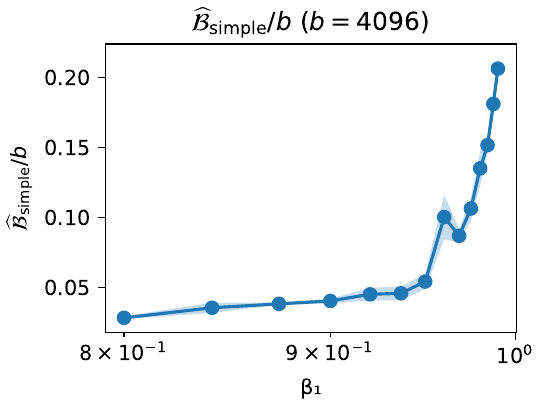}
  & \includegraphics[width=0.320\textwidth]{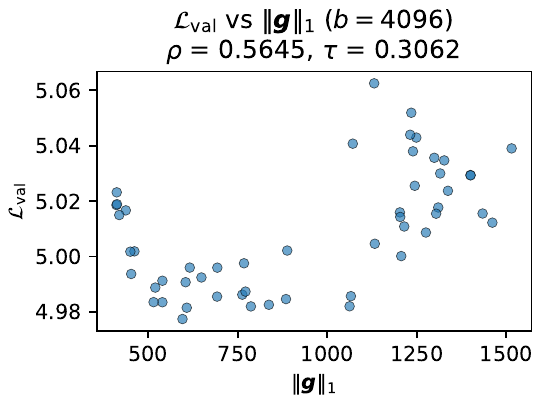}
  & \includegraphics[width=0.320\textwidth]{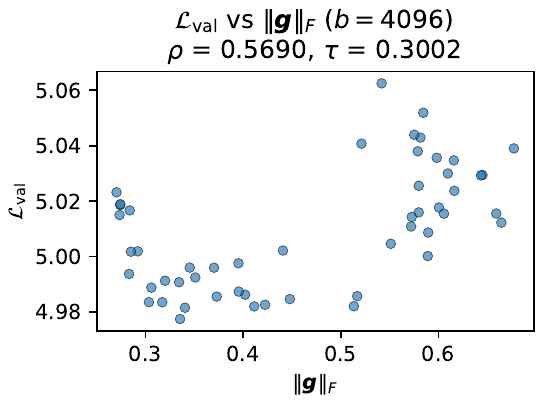}
  \\ \includegraphics[width=0.320\textwidth]{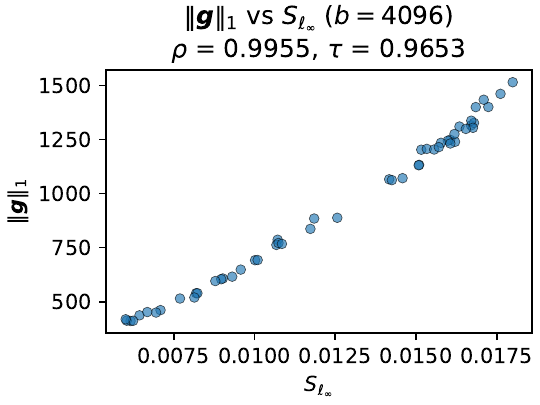}
  & \includegraphics[width=0.320\textwidth]{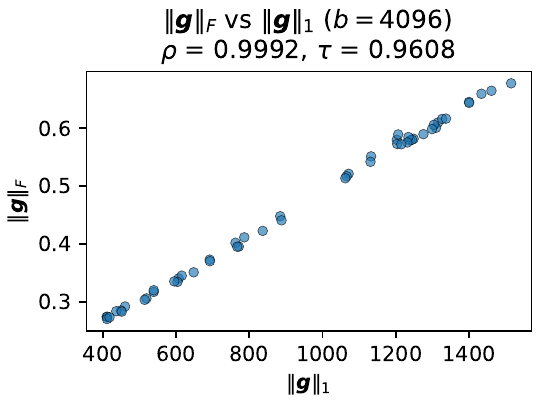}
  &
\end{longtable}}

\subsection{Llama3 on DCLM}

Llama3 with 12 layers and 12 heads, head dimension 64, sequence length 4096 is trained on DCLM for 6\,144\,000 sequences (25\,B tokens, or about 90 tokens per parameter).
No weight tying is used, and the total number of parameters is around 280\,M.
The dataset is tokenized with the Llama 3.1 tokenizer.
We isolate the first $2^{24}$ documents of DCLM as a validation set,
and choose the first 1024 sequences (4\,M tokens) for calculating the validation loss and other metrics.
AdamW is used with $\epsilon = 10^{-8}$ and decoupled weight decay $\lambda = 0.1$ (PyTorch parametrization) applied to all but the normalization parameters.
The number of training steps is chosen to make a full pass over the sequences
(96\,000 steps for $b = 64$ and 6\,000 steps for $b = 1024$). We plot validation loss and sharpness metrics at the last step, along with scatterplots showing correlations
($\rho$ in the title denotes the Pearson correlation coefficient, $\tau$ denotes Kendall's tau).

{\setlength{\tabcolsep}{0pt}%
  \begin{longtable}{ccc}
    \includegraphics[width=0.320\textwidth]{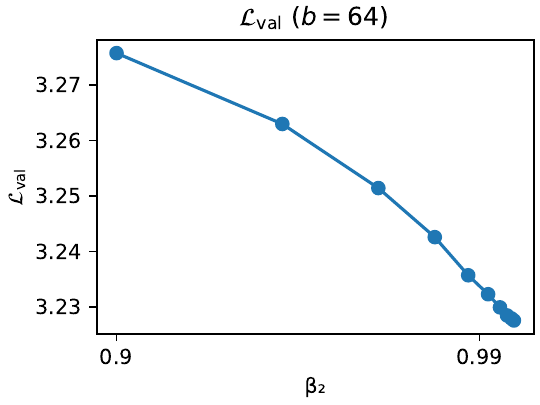}
  &
  \includegraphics[width=0.320\textwidth]{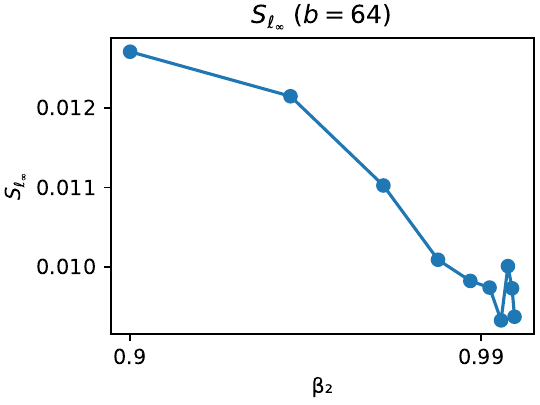}
  &
  \includegraphics[width=0.320\textwidth]{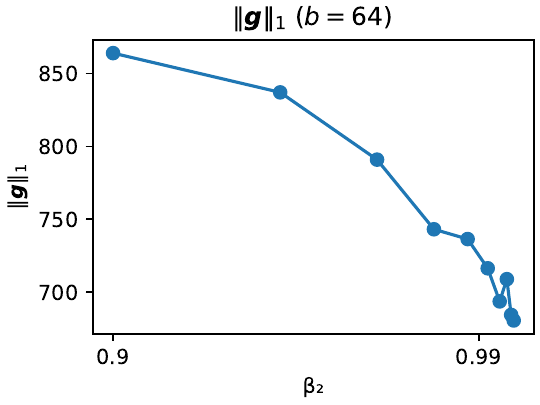}
  \\
  \includegraphics[width=0.320\textwidth]{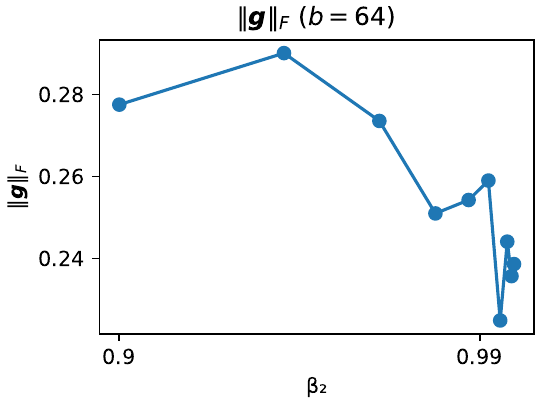}
  &
    \includegraphics[width=0.320\textwidth]{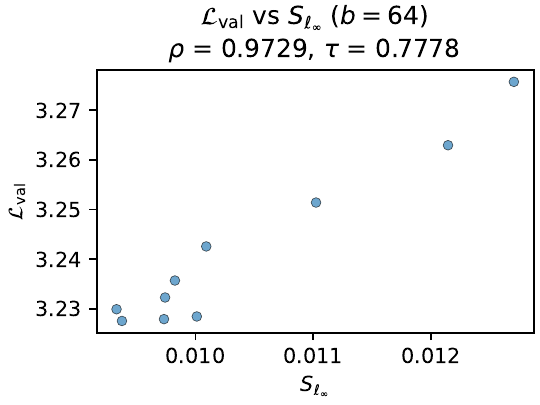}
  &
    \includegraphics[width=0.320\textwidth]{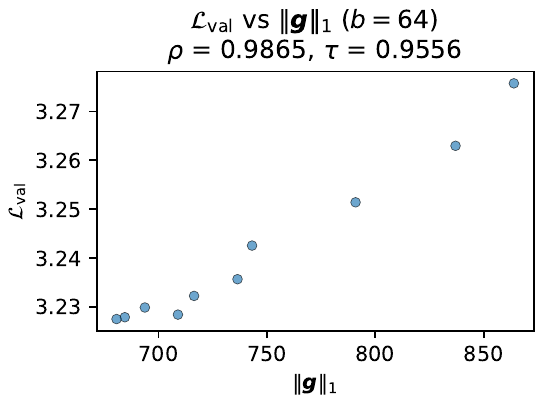}
  \\
  \includegraphics[width=0.320\textwidth]{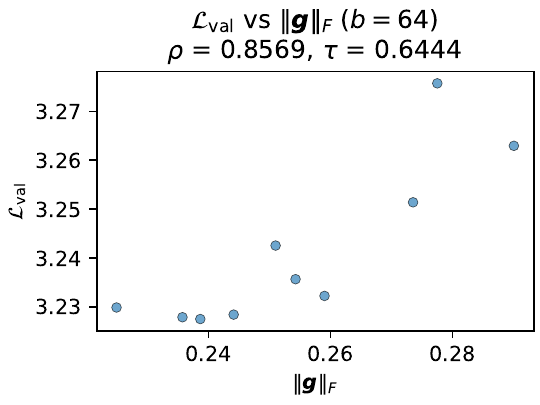}
  &
  \includegraphics[width=0.320\textwidth]{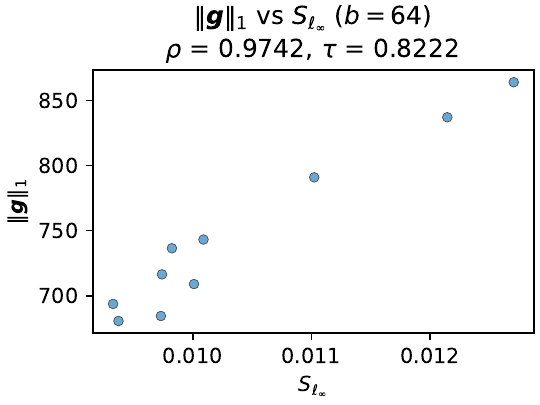}
  &
    \includegraphics[width=0.320\textwidth]{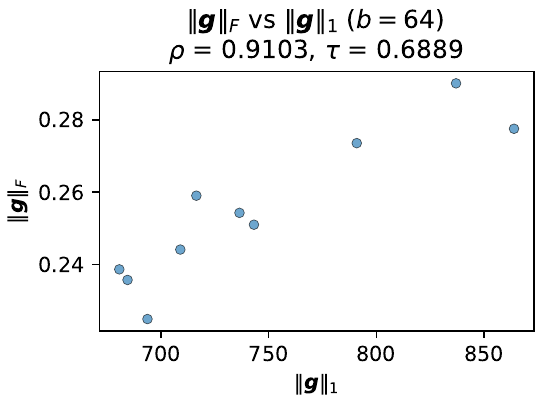}
 \\
  \includegraphics[width=0.320\textwidth]{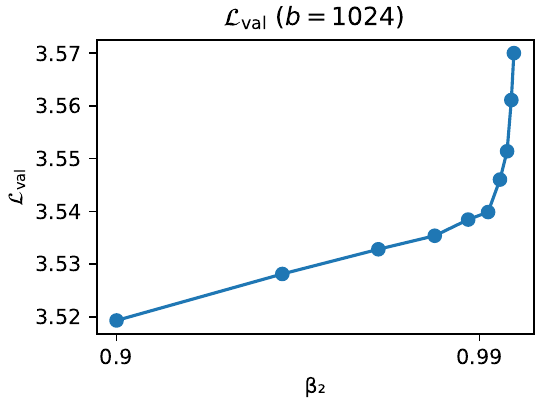}
  &
  \includegraphics[width=0.320\textwidth]{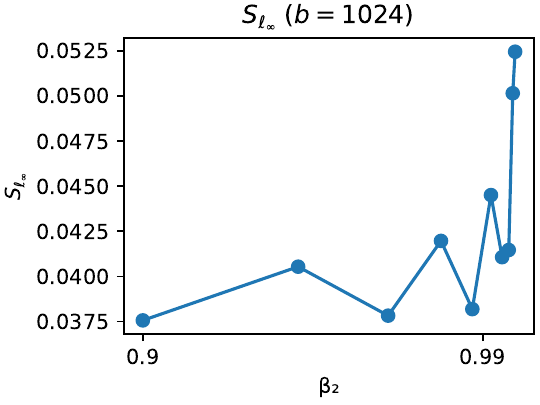}
  &
  \includegraphics[width=0.320\textwidth]{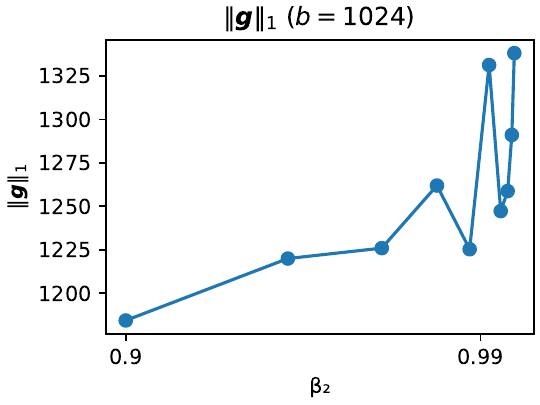}
  \\
  \includegraphics[width=0.320\textwidth]{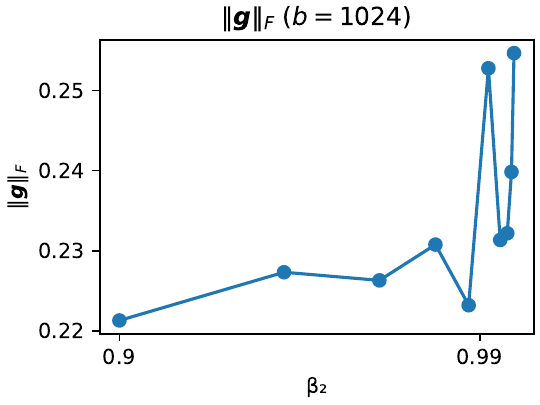}
  &
    \includegraphics[width=0.320\textwidth]{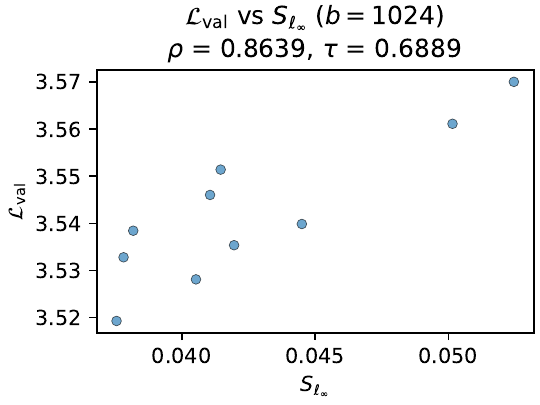}
  &
    \includegraphics[width=0.320\textwidth]{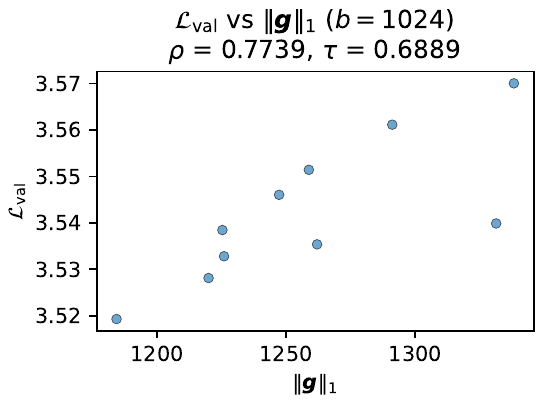}
  \\
  \includegraphics[width=0.320\textwidth]{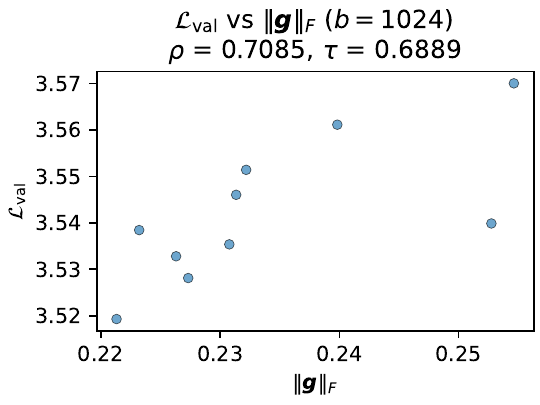}
  &
  \includegraphics[width=0.320\textwidth]{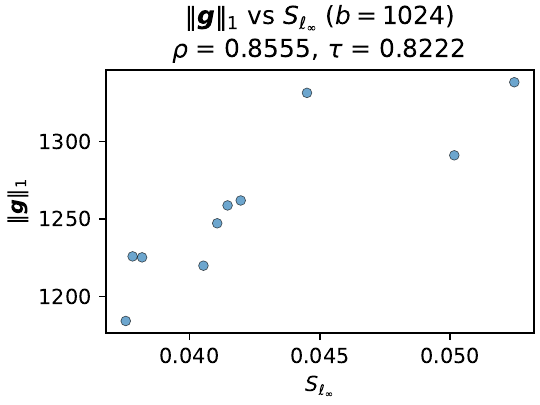}
  &
  \includegraphics[width=0.320\textwidth]{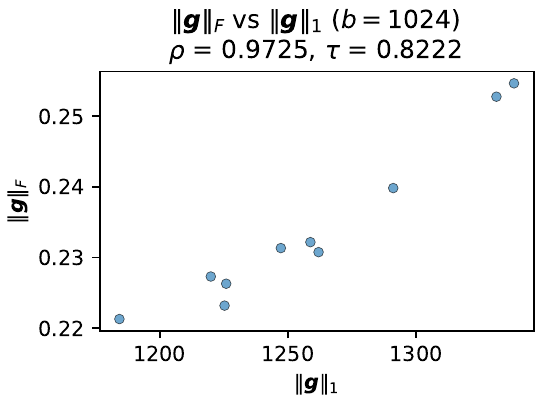}
\end{longtable}}

\section{Main Theorems}\label{sec:general-theorems}

We start with the full statement of the memory removal step (\cref{sec:removing-memory}), from which technical expressions were omitted in the main part of this article.

\begin{theorem}[Memory removal]\label{thm:applying-memory-removal}
Let $\Theta \subset \mathbb{R}^{\dim \vtheta}$ be an open convex domain of parameters $\vtheta$ of interest, and assume $\ell_r(\cdot) \in \mathcal{C}^3(\Theta; \mathbb{R})$ with
\begin{equation*}
\sup_{\vtheta \in \Theta} \sup_N \max_{1 \leq r \leq N} \max_{1 \leq s \leq 3} \norm{\nabla^s \ell_r(\vtheta)} < \infty,
\end{equation*}
where $\norm{\cdot}$ is the corresponding operator norm.
Let the iteration $\crl[\big]{\vtheta_t}_{t = 0}^{\infty}$ be given by Adam, \cref{def:adam},
and the iteration $\crl[\big]{\tilde{\vtheta}_t}_{t = 0}^{\infty}$ be given by
\begin{align*}
  &\tilde{\vtheta}_{t + 1} = \tilde{\vtheta}_t
    - \eta \, \main{t}(\tilde{\vtheta}_t) - \eta^2 \, \corr{t}(\tilde{\vtheta}_t), \quad \tilde{\vtheta}_0 = \vtheta_0,\\
  \text{with} \quad &\mainsc{j}{t}(\btheta) := \frac{\sum_{k = 0}^t \mu_{t, k} \partial_j \mathcal{L}_k(\btheta)}{\sqrt{\sum_{k = 0}^t \nu_{t, k} \abs{\partial_j \mathcal{L}_k(\btheta)}^2 + \epsilon}},\numberthis\label{eq:main-term-adam-def}\\
  &\corrsc{j}{t}(\btheta) := \frac{L_{t, j}(\btheta)}{R_{t, j}(\btheta)} - \frac{M_{t, j}(\btheta) P_{t, j}(\btheta)}{R_{t, j}(\btheta)^3},\numberthis\label{eq:corr-term-adam-def}\\
  &M_{t, j}(\btheta) := \sum_{k = 0}^t \mu_{t, k} \partial_j \mathcal{L}_k(\btheta),\\
  &R_{t, j}(\btheta) := \sqrt{\sum_{k = 0}^t \nu_{t, k} \abs{\partial_j \mathcal{L}_k(\btheta)}^2 + \epsilon},\\
  &L_{t, j}(\btheta) := \sum_{k = 0}^{t - 1} \mu_{t, k} \sum_{i = 1}^{\dim \btheta} \partial_{i j} \mathcal{L}_k(\btheta) \sum_{l = k}^{t - 1} \frac{M_{l, i}(\btheta)}{R_{l, i}(\btheta)},\\
  &P_{t, j}(\btheta) := \sum_{k = 0}^{t - 1} \nu_{t, k} \partial_j \mathcal{L}_k(\btheta) \sum_{i = 1}^{\dim \btheta} \partial_{i j} \mathcal{L}_k(\btheta) \sum_{l = k}^{t - 1} \frac{M_{l, i}(\btheta)}{R_{l, i}(\btheta)}.
\end{align*}
Then, for any constant ``physical time'' horizon $T > 0$, the following global error bound holds:
\begin{align*}
    \max_{t \in \range{0}{\lfloor T / \eta \rfloor}} \norm[\big]{\vtheta_t - \tilde{\vtheta}_t}_{\infty} \leq C \eta^2
\end{align*}
for some constant $C = C_T$ not depending on $\eta$.
\end{theorem}

Next, we state the full version of the mini-batch noise expansions.

\begin{theorem}[Mini-batch noise expansion of the memoryless dynamics]\label{thm:finite-n-nonzero-eps}
In the setting of \cref{thm:applying-memory-removal}, for every
\(j \in \range{1}{\dim \btheta}\),
  \begin{equation}\label{eq:expect-eps}
    \begin{aligned}
      \sqrt{g_j^2+\epsilon}\,\Eletpi \corrsc{j}{n}
      ={}& \fbneps{j}
      + \mbnoneneps{j}
      + \mbntwoneps{j}\\
      &+ \mbnthreeneps{j}
      + \mbnfourneps{j}
      + \mbnfiveneps{j}\\
      &+ O(d^3) + o_n(b^{-1}),
    \end{aligned}
  \end{equation}
  where
  \[
    \fbneps{j}
    :=
    \left(
      \sum_{k=0}^{n-1}\mu_{n,k}(n-k)
      -
      \frac{g_j^2}{g_j^2+\epsilon}
      \sum_{k=0}^{n-1}\nu_{n,k}(n-k)
    \right)
    \partial_j \|\vg\|_{1,\epsilon},
  \]
  and the five mini-batch-noise corrections are given by
  \begin{align*}
    \mbnoneneps{j}
    :=&\
    \frac{n}{(n + 1) b-1}\,
    \partial_j \|\vg\|_{1,\epsilon}\,\Sigma_{jj}\,
    A^{(n,\epsilon)}_{j},
    \\
    \mbntwoneps{j}
    :=&\
    \frac{n}{(n + 1) b-1}
    \sum_i
    \frac{\hess{i}{j}g_i}{(g_i^2+\epsilon)^{5/2}}
    \Sigma_{ii}\,
    B^{(n,\epsilon)}_{i,j},
    \\
    \mbnthreeneps{j}
    :=&\
    \frac{n}{2((n + 1) b - 1)}
    \frac{g_j}{g_j^2+\epsilon}\\
    & \times \prn[\bigg]{
      -2 \sum_{k=0}^{n-1}(n-k)\mu_{n,k}\nu_{n,k}
      - \sum_{k=0}^{n-1}(n-k)\nu_{n,k}
      + 3\frac{g_j^2}{g_j^2+\epsilon}\sum_{k=0}^{n-1}(n-k)\nu_{n,k}^2
    }\\
    & \times \sum_i \frac{g_i}{\sqrt{g_i^2+\epsilon}}\,\partial_i \Sigma_{jj},
    \\
    \mbnfourneps{j}
    :=&\
    \frac{n}{2((n + 1) b-1)}
    \sum_i
    \frac{1}{\sqrt{g_i^2+\epsilon}}
    D^{(n,\epsilon)}_{i,j}\,
    \partial_j \Sigma_{ii},
    \\
    \mbnfiveneps{j}
    :=&\
    \frac{n}{(n + 1)b-1}
    \frac{g_j}{g_j^2+\epsilon}
    \sum_i
    \frac{\hess{i}{j}}{\sqrt{g_i^2+\epsilon}}
    E^{(n,\epsilon)}_{i,j}\,
    \Sigma_{ij}.
  \end{align*}
  Here
  \begin{align*}
    A^{(n,\epsilon)}_{j}
    :={}&
      \frac{1}{2 (g_j^2+\epsilon)^2}
      \sum_{r = 0}^{n}
      \prn*{3 g_j^2 \nu_{n, r}^2 - (g_j^2+\epsilon)\nu_{n, r}}
      \sum_{k = 0}^{n - 1} \mu_{n, k} (n-k)
\\
    &- \frac{g_j^2}{(g_j^2+\epsilon)^3}
      \sum_{r = 0}^{n - 1} (n-r)\nu_{n,r}
      \sum_{k = 0}^{n}
      \prn*{4 g_j^2 \nu_{n, k}^2 - (g_j^2+\epsilon)\nu_{n, k}
      - 2 (g_j^2+\epsilon)\mu_{n, k}\nu_{n, k}}
\\
    &- \frac{1}{(g_j^2+\epsilon)^2}
      \sum_{k = 0}^{n - 1} (n-k)\nu_{n,k}
      \prn*{(g_j^2+\epsilon)\mu_{n,k} - 2 g_j^2 \nu_{n,k}}
\\
    &+ \frac{g_j^2}{(g_j^2+\epsilon)^3}
      \sum_{p = 0}^{n - 1} (n-p)\nu_{n,p}
      \sum_{k = 0}^{n}\nu_{n,k}
      \prn*{(g_j^2+\epsilon)\mu_{n,k} - 2 g_j^2 \nu_{n,k}}
\\
    &- \frac{g_j^2}{2 (g_j^2+\epsilon)^3}
      \sum_{r = 0}^{n - 1} (n-r)\nu_{n,r}
      \sum_{k = 0}^{n}
      \prn*{3 g_j^2 \nu_{n, k}^2 - (g_j^2+\epsilon)\nu_{n, k}}
\\
    &\ + \frac{g_j^2}{(g_j^2+\epsilon)^2}
      \sum_{k = 0}^{n - 1}(n-k)\nu_{n,k}^2,\\
    B^{(n,\epsilon)}_{i,j}
    :={}&
      \frac12
      \sum_{l=0}^{n-1}\sum_{k,p=0}^{l}
      \mu_{n,k}
      \prn*{3 g_i^2\nu_{l,p}^2-(g_i^2+\epsilon)\nu_{l,p}
      -2(g_i^2+\epsilon)\mu_{l,p}\nu_{l,p}}
\\
    &- \frac{g_j^2}{2(g_j^2+\epsilon)}
      \sum_{l=0}^{n-1}\sum_{k,p=0}^{l}
      \nu_{n,k}
      \prn*{3 g_i^2\nu_{l,p}^2-(g_i^2+\epsilon)\nu_{l,p}
      -2(g_i^2+\epsilon)\mu_{l,p}\nu_{l,p}},\\
    D^{(n,\epsilon)}_{i,j}
    :={}&
      \sum_{l=0}^{n-1}\sum_{k=0}^{l}
      \mu_{n,k}
        \brk*{\mu_{l,k}-\frac{g_i^2}{g_i^2+\epsilon}\nu_{l,k}}
     - \frac{g_j^2}{g_j^2+\epsilon}
      \sum_{l=0}^{n-1}\sum_{k=0}^{l}
      \nu_{n,k}
      \brk*{\mu_{l,k}-\frac{g_i^2}{g_i^2+\epsilon}\nu_{l,k}},\\
    E^{(n,\epsilon)}_{i,j}
    :={}&
      - \sum_{l=0}^{n-1}\sum_{k,p=0}^{l}
        \mu_{n,k}\nu_{n,p}
        \brk*{\mu_{l,p}-\frac{g_i^2}{g_i^2+\epsilon}\nu_{l,p}}
\\
    &- \frac{1}{g_j^2+\epsilon}
        \sum_{l=0}^{n-1}\sum_{k,p=0}^{l}
        \nu_{n,k}
        \brk*{\mu_{l,p}-\frac{g_i^2}{g_i^2+\epsilon}\nu_{l,p}}
        \prn*{(g_j^2+\epsilon)\mu_{n,p}-2g_j^2\nu_{n,p}}
\\
    &+ \frac{g_j^2}{g_j^2+\epsilon}
        \sum_{l=0}^{n-1}\sum_{k,p=0}^{l}
        \nu_{n,k}
        \brk*{\mu_{l,p}-\frac{g_i^2}{g_i^2+\epsilon}\nu_{l,p}}
        \nu_{n,p}
\\
    &- \sum_{l=0}^{n-1}\sum_{k=0}^{l}
        \nu_{n,k}
        \brk*{\mu_{l,k}-\frac{g_i^2}{g_i^2+\epsilon}\nu_{l,k}}.
  \end{align*}
\end{theorem}

The following theorem provides limits of the above expressions which is useful for interpreting them.

\begin{theorem}[Limits of full-batch and mini-batch corrections]\label{thm:limits}\leavevmode
  \begin{assertions}[wide]
\item\label{thm:large-n-limits-eps} The finite-\(n\) quantities in
\cref{thm:finite-n-nonzero-eps} satisfy, as \(n\to\infty\),
\begin{equation*}
\fbneps{j}
=
\fbinfeps{j} + o_n(1),
\qquad
\mbnrneps{j}
=
\mbnrinfeps{j} + o_n(b^{-1}),
\qquad r \in \range{1}{5},
\end{equation*}
where
\begin{align*}
\fbinfeps{j}
:=&\
\left(
\frac{\beta_1}{1-\beta_1}
-
\frac{g_j^2}{g_j^2+\epsilon}\frac{\beta_2}{1-\beta_2}
\right)\partial_j\|\vg\|_{1,\epsilon},
\\[1ex]
\mbnoneinfeps{j}
:=&\
\frac{1}{b}\,
\partial_j\|\vg\|_{1,\epsilon}\,\Sigma_{jj}\,
A^{(\infty,\epsilon)}_j,
\\[1ex]
\mbntwoinfeps{j}
:=&\
\frac{1}{b}
\sum_i
\frac{\hess{i}{j}g_i}{(g_i^2+\epsilon)^{5/2}}
\Sigma_{ii}\,
B^{(\infty,\epsilon)}_{i,j},
\\[1ex]
\mbnthreeinfeps{j}
:=&\
\frac{1}{2b}
    \frac{g_j}{g_j^2+\epsilon}
\prn[\bigg]{
-2\frac{\beta_1\beta_2(1-\beta_1)(1-\beta_2)}
{(1-\beta_1\beta_2)^2}
-\frac{\beta_2}{1-\beta_2}
+3\frac{g_j^2}{g_j^2+\epsilon}\frac{\beta_2^2}{(1+\beta_2)^2}
    }
  \\
&\times
\sum_i \frac{g_i}{\sqrt{g_i^2+\epsilon}}\,\partial_i\Sigma_{jj},
\\[1ex]
\mbnfourinfeps{j}
:=&\
\frac{1}{2b}
\sum_i
\frac{1}{\sqrt{g_i^2+\epsilon}}
D^{(\infty,\epsilon)}_{i,j}\,
\partial_j\Sigma_{ii},
\\[1ex]
\mbnfiveinfeps{j}
:=&\
\frac{1}{b}
\frac{g_j}{g_j^2+\epsilon}
\sum_i
\frac{\hess{i}{j}}{\sqrt{g_i^2+\epsilon}}
E^{(\infty,\epsilon)}_{i,j}\,
\Sigma_{ij}.
\end{align*}
The limiting coefficient functions are
\begin{align*}
A^{(\infty,\epsilon)}_j
:=&\
\frac{\beta_1}{2(1-\beta_1)(g_j^2+\epsilon)}
\left(
3\frac{g_j^2}{g_j^2+\epsilon}\frac{1-\beta_2}{1+\beta_2}
-1
\right)
\\
&
-\frac{g_j^2}{(g_j^2+\epsilon)^2}
\frac{\beta_2}{1-\beta_2}
\left(
4\frac{g_j^2}{g_j^2+\epsilon}\frac{1-\beta_2}{1+\beta_2}
-1
-\frac{2(1-\beta_1)(1-\beta_2)}{1-\beta_1\beta_2}
\right)
\\
&
-\frac{1}{g_j^2+\epsilon}
\left(
\frac{\beta_1\beta_2(1-\beta_1)(1-\beta_2)}
{(1-\beta_1\beta_2)^2}
-2\frac{g_j^2}{g_j^2+\epsilon}
\frac{\beta_2^2}{(1+\beta_2)^2}
\right)
\\
&
+\frac{g_j^2}{(g_j^2+\epsilon)^2}
\frac{\beta_2}{1-\beta_2}
\left(
\frac{(1-\beta_1)(1-\beta_2)}{1-\beta_1\beta_2}
-2\frac{g_j^2}{g_j^2+\epsilon}\frac{1-\beta_2}{1+\beta_2}
\right)
\\
&
-\frac{g_j^2}{2(g_j^2+\epsilon)^2}
\frac{\beta_2}{1-\beta_2}
\left(
3\frac{g_j^2}{g_j^2+\epsilon}\frac{1-\beta_2}{1+\beta_2}
-1
\right)
+\frac{g_j^2}{(g_j^2+\epsilon)^2}
\frac{\beta_2^2}{(1+\beta_2)^2},
\\[1ex]
B^{(\infty,\epsilon)}_{i,j}
:=&\
\frac{g_i^2+\epsilon}{2}
\left(
3\frac{g_i^2}{g_i^2+\epsilon}\frac{1-\beta_2}{1+\beta_2}
-1
-\frac{2(1-\beta_1)(1-\beta_2)}{1-\beta_1\beta_2}
\right)
\\
&\times
\left(
\frac{\beta_1}{1-\beta_1}
-\frac{g_j^2}{g_j^2+\epsilon}\frac{\beta_2}{1-\beta_2}
\right),
\\[1ex]
D^{(\infty,\epsilon)}_{i,j}
:=&\
\frac{\beta_1}{1+\beta_1}
-\frac{g_i^2}{g_i^2+\epsilon}
\frac{\beta_1(1-\beta_2)}{1-\beta_1\beta_2}
-\frac{g_j^2}{g_j^2+\epsilon}
\frac{\beta_2(1-\beta_1)}{1-\beta_1\beta_2}
\\
&
+\frac{g_i^2}{g_i^2+\epsilon}
\frac{g_j^2}{g_j^2+\epsilon}
\frac{\beta_2}{1+\beta_2},
\\[1ex]
E^{(\infty,\epsilon)}_{i,j}
:=&\
-\frac{\beta_1\beta_2(1-\beta_1)(1-\beta_2)}
{(1-\beta_1\beta_2)^2}
+\frac{g_i^2}{g_i^2+\epsilon}
\frac{\beta_1\beta_2(1-\beta_2)}
{(1+\beta_2)(1-\beta_1\beta_2)}
\\
&
-\frac{\beta_1\beta_2(1-\beta_1)}
{(1+\beta_1)(1-\beta_1\beta_2)}
+\frac{g_i^2}{g_i^2+\epsilon}
\frac{\beta_1\beta_2(1-\beta_1)(1-\beta_2)}
{(1-\beta_1\beta_2)^2}
\\
&
+3\frac{g_j^2}{g_j^2+\epsilon}
\frac{\beta_2^2(1-\beta_1)}
{(1+\beta_2)(1-\beta_1\beta_2)}
-3\frac{g_i^2}{g_i^2+\epsilon}
\frac{g_j^2}{g_j^2+\epsilon}
\frac{\beta_2^2}{(1+\beta_2)^2}
\\
&
-\frac{\beta_2(1-\beta_1)}{1-\beta_1\beta_2}
+\frac{g_i^2}{g_i^2+\epsilon}\frac{\beta_2}{1+\beta_2}.
\end{align*}

\item\label{thm:large-n-zero-eps-simple-version}
Further, the limits of the full-batch and mini-batch noise expansions as $\epsilon \to 0$
\begin{equation*}
\fbinfeps{j} = \fb{j} + o_{\epsilon}(1), \quad \mbnrinfeps{j} = \mbnr{j} + o_{\epsilon}(1), \quad r \in \range{1}{5},
\end{equation*}
are given by \cref{eq:fb-simple,eq:mbn-simple} with
\begin{align*}
  C_1(\beta_1, \beta_2) :=&\
  \frac{1 - \beta_1^2}{\beta_1 (1 - \beta_1 \beta_2)}
  + \frac{(1 - \beta_1)^2}{\beta_1 (1 - \beta_1 \beta_2)^2}
  + \frac{3 (1 + \beta_1)}{2 (1 - \beta_1) (1 + \beta_2)}
  \\
  &- \frac{2}{\beta_1 (1 - \beta_1)}
  + \frac{3}{2 - 2 \beta_2}
  + \frac{3}{(1 + \beta_2)^2}
  - 2,
  \\
  C_2(\beta_1, \beta_2) :=&\
  \frac{(\beta_1 - \beta_2)
  (\beta_1 \beta_2^2 - \beta_1 \beta_2 + \beta_1 + \beta_2^2 - 2 \beta_2)}
  {(1 - \beta_1) (1 - \beta_2) (1 + \beta_2) (1 - \beta_1 \beta_2)},
  \\
  C_3(\beta_1, \beta_2) :=&\
  \frac{1}{2(1 - \beta_2)(1 + \beta_2)^2 (1 - \beta_1 \beta_2)^2}
  \crl[\big]{
  -2 \beta_1(\beta_1 + 1) \beta_2^5
  + (\beta_1^2 + 8 \beta_1) \beta_2^4
  \\
  &
  + (2 \beta_1 - 5 \beta_1^2 - 4) \beta_2^3
  + (2 \beta_1 + 1) \beta_2^2
  + (2 \beta_1^2 - 2 \beta_1 - 1) \beta_2
  },
  \\
  C_4(\beta_1, \beta_2) :=&
  - \frac{(\beta_1 - \beta_2)^2}
  {2(1 + \beta_1) (1 + \beta_2) (1 - \beta_1 \beta_2)},
  \\
  C_5(\beta_1, \beta_2) :=&\
  \frac{\beta_2 (\beta_2 - \beta_1) (2 \beta_2 - 3 \beta_1 - 1)}
  {(1 + \beta_1) (1 + \beta_2)^2 (1 - \beta_1 \beta_2)}.
  \numberthis\label{eq:constants}
\end{align*}
\end{assertions}
\end{theorem}

\section{Proof of \cref{thm:applying-memory-removal}}\label{sec:proof-of-thm-applying-memory-removal}

This result is taken from \citet{cattaneo2025howmemoryoptimization}.
Specifically, it is a special case of the following general theorem,
a reformulation of their Corollary 3.3.

\begin{theorem}[General memory removal theorem]\label{th:UCVELF}
  Let $\Theta$ be an open convex domain in $\mathbb{R}^{\dim \vtheta}$ and $\vF_t \in C^2(\Theta^{t + 1}; \mathbb{R}^d)$ be a family of functions, such that for any $t \in \mathbb{Z}_{\geq 0}$, $k_1, k_2 \in \range{0}{t}$, $r, i, j \in \range{1}{\dim \vtheta}$,
\begin{equation*}
\abs{F_{t, r}} \leq \gamma_{-1}, \quad \abs[\bigg]{\frac{\partial F_{t, r}}{\partial \theta_{t - k_1, i}}} \leq \gamma_{k_1}, \quad \abs[\bigg]{\frac{\partial^2 F_{t, r}}{\partial \theta_{t - k_1, i} \partial \theta_{t - k_2, j}}} \leq \gamma_{k_1, k_2},
\end{equation*}
where $\gamma_{-1}$, $\gamma_{k_1}$ and $\gamma_{k_1, k_2}$ are families of positive reals (not depending on $t$) satisfying $\sum_{k_1 = 1}^{\infty} \gamma_{k_1} k_1^2 + \sum_{k_1, k_2 = 1}^{\infty} \gamma_{k_1, k_2} k_1 k_2 < \infty$ (sufficiently fast decay of memory).
Let $T \geq 0$ be a fixed ``physical time'' horizon.
Then iterations $\crl[\big]{\vtheta_t}_{t = 0}^{\infty}$ and $\crl[\big]{\tilde{\vtheta}_t}_{t = 0}^{\infty}$, given in \cref{eq:quzxvB,eq:HsXmgC,eq:tetoJR} with the same initial condition $\tilde{\vtheta}_0 = \vtheta_0$, satisfy
\begin{align*}
    \max_{t \in \range{0}{\lfloor T / \eta \rfloor}} \norm[\big]{\vtheta_t - \tilde{\vtheta}_t}_{\infty} \leq C \eta^2
\end{align*}
for some constant $C$ depending on $T$ but not depending on $\eta$.
\end{theorem}

The proof of \cref{thm:applying-memory-removal} is a direct application of \cref{th:UCVELF}, with the function $\boldsymbol{F}_t$ given by
\[
  F_{t,j}(\btheta_t,\ldots,\btheta_0)
  :=
  \frac{
    \sum_{k=0}^t \mu_{t,k}\partial_j\mathcal L_k(\btheta_k)
  }{
    \sqrt{
      \sum_{k=0}^t \nu_{t,k}|\partial_j\mathcal L_k(\btheta_k)|^2+\epsilon
    }
  }.
\]
We check the assumptions below.

By the boundedness assumption on
the per-sample losses, there exist constants \(G,H,K<\infty\) such that, uniformly in
the batch index and in \(\btheta\in\Theta\),
\[
  \|\nabla \mathcal L_k(\btheta)\|\le G,\qquad
  \|\nabla^2 \mathcal L_k(\btheta)\|\le H,\qquad
  \|\nabla^3 \mathcal L_k(\btheta)\|\le K.
\]
Since
$\sum_{k=0}^t \mu_{t,k}=1$
and $\sum_{k=0}^t \nu_{t,k}=1$, we have a bound
$|F_{t,j}(\btheta_t,\ldots,\btheta_0)|
\le
G / \sqrt{\epsilon}.
$
Hence the zeroth-order bound in \cref{th:UCVELF} holds with, for example,
$
  \gamma_{-1}:= G / \sqrt{\epsilon}.
$

Now fix \(a\in\range{0}{t}\). Differentiating \(F_{t,j}\) with respect to the
coordinate \(\theta_{a,i}\) gives
\[
  \frac{\partial F_{t,j}}{\partial \theta_{a,i}}
  =
  \frac{\mu_{t,a}\partial_{ij}\mathcal L_a(\btheta_a)}{\prn[\big]{
      \sum_{k=0}^t \nu_{t,k}|\partial_j\mathcal L_k(\btheta_k)|^2+\epsilon
    }^{1 / 2}}
  -
  \sum_{k=0}^t \mu_{t,k}\partial_j\mathcal L_k(\btheta_k) \frac{\nu_{t,a}\partial_j\mathcal L_a(\btheta_a)
  \partial_{ij}\mathcal L_a(\btheta_a)}{\prn[\big]{
      \sum_{k=0}^t \nu_{t,k}|\partial_j\mathcal L_k(\btheta_k)|^2+\epsilon
    }^{3 / 2}}.
\]
Using the bounds above,
\[
  \left|
  \frac{\partial F_{t,j}}{\partial \theta_{a,i}}
  \right|
  \le
  C_\epsilon(\mu_{t,a}+\nu_{t,a})
\]
for a constant \(C_\epsilon\) depending only on \(G,H,\epsilon\), but not on \(t\) or
\(a\). If \(a=t-q\), then
\[
  \mu_{t,t-q}
  =
  \frac{(1-\beta_1)\beta_1^q}{1-\beta_1^{t+1}}
  \le
  \beta_1^q,
  \qquad
  \nu_{t,t-q}
  =
  \frac{(1-\beta_2)\beta_2^q}{1-\beta_2^{t+1}}
  \le
  \beta_2^q.
\]
Thus
\[
  \left|
  \frac{\partial F_{t,j}}{\partial \theta_{t-q,i}}
  \right|
  \le
  C_\epsilon(\beta_1^q+\beta_2^q).
\]
So we may take
\[
  \gamma_q := C_\epsilon(\beta_1^q+\beta_2^q).
\]
Since \(0<\beta_1,\beta_2<1\),
\[
  \sum_{q=1}^\infty q^2\gamma_q<\infty.
\]

Similarly, differentiating once more, every second derivative of \(F_{t,j}\) is a
finite linear combination of terms of the following schematic forms:
\[
  \mathbf 1_{\{a=b\}}\mu_{t,a},
  \qquad
  \mathbf 1_{\{a=b\}}\nu_{t,a},
  \qquad
  \mu_{t,a}\nu_{t,b},
  \qquad
  \nu_{t,a}\mu_{t,b},
  \qquad
  \nu_{t,a}\nu_{t,b},
\]
multiplied by bounded derivatives of the losses and by powers of \(\prn[\big]{
      \sum_{k=0}^t \nu_{t,k}|\partial_j\mathcal L_k(\btheta_k)|^2+\epsilon
    }^{- 1 / 2}\).
The latter are uniformly bounded because \(R_{t,j}\ge \sqrt{\epsilon}\). Therefore, if
\(a=t-q_1\) and \(b=t-q_2\), then
\[
  \left|
  \frac{\partial^2 F_{t,j}}
  {\partial \theta_{t-q_1,i}\partial \theta_{t-q_2,r}}
  \right|
  \le
  C_\epsilon
  \left(
    \mathbf 1_{\{q_1=q_2\}}(\beta_1^{q_1}+\beta_2^{q_1})
    +(\beta_1^{q_1}+\beta_2^{q_1})(\beta_1^{q_2}+\beta_2^{q_2})
  \right).
\]
Thus we may choose
\[
  \gamma_{q_1,q_2}
  :=
  C_\epsilon
  \left(
    \mathbf 1_{\{q_1=q_2\}}(\beta_1^{q_1}+\beta_2^{q_1})
    +(\beta_1^{q_1}+\beta_2^{q_1})(\beta_1^{q_2}+\beta_2^{q_2})
  \right).
\]
Then
\[
  \sum_{q_1,q_2=1}^{\infty}\gamma_{q_1,q_2}q_1q_2<\infty,
\]
again because \(0<\beta_1,\beta_2<1\). Hence all hypotheses of
\cref{th:UCVELF} are satisfied.

The main and correction terms from \cref{eq:main-term-adam-def,eq:corr-term-adam-def} are obtained by directly using \cref{eq:tetoJR} (see also \citet{cattaneo2025howmemoryoptimization}).

\section{Proof of \cref{thm:finite-n-nonzero-eps}}\label{sec:proof-of-prop-expect}

The plan is to first expand $\corrsc{j}{n}(\btheta)$ up to degree-2 monomials in noise derivatives (that is, up to $O(d^2)$) and then calculate $\Epi{\cdot}$ of the result.

\subsection{Expanding the Correction up to Quadratic Terms in Noise}\label{sec:expandion-corr-quadr-noise}

\begin{proposition}[Expansion of the correction up to quadratic terms in noise]\label{prop:exp-corr-quadratic-noise}
  The additive components $L_{n, j}(\btheta) / R_{n, j}(\btheta)$ and $M_{n, j}(\btheta) P_{n, j}(\btheta) / R_{n, j}(\btheta)^3$ of the correction defined in \cref{eq:corr-term-adam-def} admit the following formal expansion up to $O(d^2)$ and vanishing quantities as $\epsilon \to 0$:
\begin{align*}
  L_{n, j}(\btheta) R_{n, j}(\btheta)^{-1} = &\comp[\big]{L_{n, j}(\btheta) R_{n, j}(\btheta)^{-1}}{0} + \comp[\big]{L_{n, j}(\btheta) R_{n, j}(\btheta)^{-1}}{1} + \comp[\big]{L_{n, j}(\btheta) R_{n, j}(\btheta)^{-1}}{2}\\
  &+ O(d^3),
\end{align*}
where\footnote{We skip the monomials of degree exactly 1 in noise derivatives because they are mean-zero and will not influence the expectation $\Epi{\cdot}$.}
{\small
\begin{align*}
  &\comp[\big]{L_{n, j}(\btheta) R_{n, j}(\btheta)^{-1}}{0}
    = \frac{\partial_j \|\vg\|_{1,\epsilon}}{\sqrt{g_j^2+\epsilon}}
      \sum_{k = 0}^{n - 1} \mu_{n, k} (n - k),\\
  &\comp[\big]{L_{n, j}(\btheta) R_{n, j}(\btheta)^{-1}}{1} =
\text{[skipped]},\\
  &\comp[\big]{L_{n, j}(\btheta) R_{n, j}(\btheta)^{-1}}{2}\\
  &\quad =\frac{1}{\sqrt{g_j^2+\epsilon}}
      \sum_i \frac{\hess{i}{j}g_i}{2 (g_i^2+\epsilon)^{5/2}}
      \sum_{l = 0}^{n - 1} \sum_{k, p = 0}^l
      \mu_{n, k}
      \prn*{3 g_i^2 \nu_{l, p}^2 - (g_i^2+\epsilon)\nu_{l, p}
      - 2 (g_i^2+\epsilon)\mu_{l, p}\nu_{l, p}}
    (\noise[i]{p})^2\\
  &\qquad + \frac{1}{\sqrt{g_j^2+\epsilon}}
      \sum_i \frac{1}{\sqrt{g_i^2+\epsilon}}
      \sum_{l = 0}^{n - 1} \sum_{k, p = 0}^l
      \mu_{n, k}
      \brk*{\mu_{l, p} - \frac{g_i^2}{g_i^2+\epsilon}\nu_{l, p}}
      \noise[i j]{k} \noise[i]{p}\\
  &\qquad + \frac{1}{\sqrt{g_j^2+\epsilon}}
      \sum_i \frac{\hess{i}{j}g_i}{(g_i^2+\epsilon)^{5/2}}
\\
  &\phantom{\qquad + \frac{1}{\sqrt{g_j^2+\epsilon}}
    \sum_i}
    \times \sum_{l = 0}^{n - 1} \sum_{k = 0}^l \sum_{0 \le p < q \le l}
      \mu_{n, k}
      \prn*{3 g_i^2 \nu_{l, p}\nu_{l, q}
      - (g_i^2+\epsilon)\mu_{l, p}\nu_{l, q}
      - (g_i^2+\epsilon)\mu_{l, q}\nu_{l, p}}
      \noise[i]{p} \noise[i]{q}\\
  &\qquad - \frac{g_j}{(g_j^2+\epsilon)^{3/2}}
      \sum_{r = 0}^{n} \nu_{n, r}
      \sum_i \frac{\hess{i}{j}}{\sqrt{g_i^2+\epsilon}}
      \sum_{l = 0}^{n - 1} \sum_{k, p = 0}^l
      \mu_{n, k}
      \brk*{\mu_{l, p} - \frac{g_i^2}{g_i^2+\epsilon}\nu_{l, p}}
      \noise[i]{p} \noise[j]{r}\\
  &\qquad - \frac{g_j}{(g_j^2+\epsilon)^{3/2}}
      \sum_{r = 0}^{n} \nu_{n, r}
      \sum_i \frac{g_i}{\sqrt{g_i^2+\epsilon}}
      \sum_{k = 0}^{n - 1} \mu_{n, k} (n - k)
      \noise[i j]{k} \noise[j]{r}\\
  &\qquad + \frac{\partial_j \|\vg\|_{1,\epsilon}}{2 (g_j^2+\epsilon)^{5/2}}
      \sum_{r = 0}^{n}
      \prn*{3 g_j^2 \nu_{n, r}^2 - (g_j^2+\epsilon)\nu_{n, r}}
      \sum_{k = 0}^{n - 1} \mu_{n, k} (n - k)
      (\noise[j]{r})^2\\
  &\qquad + \frac{3 g_j^2 \partial_j \|\vg\|_{1,\epsilon}}{(g_j^2+\epsilon)^{5/2}}
      \sum_{k = 0}^{n - 1} \mu_{n, k} (n - k)
      \sum_{0 \le p < q \le n} \nu_{n, p}\nu_{n, q}\, \noise[j]{p} \noise[j]{q},
 \end{align*}}
and
\begin{align*}
  \frac{M_{n, j}(\btheta) P_{n, j}(\btheta)}{R_{n, j}(\btheta)^3} = &\comp[\bigg]{\frac{M_{n, j}(\btheta) P_{n, j}(\btheta)}{R_{n, j}(\btheta)^3}}{0} + \comp[\bigg]{\frac{M_{n, j}(\btheta) P_{n, j}(\btheta)}{R_{n, j}(\btheta)^3}}{1} + \comp[\bigg]{\frac{M_{n, j}(\btheta) P_{n, j}(\btheta)}{R_{n, j}(\btheta)^3}}{2}\\
  &+ O(d^3),
\end{align*}
where
{\small\begin{align*}
  &\comp[\bigg]{\frac{M_{n, j}(\btheta) P_{n, j}(\btheta)}{R_{n, j}(\btheta)^3}}{0} :=
    \frac{g_j^2}{(g_j^2+\epsilon)^{3/2}}
      \partial_j \norm{\vg}_{1,\epsilon}
    \sum_{k = 0}^{n - 1} (n - k) \nu_{n, k},\\
  &\comp[\bigg]{\frac{M_{n, j}(\btheta) P_{n, j}(\btheta)}{R_{n, j}(\btheta)^3}}{1}
   := \text{[skipped]},\\
  &\comp[\bigg]{\frac{M_{n, j}(\btheta) P_{n, j}(\btheta)}{R_{n, j}(\btheta)^3}}{2}\\
  &\quad :=
\frac{g_j^2 \partial_j \norm{\vg}_{1,\epsilon}}{(g_j^2+\epsilon)^{7/2}}
    \sum_{r = 0}^{n - 1} (n - r) \nu_{n, r}
      \sum_{k = 0}^n
      \prn*{4 g_j^2 \nu_{n, k}^2 - (g_j^2+\epsilon)\nu_{n, k}
      - 2 (g_j^2+\epsilon)\mu_{n, k}\nu_{n, k}} (\noise[j]{k})^2
\\
  &\qquad +
      \frac{2 g_j^2 \partial_j \norm{\vg}_{1,\epsilon}}{(g_j^2+\epsilon)^{7/2}}
    \sum_{r = 0}^{n - 1} (n - r) \nu_{n, r}
  \\
  &\phantom{\qquad +
      \frac{2 g_j^2 \partial_j \norm{\vg}_{1,\epsilon}}{(g_j^2+\epsilon)^{7/2}}
    \sum_{r = 0}^{n - 1}}
    \times \sum_{0 \le p < q \le n}
      \prn*{4 g_j^2 \nu_{n, p}\nu_{n, q}
      - (g_j^2+\epsilon)\mu_{n, p}\nu_{n, q}
      - (g_j^2+\epsilon)\mu_{n, q}\nu_{n, p}}
      \noise[j]{p} \noise[j]{q}
\\
  &\qquad +
      \frac{g_j}{(g_j^2+\epsilon)^{5/2}}
      \sum_i \frac{\hess{i}{j}}{\sqrt{g_i^2+\epsilon}}
      \sum_{l = 0}^{n - 1} \sum_{k, p = 0}^l \sum_{r = 0}^n
      \nu_{n, k}
    \brk*{\mu_{l, p} - \frac{g_i^2}{g_i^2+\epsilon}\nu_{l, p}}\\
  &\phantom{\qquad +
      \frac{g_j}{(g_j^2+\epsilon)^{5/2}}
      \sum_i \frac{\hess{i}{j}}{\sqrt{g_i^2+\epsilon}}
    \sum_{l = 0}^{n - 1} \sum_{k, p = 0}^l \sum_{r = 0}^n}
    \times
      \brk*{(g_j^2+\epsilon)\mu_{n, r} - 2 g_j^2 \nu_{n, r}}
      \noise[i]{p} \noise[j]{r}
\\
  &\qquad +
      \frac{g_j}{(g_j^2+\epsilon)^{5/2}}
      \sum_i \frac{g_i}{\sqrt{g_i^2+\epsilon}}
      \sum_{k = 0}^{n - 1} \sum_{r = 0}^n
      (n - k) \nu_{n, k}
      \brk*{(g_j^2+\epsilon)\mu_{n, r} - 2 g_j^2 \nu_{n, r}}
      \noise[i j]{k} \noise[j]{r}
\\
  &\qquad +
      \frac{\partial_j \norm{\vg}_{1,\epsilon}}{(g_j^2+\epsilon)^{5/2}}
      \sum_{k = 0}^{n - 1} \sum_{r = 0}^n
      (n - k) \nu_{n, k}
      \brk*{(g_j^2+\epsilon)\mu_{n, r} - 2 g_j^2 \nu_{n, r}}
      \noise[j]{k} \noise[j]{r}
\\
  &\qquad -
      \frac{g_j^2 \partial_j \norm{\vg}_{1,\epsilon}}{(g_j^2+\epsilon)^{7/2}}
      \sum_{p = 0}^{n - 1} (n - p) \nu_{n, p}
      \sum_{k = 0}^n \sum_{r = 0}^n
      \nu_{n, k}
      \brk*{(g_j^2+\epsilon)\mu_{n, r} - 2 g_j^2 \nu_{n, r}}
      \noise[j]{k} \noise[j]{r}
\\
  &\qquad +
      \frac{g_j^2 \partial_j \norm{\vg}_{1,\epsilon}}{2 (g_j^2+\epsilon)^{7/2}}
      \sum_{r = 0}^{n - 1} (n - r) \nu_{n, r}
      \sum_{k = 0}^n
      \prn*{3 g_j^2 \nu_{n, k}^2 - (g_j^2+\epsilon)\nu_{n, k}}
      (\noise[j]{k})^2
\\
  &\qquad +
      \frac{3 g_j^4 \partial_j \norm{\vg}_{1,\epsilon}}{(g_j^2+\epsilon)^{7/2}}
      \sum_{r = 0}^{n - 1} (n - r) \nu_{n, r}
      \sum_{0 \le p < q \le n}
      \nu_{n, p}\nu_{n, q}\, \noise[j]{p} \noise[j]{q}
\\
  &\qquad -
      \frac{g_j^3}{(g_j^2+\epsilon)^{5/2}}
      \sum_i \frac{\hess{i}{j}}{\sqrt{g_i^2+\epsilon}}
      \sum_{l = 0}^{n - 1} \sum_{k, p = 0}^l
      \nu_{n, k}
      \brk*{\mu_{l, p} - \frac{g_i^2}{g_i^2+\epsilon}\nu_{l, p}}
      \sum_{r = 0}^n \nu_{n, r} \noise[i]{p} \noise[j]{r}
\\
  &\qquad -
      \frac{g_j^3}{(g_j^2+\epsilon)^{5/2}}
      \sum_i \frac{g_i}{\sqrt{g_i^2+\epsilon}}
      \sum_{k = 0}^{n - 1} (n - k) \nu_{n, k}
      \sum_{r = 0}^n \nu_{n, r} \noise[i j]{k} \noise[j]{r}
\\
  &\qquad -
      \frac{g_j^2 \partial_j \norm{\vg}_{1,\epsilon}}{(g_j^2+\epsilon)^{5/2}}
      \sum_{k = 0}^{n - 1} (n - k) \nu_{n, k}
      \sum_{r = 0}^n \nu_{n, r} \noise[j]{k} \noise[j]{r}
\\
  &\qquad +
      \frac{g_j^2}{2 (g_j^2+\epsilon)^{3/2}}
      \sum_i \frac{\hess{i}{j}g_i}{(g_i^2+\epsilon)^{5/2}}
      \sum_{l = 0}^{n - 1} \sum_{k, p = 0}^l
      \nu_{n, k}
      \prn*{3 g_i^2 \nu_{l, p}^2 - (g_i^2+\epsilon)\nu_{l, p}
      - 2 (g_i^2+\epsilon)\mu_{l, p}\nu_{l, p}}
      (\noise[i]{p})^2
\\
  &\qquad +
      \frac{g_j^2}{(g_j^2+\epsilon)^{3/2}}
    \sum_i \frac{\hess{i}{j}g_i}{(g_i^2+\epsilon)^{5/2}}\\
  &\qquad \qquad \qquad
     \times \sum_{l = 0}^{n - 1} \sum_{k = 0}^l \sum_{0 \le p < q \le l}
      \nu_{n, k}
      \prn*{3 g_i^2 \nu_{l, p}\nu_{l, q}
      - (g_i^2+\epsilon)\mu_{l, p}\nu_{l, q}
      - (g_i^2+\epsilon)\mu_{l, q}\nu_{l, p}}
      \noise[i]{p} \noise[i]{q}
\\
  &\qquad +
      \frac{g_j^2}{(g_j^2+\epsilon)^{3/2}}
      \sum_i \frac{1}{\sqrt{g_i^2+\epsilon}}
      \sum_{l = 0}^{n - 1} \sum_{k, p = 0}^l
      \nu_{n, k}
      \brk*{\mu_{l, p} - \frac{g_i^2}{g_i^2+\epsilon}\nu_{l, p}}
      \noise[i j]{k} \noise[i]{p}
\\
  &\qquad +
      \frac{g_j}{(g_j^2+\epsilon)^{3/2}}
      \sum_i \frac{\hess{i}{j}}{\sqrt{g_i^2+\epsilon}}
      \sum_{l = 0}^{n - 1} \sum_{k, p = 0}^l
      \nu_{n, k}
      \brk*{\mu_{l, p} - \frac{g_i^2}{g_i^2+\epsilon}\nu_{l, p}}
      \noise[j]{k} \noise[i]{p}
\\
  &\qquad +
      \frac{g_j}{(g_j^2+\epsilon)^{3/2}}
      \sum_i \frac{g_i}{\sqrt{g_i^2+\epsilon}}
      \sum_{k = 0}^{n - 1} (n - k) \nu_{n, k}
      \noise[j]{k} \noise[i j]{k}.
\end{align*}}
\end{proposition}

The proof is immediate from lemmas collected below.

\begin{proof}
To get the expansion for $L_{n, j}(\btheta) / R_{n, j}(\btheta)$, multiply the expansions for $L_{n, j}(\btheta)$ (from \cref{lem:expansion-LP}) and $R_{n, j}(\btheta)^{-1}$ (from \cref{lem:expansion-MR}).

To get the expansion for $M_{n, j}(\btheta) P_{n, j}(\btheta) / R_{n, j}(\btheta)^3$, multiply the expansions for $P_{n, j}(\btheta) R_{n, j}(\btheta)^{-1}$ and $M_{n, j}(\btheta) R_{n, j}(\btheta)^{-2}$ from \cref{lem:preparation-for-second-corr}.
\end{proof}

Now we state and prove the lemmas.

We start with a very simple expansion separated for pedagogical reasons to illustrate the approach (all following expansions are done similarly).

\begin{lemma}[Illustration of the approach: expansions for $M_{n, j}(\btheta)$ and $R_{n, j}(\btheta)^{-1}$]\label{lem:expansion-MR}
  We have
  \begin{align*}
    M_{n, j}(\btheta) &= g_j + \sum_{k = 0}^n \mu_{n, k} \noise[j]{k},\numberthis\label{eq:exp-m}\\
R_{n,j}(\btheta)^{-1}
&=
(g_j^2+\epsilon)^{-1/2}
-\frac{g_j}{(g_j^2+\epsilon)^{3/2}}
\sum_{k=0}^n \nu_{n,k}\noise[j]{k}
\\
&\quad
+\frac12\sum_{k=0}^n
\left(
\frac{3g_j^2\nu_{n,k}^2}{(g_j^2+\epsilon)^{5/2}}
-\frac{\nu_{n,k}}{(g_j^2+\epsilon)^{3/2}}
\right)(\noise[j]{k})^2
\\
&\quad
+\frac{3g_j^2}{(g_j^2+\epsilon)^{5/2}}
\sum_{0\le p<q\le n}\nu_{n,p}\nu_{n,q}\,\noise[j]{p}\noise[j]{q}
+O(d^3).
  \end{align*}
\end{lemma}

\begin{proof}
\Cref{eq:exp-m} follows directly from definitions. The expansion $R_{n, j}(\btheta)^{-1}$ is obtained by the following chain of equalities:
\begin{align*}
  R_{n, j}(\btheta)^{-1} &= \prn[\bigg]{g_j^2 + \epsilon + 2 g_j \sum_{k = 0}^n \nu_{n, k} \noise[j]{k} + \sum_{k = 0}^n \nu_{n, k} (\noise[j]{k})^2}^{-1 / 2}\\
                          &= (g_j^2 + \epsilon)^{-1 / 2} - (g_j^2 + \epsilon)^{-3 / 2} g_j \sum_{k = 0}^n \nu_{n, k} \noise[j]{k} - \frac{1}{2} (g_j^2 + \epsilon)^{-3 / 2} \sum_{k = 0}^n \nu_{n, k} (\noise[j]{k})^2\\
                          &\quad + \frac{3}{2} (g_j^2 + \epsilon)^{-5 / 2} g_j^2 \prn[\bigg]{\sum_{k = 0}^n \nu_{n, k} \noise[j]{k}}^2 + O(d^3),
\end{align*}
where we used $\sum_{k = 0}^n \nu_{n, k} = 1$.
\end{proof}

\begin{lemma}[Warm-up: expansions for $L_{n, j}(\btheta)$ and $P_{n, j}(\btheta)$]\label{lem:expansion-LP}
  The following formal expansions (up to quadratic terms in noise) hold:
  \begin{align*}
  L_{n, j}(\btheta)
&= \sum_i \hess{i}{j}\frac{g_i}{\sqrt{g_i^2 + \epsilon}} \sum_{k = 0}^n \mu_{n, k} (n - k)\\
&\quad + \sum_i \frac{\hess{i}{j}}{\sqrt{g_i^2 + \epsilon}}
\sum_{l = 0}^{n - 1} \sum_{k, p = 0}^l
\mu_{n, k} \brk*{\mu_{l, p} - \frac{g_i^2}{g_i^2 + \epsilon}\nu_{l, p}} \noise[i]{p}\\
&\quad + \sum_i \frac{g_i}{\sqrt{g_i^2 + \epsilon}}
\sum_{k = 0}^n \mu_{n, k} (n - k) \noise[i j]{k}\\
&\quad + \sum_i \frac{\hess{i}{j}g_i}{2 (g_i^2 + \epsilon)^{5 / 2}}
\sum_{l = 0}^{n - 1} \sum_{k, p = 0}^l
\mu_{n, k}
\prn*{3 g_i^2 \nu_{l, p}^2 - (g_i^2 + \epsilon) \nu_{l, p} - 2 (g_i^2 + \epsilon) \mu_{l, p} \nu_{l, p}} (\noise[i]{p})^2\\
&\quad + \sum_i \frac{1}{\sqrt{g_i^2 + \epsilon}}
\sum_{l = 0}^{n - 1} \sum_{k, p = 0}^l
\mu_{n, k} \brk*{\mu_{l, p} - \frac{g_i^2}{g_i^2 + \epsilon}\nu_{l, p}} \noise[i j]{k} \noise[i]{p}\\
&\quad + \sum_i \frac{\hess{i}{j}g_i}{(g_i^2 + \epsilon)^{5 / 2}}
  \\
    &\qquad \times \sum_{l = 0}^{n - 1} \sum_{k = 0}^l \sum_{0 \leq p < q \leq l}
\mu_{n, k}
\prn*{3 g_i^2 \nu_{l, p} \nu_{l, q} - (g_i^2 + \epsilon) \mu_{l, p} \nu_{l, q} - (g_i^2 + \epsilon) \mu_{l, q} \nu_{l, p}} \noise[i]{p} \noise[i]{q}\\
&\quad + O(d^3),\\
    P_{n, j}(\btheta)
&= g_j \partial_j \|\vg\|_{1,\epsilon} \sum_{k = 0}^{n - 1} (n - k) \nu_{n, k}\\
&\quad + g_j \sum_i \frac{\hess{i}{j}}{\sqrt{g_i^2 + \epsilon}}
\sum_{l = 0}^{n - 1} \sum_{k, p = 0}^l
\nu_{n, k} \brk*{\mu_{l, p} - \frac{g_i^2}{g_i^2 + \epsilon}\nu_{l, p}} \noise[i]{p}\\
&\quad + g_j \sum_i \frac{g_i}{\sqrt{g_i^2 + \epsilon}}
\sum_{k = 0}^{n - 1} (n - k) \nu_{n, k} \noise[i j]{k}\\
&\quad + \partial_j \|\vg\|_{1,\epsilon}
\sum_{k = 0}^{n - 1} (n - k) \nu_{n, k} \noise[j]{k}\\
&\quad + g_j \sum_i \frac{\hess{i}{j}g_i}{2 (g_i^2 + \epsilon)^{5 / 2}}\\
&\qquad \times \sum_{l = 0}^{n - 1} \sum_{k, p = 0}^l
\nu_{n, k}
\prn*{3 g_i^2 \nu_{l, p}^2 - (g_i^2 + \epsilon) \nu_{l, p} - 2 (g_i^2 + \epsilon) \mu_{l, p} \nu_{l, p}} (\noise[i]{p})^2\\
    &\quad + g_j \sum_i \frac{\hess{i}{j}g_i}{(g_i^2 + \epsilon)^{5 / 2}}
    \\
&\qquad \times \sum_{l = 0}^{n - 1} \sum_{k = 0}^l \sum_{0 \leq p < q \leq l}
\nu_{n, k}
\prn*{3 g_i^2 \nu_{l, p} \nu_{l, q} - (g_i^2 + \epsilon) \mu_{l, p} \nu_{l, q} - (g_i^2 + \epsilon) \mu_{l, q} \nu_{l, p}} \noise[i]{p} \noise[i]{q}\\
&\quad + g_j \sum_i \frac{1}{\sqrt{g_i^2 + \epsilon}}
\sum_{l = 0}^{n - 1} \sum_{k, p = 0}^l
\nu_{n, k} \brk*{\mu_{l, p} - \frac{g_i^2}{g_i^2 + \epsilon}\nu_{l, p}} \noise[i j]{k} \noise[i]{p}\\
&\quad + \sum_i \frac{\hess{i}{j}}{\sqrt{g_i^2 + \epsilon}}
\sum_{l = 0}^{n - 1} \sum_{k, p = 0}^l
\nu_{n, k} \brk*{\mu_{l, p} - \frac{g_i^2}{g_i^2 + \epsilon}\nu_{l, p}} \noise[j]{k} \noise[i]{p}\\
&\quad + \sum_i \frac{g_i}{\sqrt{g_i^2 + \epsilon}}
\sum_{k = 0}^{n - 1} (n - k) \nu_{n, k} \noise[j]{k} \noise[i j]{k}\\
&\quad + O(d^3).
\end{align*}
\end{lemma}

\begin{proof}
By the expansion of \cref{lem:expansion-MR}, for every $l \in \range{0}{n-1}$ and coordinate $i$ we have
\begin{align*}
  &M_{l, i}(\btheta) R_{l, i}(\btheta)^{-1}\\
  &\quad = \frac{g_i}{\sqrt{g_i^2 + \epsilon}}
    + \frac{1}{\sqrt{g_i^2 + \epsilon}} \sum_{p = 0}^l
    \brk[\bigg]{\mu_{l, p} - \frac{g_i^2}{g_i^2 + \epsilon}\nu_{l, p}} \noise[i]{p}\\
  &\quad + \frac{g_i}{2 (g_i^2 + \epsilon)^{5/2}} \sum_{p = 0}^l
    \prn[\big]{3 g_i^2 \nu_{l, p}^2 - (g_i^2 + \epsilon)\nu_{l, p} - 2 (g_i^2 + \epsilon)\mu_{l, p}\nu_{l, p}} (\noise[i]{p})^2\\
  &\quad + \frac{g_i}{(g_i^2 + \epsilon)^{5/2}} \sum_{0 \le p < q \le l}
    \prn[\big]{3 g_i^2 \nu_{l, p} \nu_{l, q} - (g_i^2 + \epsilon)\mu_{l, p}\nu_{l, q} - (g_i^2 + \epsilon)\mu_{l, q}\nu_{l, p}} \noise[i]{p} \noise[i]{q}\\
  &\quad + O(d^3).
\end{align*}

Insert this into the definition of $L_{n, j}(\btheta)$.
Using
$\partial_{i j}\mathcal{L}_k(\btheta)=\hess{i}{j}+\noise[i j]{k}$
and keeping only terms up to degree two in the noise variables yields the claimed expansion for $L_{n,j}(\btheta)$ after exchanging the order of summation and using
$\sum_{l = k}^{n - 1} 1 = n-k$.

Similarly, insert the same expansion into the definition of $P_{n, j}(\btheta)$.
Using
\[
\partial_j \mathcal{L}_k(\btheta)=g_j+\noise[j]{k},
\qquad
\partial_{i j}\mathcal{L}_k(\btheta)=\hess{i}{j}+\noise[i j]{k},
\]
and truncating at quadratic order in noise gives the formula for $P_{n,j}(\btheta)$.
The zeroth-order term in $P_{n,j}$ is
\[
g_j \sum_i \hess{i}{j}\frac{g_i}{\sqrt{g_i^2+\epsilon}} \sum_{k = 0}^{n - 1} (n-k)\nu_{n,k}
= g_j \partial_j \|\vg\|_{1,\epsilon} \sum_{k = 0}^{n - 1} (n-k)\nu_{n,k},
\]
since
\[
\partial_j \|\vg\|_{1,\epsilon}
= \sum_i \partial_j \sqrt{g_i^2+\epsilon}
= \sum_i \hess{i}{j} \frac{g_i}{\sqrt{g_i^2+\epsilon}}.
\]
This proves the result.
\end{proof}

\begin{lemma}[Preparation: expansions for $P_{n, j}(\btheta) R_{n, j}(\btheta)^{-1}$ and $M_{n, j}(\btheta) R_{n, j}(\btheta)^{-2}$]\label{lem:preparation-for-second-corr}
  We have
\begin{align*}
  P_{n, j}(\btheta) R_{n, j}(\btheta)^{-1} =& \comp[\big]{P_{n, j}(\btheta) R_{n, j}(\btheta)^{-1}}{0} + \comp[\big]{P_{n, j}(\btheta) R_{n, j}(\btheta)^{-1}}{1} + \comp[\big]{P_{n, j}(\btheta) R_{n, j}(\btheta)^{-1}}{2}\\
  &+ O(d^3),
\end{align*}
where
\begin{align*}
  &\comp[\big]{P_{n, j}(\btheta) R_{n, j}(\btheta)^{-1}}{0}
  := \frac{g_j}{\sqrt{g_j^2+\epsilon}} \partial_j \norm{\vg}_{1,\epsilon} \sum_{k = 0}^{n - 1} (n - k) \nu_{n, k},\\
  &\comp[\big]{P_{n, j}(\btheta) R_{n, j}(\btheta)^{-1}}{1}
    \\
  &\quad := \frac{g_j}{\sqrt{g_j^2+\epsilon}}
      \sum_i \frac{\hess{i}{j}}{\sqrt{g_i^2+\epsilon}}
      \sum_{l = 0}^{n - 1} \sum_{k, p = 0}^l
      \nu_{n, k}
      \brk*{\mu_{l, p} - \frac{g_i^2}{g_i^2+\epsilon}\nu_{l, p}}
      \noise[i]{p}
\\
  &\quad + \frac{g_j}{\sqrt{g_j^2+\epsilon}}
      \sum_i \frac{g_i}{\sqrt{g_i^2+\epsilon}}
      \sum_{k = 0}^{n - 1} (n - k) \nu_{n, k} \noise[i j]{k}
\\
  &\quad + \frac{\partial_j \norm{\vg}_{1,\epsilon}}{\sqrt{g_j^2+\epsilon}}
      \sum_{k = 0}^{n - 1} (n - k) \nu_{n, k} \noise[j]{k}
\\
  &\quad - \frac{g_j^2 \partial_j \norm{\vg}_{1,\epsilon}}{(g_j^2+\epsilon)^{3/2}}
      \sum_{r = 0}^{n - 1} (n - r) \nu_{n, r}
      \sum_{k = 0}^n \nu_{n, k} \noise[j]{k}\\
  &\comp[\big]{P_{n, j}(\btheta) R_{n, j}(\btheta)^{-1}}{2}
  \\
  &\quad := \frac{g_j \partial_j \norm{\vg}_{1,\epsilon}}{2 (g_j^2+\epsilon)^{5/2}}
      \sum_{r = 0}^{n - 1} (n - r) \nu_{n, r}
      \sum_{k = 0}^n
      \prn*{3 g_j^2 \nu_{n, k}^2 - (g_j^2+\epsilon)\nu_{n, k}}
      (\noise[j]{k})^2
\\
  &\quad + \frac{g_j \partial_j \norm{\vg}_{1,\epsilon}}{(g_j^2+\epsilon)^{5/2}}
      \sum_{r = 0}^{n - 1} (n - r) \nu_{n, r}
      \sum_{0 \leq p < q \leq n}
      3 g_j^2 \nu_{n, p} \nu_{n, q} \, \noise[j]{p} \noise[j]{q}
\\
  &\quad - \frac{g_j^2}{(g_j^2+\epsilon)^{3/2}}
      \sum_i \frac{\hess{i}{j}}{\sqrt{g_i^2+\epsilon}}
      \sum_{l = 0}^{n - 1} \sum_{k, p = 0}^l
      \nu_{n, k}
      \brk*{\mu_{l, p} - \frac{g_i^2}{g_i^2+\epsilon}\nu_{l, p}}
      \sum_{r = 0}^n \nu_{n, r} \noise[i]{p} \noise[j]{r}
\\
  &\quad - \frac{g_j^2}{(g_j^2+\epsilon)^{3/2}}
      \sum_i \frac{g_i}{\sqrt{g_i^2+\epsilon}}
      \sum_{k = 0}^{n - 1} (n - k) \nu_{n, k}
      \sum_{r = 0}^n \nu_{n, r} \noise[i j]{k} \noise[j]{r}
\\
  &\quad - \frac{g_j \partial_j \norm{\vg}_{1,\epsilon}}{(g_j^2+\epsilon)^{3/2}}
      \sum_{k = 0}^{n - 1} (n - k) \nu_{n, k}
      \sum_{r = 0}^n \nu_{n, r} \noise[j]{k} \noise[j]{r}
\\
  &\quad + \frac{g_j}{\sqrt{g_j^2+\epsilon}}
    \sum_i \frac{\hess{i}{j}g_i}{2 (g_i^2+\epsilon)^{5/2}}\\
  &\qquad \times
      \sum_{l = 0}^{n - 1} \sum_{k, p = 0}^l
      \nu_{n, k}
      \prn*{3 g_i^2 \nu_{l, p}^2 - (g_i^2+\epsilon)\nu_{l, p}
      - 2 (g_i^2+\epsilon)\mu_{l, p}\nu_{l, p}}
      (\noise[i]{p})^2
\\
  &\quad + \frac{g_j}{\sqrt{g_j^2+\epsilon}}
    \sum_i \frac{\hess{i}{j}g_i}{(g_i^2+\epsilon)^{5/2}}\\
  &\qquad \times
      \sum_{l = 0}^{n - 1} \sum_{k = 0}^l
      \sum_{0 \leq p < q \leq l}
      \nu_{n, k}
      \prn*{3 g_i^2 \nu_{l, p} \nu_{l, q}
      - (g_i^2+\epsilon)\mu_{l, p}\nu_{l, q}
      - (g_i^2+\epsilon)\mu_{l, q}\nu_{l, p}}
      \noise[i]{p} \noise[i]{q}
\\
  &\quad + \frac{g_j}{\sqrt{g_j^2+\epsilon}}
      \sum_i \frac{1}{\sqrt{g_i^2+\epsilon}}
      \sum_{l = 0}^{n - 1} \sum_{k, p = 0}^l
      \nu_{n, k}
      \brk*{\mu_{l, p} - \frac{g_i^2}{g_i^2+\epsilon}\nu_{l, p}}
      \noise[i j]{k} \noise[i]{p}
\\
  &\quad + \frac{1}{\sqrt{g_j^2+\epsilon}}
      \sum_i \frac{\hess{i}{j}}{\sqrt{g_i^2+\epsilon}}
      \sum_{l = 0}^{n - 1} \sum_{k, p = 0}^l
      \nu_{n, k}
      \brk*{\mu_{l, p} - \frac{g_i^2}{g_i^2+\epsilon}\nu_{l, p}}
      \noise[j]{k} \noise[i]{p}
\\
  &\quad + \frac{1}{\sqrt{g_j^2+\epsilon}}
      \sum_i \frac{g_i}{\sqrt{g_i^2+\epsilon}}
      \sum_{k = 0}^{n - 1} (n - k) \nu_{n, k}
      \noise[j]{k} \noise[i j]{k},
\end{align*}
and
\begin{align*}
  M_{n, j}(\btheta) R_{n, j}(\btheta)^{-2} =& \comp[\big]{M_{n, j}(\btheta) R_{n, j}(\btheta)^{-2}}{0} + \comp[\big]{M_{n, j}(\btheta) R_{n, j}(\btheta)^{-2}}{1} + \comp[\big]{M_{n, j}(\btheta) R_{n, j}(\btheta)^{-2}}{2}\\
  &+ O(d^3),
\end{align*}
where
\begin{align*}
  &\comp[\big]{M_{n, j}(\btheta) R_{n, j}(\btheta)^{-2}}{0} := \frac{g_j}{g_j^2+\epsilon},\\
  &\comp[\big]{M_{n, j}(\btheta) R_{n, j}(\btheta)^{-2}}{1}
    := \frac{1}{(g_j^2+\epsilon)^2}
      \sum_{k = 0}^n \prn[\big]{(g_j^2+\epsilon)\mu_{n, k} - 2 g_j^2 \nu_{n, k}} \noise[j]{k},\\
  &\comp[\big]{M_{n, j}(\btheta) R_{n, j}(\btheta)^{-2}}{2}
  := \frac{g_j}{(g_j^2+\epsilon)^3}
      \sum_{k = 0}^n
      \Bigl(4 g_j^2 \nu_{n, k}^2 - (g_j^2+\epsilon)\nu_{n, k}
      - 2 (g_j^2+\epsilon)\mu_{n, k}\nu_{n, k}\Bigr)(\noise[j]{k})^2
\\
  &\quad + \frac{2 g_j}{(g_j^2+\epsilon)^3}
      \sum_{0 \le p < q \le n}
      \Bigl(4 g_j^2 \nu_{n, p}\nu_{n, q}
      - (g_j^2+\epsilon)\mu_{n, p}\nu_{n, q}
      - (g_j^2+\epsilon)\mu_{n, q}\nu_{n, p}\Bigr)
      \noise[j]{p} \noise[j]{q}.
\end{align*}
\end{lemma}

\begin{proof}
The expansion for $P_{n, j}(\btheta) R_{n, j}(\btheta)^{-1}$ follows by multiplying the expansions for $R_{n, j}(\btheta)^{-1}$ (from \cref{lem:expansion-MR}) and $P_{n, j}(\btheta)$ (from \cref{lem:expansion-LP}).

Raising the expansion for $R_{n, j}(\btheta)^{-1}$ (from \cref{lem:expansion-MR}) to the second power yields
\begin{align*}
  R_{n, j}(\btheta)^{-2}
  =& \frac{1}{g_j^2+\epsilon}\\
&- \frac{2 g_j}{(g_j^2+\epsilon)^2} \sum_{k=0}^{n} \nu_{n, k} \noise[j]{k}\\
&+ \frac{1}{(g_j^2+\epsilon)^3}
\sum_{k = 0}^n \bigl(4 g_j^2 \nu_{n, k}^{2} - (g_j^2+\epsilon)\nu_{n, k}\bigr) (\noise[j]{k})^2\\
&+ \frac{8 g_j^2}{(g_j^2+\epsilon)^3}
\sum_{0 \le p < q \le n} \nu_{n, p} \nu_{n, q} \noise[j]{p} \noise[j]{q}
+ O(d^3).
\end{align*}
Multiplying this by the expansion for $M_{n, j}(\btheta)$ (from \cref{lem:expansion-MR}), we obtain the expansion for $M_{n, j}(\btheta) R_{n, j}(\btheta)^{-2}$, concluding the proof.
\end{proof}

\subsection{Calculating the Expectation of the Result}

Next, we calculate $\Epi{\cdot}$ of the result.

\begin{proposition}[Calculating $\Epi{\cdot}$ of the expansions obtained]\label{prop:e-pi-of-res}
  We have
  {\small
  \begin{align*}
    &\Eletpi \frac{L_{n, j}(\btheta)}{R_{n, j}(\btheta)}
  = \frac{\partial_j \|\vg\|_{1,\epsilon}}{\sqrt{g_j^2+\epsilon}}
   \sum_{k=0}^{n-1}\mu_{n,k}(n-k)\\
&\quad + \frac{1}{\sqrt{g_j^2+\epsilon}}
\sum_i \frac{\hess{i}{j}g_i}{2(g_i^2+\epsilon)^{5/2}}
\sum_{l=0}^{n-1}\sum_{k,p=0}^{l}
\mu_{n,k}
\prn*{3g_i^2\nu_{l,p}^2-(g_i^2+\epsilon)\nu_{l,p}-2(g_i^2+\epsilon)\mu_{l,p}\nu_{l,p}}
\,\Eletpi (\noise[i]{0})^2
\\
&\quad + \frac{1}{\sqrt{g_j^2+\epsilon}}
\sum_i \frac{1}{\sqrt{g_i^2+\epsilon}}
\sum_{l=0}^{n-1}\sum_{k=0}^{l}
\mu_{n,k}
\brk*{\mu_{l,k}-\frac{g_i^2}{g_i^2+\epsilon}\nu_{l,k}}
\,\Eletpi \noise[ij]{0}\noise[i]{0}
\\
&\quad - \frac{g_j}{(g_j^2+\epsilon)^{3/2}}
\sum_i \frac{\hess{i}{j}}{\sqrt{g_i^2+\epsilon}}
\sum_{l=0}^{n-1}\sum_{k,p=0}^{l}
\mu_{n,k}\nu_{n,p}
\brk*{\mu_{l,p}-\frac{g_i^2}{g_i^2+\epsilon}\nu_{l,p}}
\,\Eletpi \noise[i]{0}\noise[j]{0}
\\
&\quad - \frac{g_j}{(g_j^2+\epsilon)^{3/2}}
\sum_i \frac{g_i}{\sqrt{g_i^2+\epsilon}}
\sum_{k=0}^{n-1}\mu_{n,k}(n-k)\nu_{n,k}
\,\Eletpi \noise[ij]{0}\noise[j]{0}
\\
&\quad + \frac{\partial_j\|\vg\|_{1,\epsilon}}{2(g_j^2+\epsilon)^{5/2}}
\sum_{r=0}^{n}\prn*{3g_j^2\nu_{n,r}^2-(g_j^2+\epsilon)\nu_{n,r}}
\sum_{k=0}^{n-1}\mu_{n,k}(n-k)
\,\Eletpi (\noise[j]{0})^2
\\
&\quad + O(d^3)+o_n(b^{-1})
  \end{align*}}
and
{\small
\begin{align*}
  &\Eletpi \frac{M_{n, j}(\btheta) P_{n, j}(\btheta)}{R_{n, j}(\btheta)^3}\\
  &\quad = \frac{g_j^2}{(g_j^2+\epsilon)^{3/2}}
      \partial_j \norm{\vg}_{1,\epsilon}
      \sum_{k = 0}^{n - 1} (n - k) \nu_{n, k}
\\
  &\qquad + \frac{g_j^2 \partial_j \norm{\vg}_{1,\epsilon}}{(g_j^2+\epsilon)^{7/2}}
      \sum_{r = 0}^{n - 1} (n - r) \nu_{n, r}
      \sum_{k = 0}^n
      \prn*{4 g_j^2 \nu_{n, k}^2 - (g_j^2+\epsilon)\nu_{n, k}
      - 2 (g_j^2+\epsilon)\mu_{n, k}\nu_{n, k}}
      \Eletpi (\noise[j]{0})^2
\\
  &\qquad + \frac{g_j}{(g_j^2+\epsilon)^{5/2}} \sum_i \frac{\hess{i}{j}}{\sqrt{g_i^2+\epsilon}}\\
  &\qquad \qquad \times \sum_{l = 0}^{n - 1} \sum_{k, p = 0}^l
      \nu_{n, k}
      \brk*{\mu_{l, p} - \frac{g_i^2}{g_i^2+\epsilon}\nu_{l, p}}
      \brk*{(g_j^2+\epsilon)\mu_{n, p} - 2 g_j^2 \nu_{n, p}}
      \Eletpi \noise[i]{0} \noise[j]{0}
\\
  &\qquad + \frac{g_j}{(g_j^2+\epsilon)^{5/2}} \sum_i \frac{g_i}{\sqrt{g_i^2+\epsilon}}
      \sum_{k = 0}^{n - 1}
      (n - k) \nu_{n, k}
      \brk*{(g_j^2+\epsilon)\mu_{n, k} - 2 g_j^2 \nu_{n, k}}
      \Eletpi \noise[i j]{0} \noise[j]{0}
\\
  &\qquad + \frac{\partial_j \norm{\vg}_{1,\epsilon}}{(g_j^2+\epsilon)^{5/2}}
      \sum_{k = 0}^{n - 1}
      (n - k) \nu_{n, k}
      \brk*{(g_j^2+\epsilon)\mu_{n, k} - 2 g_j^2 \nu_{n, k}}
      \Eletpi (\noise[j]{0})^2
\\
  &\qquad - \frac{g_j^2 \partial_j \norm{\vg}_{1,\epsilon}}{(g_j^2+\epsilon)^{7/2}}
      \sum_{p = 0}^{n - 1} (n - p) \nu_{n, p}
      \sum_{k = 0}^n
      \nu_{n, k}
      \brk*{(g_j^2+\epsilon)\mu_{n, k} - 2 g_j^2 \nu_{n, k}}
      \Eletpi (\noise[j]{0})^2
\\
  &\qquad + \frac{g_j^2 \partial_j \norm{\vg}_{1,\epsilon}}{2 (g_j^2+\epsilon)^{7/2}}
      \sum_{r = 0}^{n - 1} (n - r) \nu_{n, r}
      \sum_{k = 0}^n
      \prn*{3 g_j^2 \nu_{n, k}^2 - (g_j^2+\epsilon)\nu_{n, k}}
      \Eletpi (\noise[j]{0})^2
\\
  &\qquad - \frac{g_j^3}{(g_j^2+\epsilon)^{5/2}} \sum_i \frac{\hess{i}{j}}{\sqrt{g_i^2+\epsilon}}
      \sum_{l = 0}^{n - 1} \sum_{k, p = 0}^l
      \nu_{n, k}
      \brk*{\mu_{l, p} - \frac{g_i^2}{g_i^2+\epsilon}\nu_{l, p}}
      \nu_{n, p}\,
      \Eletpi \noise[i]{0} \noise[j]{0}
\\
  &\qquad - \frac{g_j^3}{(g_j^2+\epsilon)^{5/2}} \sum_i \frac{g_i}{\sqrt{g_i^2+\epsilon}}
      \sum_{k = 0}^{n - 1} (n - k) \nu_{n, k}^2\,
      \Eletpi \noise[i j]{0} \noise[j]{0}
\\
  &\qquad - \frac{g_j^2 \partial_j \norm{\vg}_{1,\epsilon}}{(g_j^2+\epsilon)^{5/2}}
      \sum_{k = 0}^{n - 1} (n - k) \nu_{n, k}^2\,
      \Eletpi (\noise[j]{0})^2
\\
  &\qquad + \frac{g_j^2}{2 (g_j^2+\epsilon)^{3/2}}
    \sum_i \frac{\hess{i}{j}g_i}{(g_i^2+\epsilon)^{5/2}}
  \\
   &\qquad \qquad \times \sum_{l = 0}^{n - 1} \sum_{k, p = 0}^l
      \nu_{n, k}
      \prn*{3 g_i^2 \nu_{l, p}^2 - (g_i^2+\epsilon)\nu_{l, p}
      - 2 (g_i^2+\epsilon)\mu_{l, p}\nu_{l, p}}
      \Eletpi (\noise[i]{0})^2
\\
  &\qquad + \frac{g_j^2}{(g_j^2+\epsilon)^{3/2}}
      \sum_i \frac{1}{\sqrt{g_i^2+\epsilon}}
      \sum_{l = 0}^{n - 1} \sum_{k = 0}^l
      \nu_{n, k}
      \brk*{\mu_{l, k} - \frac{g_i^2}{g_i^2+\epsilon}\nu_{l, k}}
      \Eletpi \noise[i j]{0} \noise[i]{0}
\\
  &\qquad + \frac{g_j}{(g_j^2+\epsilon)^{3/2}}
      \sum_i \frac{\hess{i}{j}}{\sqrt{g_i^2+\epsilon}}
      \sum_{l = 0}^{n - 1} \sum_{k = 0}^l
      \nu_{n, k}
      \brk*{\mu_{l, k} - \frac{g_i^2}{g_i^2+\epsilon}\nu_{l, k}}
      \Eletpi \noise[j]{0} \noise[i]{0}
\\
  &\qquad + \frac{g_j}{(g_j^2+\epsilon)^{3/2}}
      \sum_i \frac{g_i}{\sqrt{g_i^2+\epsilon}}
      \sum_{k = 0}^{n - 1} (n - k) \nu_{n, k}
      \Eletpi \noise[i j]{0} \noise[j]{0}
\\
  &\qquad + O(d^3) + o_n(b^{-1}).
\end{align*}
}
\end{proposition}

\begin{proof}
Consider the average of the term like $\noise[i]{p} \noise[j]{q}$ where the mini-batch indices $p$ and $q$ are not equal:
\begin{align*}
  \Eletpi \noise[i]{p} \noise[j]{q} &= \Eletpi \frac{1}{b} \sum_{r = p b + 1}^{(p + 1) b} \partial_i (\ell_{\pi(r)} - \mathcal{L}) \frac{1}{b} \sum_{s = q b + 1}^{(q + 1) b} \partial_{j} (\ell_{\pi(s)} - \mathcal{L})\\
                            &= \frac{1}{b^2} \sum_{r = p b + 1}^{(p + 1) b} \sum_{s = q b + 1}^{(q + 1) b} \Eletpi \partial_i (\ell_{\pi(r)} - \mathcal{L}) \partial_{j} (\ell_{\pi(s)} - \mathcal{L})\\
                            &= \Eletpi \partial_i (\ell_{\pi(1)} - \mathcal{L}) \partial_j (\ell_{\pi(2)} - \mathcal{L})\\
                            &= \frac{1}{(n + 1) b ((n + 1) b - 1)} \sum_{1 \leq r_1 \neq r_2 \leq (n + 1) b} \partial_i (\ell_{r_1} - \mathcal{L}) \partial_j (\ell_{r_2} - \mathcal{L})\\
                                    &= \frac{1}{(n + 1) b ((n + 1) b - 1)}\\
  &\qquad \times \prn[\bigg]{\sum_{r_1 = 1}^{(n + 1) b} \partial_i (\ell_{r_1} - \mathcal{L}) \sum_{r_2 = 1}^{(n + 1) b} \partial_j (\ell_{r_2} - \mathcal{L}) - \sum_{r = 1}^{(n + 1) b} \partial_i (\ell_r - \mathcal{L}) \partial_j (\ell_r - \mathcal{L})}\\
                            &= - \frac{1}{(n + 1) b - 1} \underbrace{\frac{1}{(n + 1) b}  \sum_{r = 1}^{(n + 1) b} \partial_i (\ell_r - \mathcal{L}) \partial_j (\ell_r - \mathcal{L})}_{O(1)}\\
  &= O(((n + 1) b)^{-1}) = o_n(b^{-1}),
\end{align*}
so, when taking expectations, we can neglect all second-degree monomials of noise derivatives where the two derivatives correspond to different mini-batches (with indices $p \neq q$ in this example). Having made this observation and recalling the expansions obtained in \cref{prop:exp-corr-quadratic-noise}, it is left to use the linearity of expectation and calculate basic exponential series limits:
\begin{align*}
  &\Eletpi \frac{L_{n, j}(\btheta)}{R_{n, j}(\btheta)}\\
  &\quad = \frac{\partial_j \|\vg\|_{1,\epsilon}}{\sqrt{g_j^2+\epsilon}}
   \sum_{k=0}^{n-1}\mu_{n,k}(n-k)\\
&\qquad + \frac{1}{\sqrt{g_j^2+\epsilon}}
\sum_i \frac{\hess{i}{j}g_i}{2(g_i^2+\epsilon)^{5/2}}\\
&\qquad \qquad \times \sum_{l=0}^{n-1}\sum_{k,p=0}^{l}
\mu_{n,k}
\prn*{3g_i^2\nu_{l,p}^2-(g_i^2+\epsilon)\nu_{l,p}-2(g_i^2+\epsilon)\mu_{l,p}\nu_{l,p}}
\,\Eletpi (\noise[i]{0})^2
\\
&\qquad + \frac{1}{\sqrt{g_j^2+\epsilon}}
\sum_i \frac{1}{\sqrt{g_i^2+\epsilon}}
\sum_{l=0}^{n-1}\sum_{k=0}^{l}
\mu_{n,k}
\brk*{\mu_{l,k}-\frac{g_i^2}{g_i^2+\epsilon}\nu_{l,k}}
\,\Eletpi \noise[ij]{0}\noise[i]{0}
\\
&\qquad - \frac{g_j}{(g_j^2+\epsilon)^{3/2}}
\sum_i \frac{\hess{i}{j}}{\sqrt{g_i^2+\epsilon}}
\sum_{l=0}^{n-1}\sum_{k,p=0}^{l}
\mu_{n,k}\nu_{n,p}
\brk*{\mu_{l,p}-\frac{g_i^2}{g_i^2+\epsilon}\nu_{l,p}}
\,\Eletpi \noise[i]{0}\noise[j]{0}
\\
&\qquad - \frac{g_j}{(g_j^2+\epsilon)^{3/2}}
\sum_i \frac{g_i}{\sqrt{g_i^2+\epsilon}}
\sum_{k=0}^{n-1}\mu_{n,k}(n-k)\nu_{n,k}
\,\Eletpi \noise[ij]{0}\noise[j]{0}
\\
&\qquad + \frac{\partial_j\|\vg\|_{1,\epsilon}}{2(g_j^2+\epsilon)^{5/2}}
\sum_{r=0}^{n}\prn*{3g_j^2\nu_{n,r}^2-(g_j^2+\epsilon)\nu_{n,r}}
\sum_{k=0}^{n-1}\mu_{n,k}(n-k)
\,\Eletpi (\noise[j]{0})^2
\\
&\qquad + O(d^3)+o_n(b^{-1}),
\end{align*}
and similarly,
\begin{align*}
  &\Eletpi \frac{M_{n, j}(\btheta) P_{n, j}(\btheta)}{R_{n, j}(\btheta)^3}\\
  &\quad = \frac{g_j^2}{(g_j^2+\epsilon)^{3/2}}
      \partial_j \norm{\vg}_{1,\epsilon}
      \sum_{k = 0}^{n - 1} (n - k) \nu_{n, k}
\\
  &\qquad + \frac{g_j^2 \partial_j \norm{\vg}_{1,\epsilon}}{(g_j^2+\epsilon)^{7/2}}
      \sum_{r = 0}^{n - 1} (n - r) \nu_{n, r}
      \sum_{k = 0}^n
      \prn*{4 g_j^2 \nu_{n, k}^2 - (g_j^2+\epsilon)\nu_{n, k}
      - 2 (g_j^2+\epsilon)\mu_{n, k}\nu_{n, k}}
      \Eletpi (\noise[j]{0})^2
\\
  &\qquad + \frac{g_j}{(g_j^2+\epsilon)^{5/2}} \sum_i \frac{\hess{i}{j}}{\sqrt{g_i^2+\epsilon}}\\
  &\qquad \qquad \times
      \sum_{l = 0}^{n - 1} \sum_{k, p = 0}^l
      \nu_{n, k}
      \brk*{\mu_{l, p} - \frac{g_i^2}{g_i^2+\epsilon}\nu_{l, p}}
      \brk*{(g_j^2+\epsilon)\mu_{n, p} - 2 g_j^2 \nu_{n, p}}
      \Eletpi \noise[i]{0} \noise[j]{0}
\\
  &\qquad + \frac{g_j}{(g_j^2+\epsilon)^{5/2}} \sum_i \frac{g_i}{\sqrt{g_i^2+\epsilon}}
      \sum_{k = 0}^{n - 1}
      (n - k) \nu_{n, k}
      \brk*{(g_j^2+\epsilon)\mu_{n, k} - 2 g_j^2 \nu_{n, k}}
      \Eletpi \noise[i j]{0} \noise[j]{0}
\\
  &\qquad + \frac{\partial_j \norm{\vg}_{1,\epsilon}}{(g_j^2+\epsilon)^{5/2}}
      \sum_{k = 0}^{n - 1}
      (n - k) \nu_{n, k}
      \brk*{(g_j^2+\epsilon)\mu_{n, k} - 2 g_j^2 \nu_{n, k}}
      \Eletpi (\noise[j]{0})^2
\\
  &\qquad - \frac{g_j^2 \partial_j \norm{\vg}_{1,\epsilon}}{(g_j^2+\epsilon)^{7/2}}
      \sum_{p = 0}^{n - 1} (n - p) \nu_{n, p}
      \sum_{k = 0}^n
      \nu_{n, k}
      \brk*{(g_j^2+\epsilon)\mu_{n, k} - 2 g_j^2 \nu_{n, k}}
      \Eletpi (\noise[j]{0})^2
\\
  &\qquad + \frac{g_j^2 \partial_j \norm{\vg}_{1,\epsilon}}{2 (g_j^2+\epsilon)^{7/2}}
      \sum_{r = 0}^{n - 1} (n - r) \nu_{n, r}
      \sum_{k = 0}^n
      \prn*{3 g_j^2 \nu_{n, k}^2 - (g_j^2+\epsilon)\nu_{n, k}}
      \Eletpi (\noise[j]{0})^2
\\
  &\qquad - \frac{g_j^3}{(g_j^2+\epsilon)^{5/2}} \sum_i \frac{\hess{i}{j}}{\sqrt{g_i^2+\epsilon}}
      \sum_{l = 0}^{n - 1} \sum_{k, p = 0}^l
      \nu_{n, k}
      \brk*{\mu_{l, p} - \frac{g_i^2}{g_i^2+\epsilon}\nu_{l, p}}
      \nu_{n, p}\,
      \Eletpi \noise[i]{0} \noise[j]{0}
\\
  &\qquad - \frac{g_j^3}{(g_j^2+\epsilon)^{5/2}} \sum_i \frac{g_i}{\sqrt{g_i^2+\epsilon}}
      \sum_{k = 0}^{n - 1} (n - k) \nu_{n, k}^2\,
      \Eletpi \noise[i j]{0} \noise[j]{0}
\\
  &\qquad - \frac{g_j^2 \partial_j \norm{\vg}_{1,\epsilon}}{(g_j^2+\epsilon)^{5/2}}
      \sum_{k = 0}^{n - 1} (n - k) \nu_{n, k}^2\,
      \Eletpi (\noise[j]{0})^2
\\
  &\qquad + \frac{g_j^2}{2 (g_j^2+\epsilon)^{3/2}}
      \sum_i \frac{\hess{i}{j}g_i}{(g_i^2+\epsilon)^{5/2}}\\
  &\qquad \qquad \times \sum_{l = 0}^{n - 1} \sum_{k, p = 0}^l
      \nu_{n, k}
      \prn*{3 g_i^2 \nu_{l, p}^2 - (g_i^2+\epsilon)\nu_{l, p}
      - 2 (g_i^2+\epsilon)\mu_{l, p}\nu_{l, p}}
      \Eletpi (\noise[i]{0})^2
\\
  &\qquad + \frac{g_j^2}{(g_j^2+\epsilon)^{3/2}}
      \sum_i \frac{1}{\sqrt{g_i^2+\epsilon}}
      \sum_{l = 0}^{n - 1} \sum_{k = 0}^l
      \nu_{n, k}
      \brk*{\mu_{l, k} - \frac{g_i^2}{g_i^2+\epsilon}\nu_{l, k}}
      \Eletpi \noise[i j]{0} \noise[i]{0}
\\
  &\qquad + \frac{g_j}{(g_j^2+\epsilon)^{3/2}}
      \sum_i \frac{\hess{i}{j}}{\sqrt{g_i^2+\epsilon}}
      \sum_{l = 0}^{n - 1} \sum_{k = 0}^l
      \nu_{n, k}
      \brk*{\mu_{l, k} - \frac{g_i^2}{g_i^2+\epsilon}\nu_{l, k}}
      \Eletpi \noise[j]{0} \noise[i]{0}
\\
  &\qquad + \frac{g_j}{(g_j^2+\epsilon)^{3/2}}
      \sum_i \frac{g_i}{\sqrt{g_i^2+\epsilon}}
      \sum_{k = 0}^{n - 1} (n - k) \nu_{n, k}
      \Eletpi \noise[i j]{0} \noise[j]{0}
\\
  &\qquad + O(d^3) + o_n(b^{-1}),
\end{align*}
concluding the proof.
\end{proof}

\begin{lemma}\label{lem:noise-exp-in-cov-matr}
  We have for all $k \in \range{0}{n}$, $i, j \in \range{1}{\dim \btheta}$
\begin{align*}
&\Eletpi \noise[i]{k} \noise[j]{k} = \frac{n}{(n + 1) b - 1} \Sigma_{i j},\\
&\Eletpi \noise[i j]{k} \noise[j]{k} = \frac{n}{2 ((n + 1) b - 1)} \partial_i \Sigma_{j j}.
\end{align*}
\end{lemma}

\begin{proof}
For any $r \in \range{1}{(n + 1) b}$ we have
\begin{equation*}
  \Epi{\partial_{i j} (\ell_{\pi(r)} - \mathcal{L}) \partial_j (\ell_{\pi(r)} - \mathcal{L})}
  = \frac{1}{(n + 1) b} \sum_{p = 1}^{(n + 1) b} \partial_{i j} (\ell_p - \mathcal{L}) \partial_j (\ell_p - \mathcal{L}) = \frac{1}{2} \partial_i \Sigma_{j j},
\end{equation*}
and for $r \neq \tilde{r}$,
\begin{align*}
  \Epi{\partial_{i j} (\ell_{\pi(r)} - \mathcal{L}) \partial_j (\ell_{\pi(\tilde{r})} - \mathcal{L})}
  &= \frac{1}{(n + 1) b ((n + 1) b - 1)} \sum_{\substack{p, q = 1 \\ p \neq q}}^{(n + 1) b} \partial_{i j} (\ell_{p} - \mathcal{L}) \partial_j (\ell_{q} - \mathcal{L})\\
  &= - \frac{1}{(n + 1) b ((n + 1) b - 1)} \sum_{\substack{p = 1}}^{(n + 1) b} \partial_{i j} (\ell_p - \mathcal{L}) \partial_j (\ell_p - \mathcal{L})\\
  &= - \frac{1}{2 ((n + 1) b - 1)} \partial_i \Sigma_{j j}.
\end{align*}

Next,
\begin{align*}
  \Eletpi \noise[i]{k} \noise[j]{k} &= \Eletpi \prn[\bigg]{\frac{1}{b} \sum_{r = k b + 1}^{k b + b} (\partial_i \ell_{\pi(r)} - \partial_i \mathcal{L})}  \prn[\bigg]{\frac{1}{b} \sum_{r = k b + 1}^{k b + b} (\partial_j \ell_{\pi(r)} - \partial_j \mathcal{L})}\\
                     &= \frac{1}{b^2} \sum_{r = k b + 1}^{k b + b} \Eletpi (\partial_i \ell_{\pi(r)} - \partial_i \mathcal{L}) (\partial_j \ell_{\pi(r)} - \partial_j \mathcal{L})\\
                     &\quad + \frac{1}{b^2} \sum_{k b + 1 \leq r \neq \tilde{r} \leq k b + b} \Eletpi (\partial_i \ell_{\pi(r)} - \partial_i \mathcal{L})(\partial_j \ell_{\pi(\tilde{r})} - \partial_j \mathcal{L})\\
                            &= \frac{1}{b} \Eletpi (\partial_i \ell_{\pi(1)} - \partial_i \mathcal{L}) (\partial_j \ell_{\pi(1)} - \partial_j \mathcal{L})\\
  &\quad + \frac{b - 1}{b} \Eletpi (\partial_i \ell_{\pi(1)} - \partial_i \mathcal{L})(\partial_j \ell_{\pi(2)} - \partial_j \mathcal{L})\\
                            &= \frac{1}{(n + 1) b^2} \sum_{p = 1}^{(n + 1) b} (\partial_i \ell_{p} - \partial_i \mathcal{L}) (\partial_j \ell_{p} - \partial_j \mathcal{L})\\
  &\quad + \frac{b - 1}{(n + 1) b^2 ((n + 1) b - 1)} \sum_{\substack{p, q = 1 \\ p \neq q}}^{(n + 1) b} (\partial_i \ell_{p} - \partial_i \mathcal{L})(\partial_j \ell_{q} - \partial_j \mathcal{L})\\
                            &= \frac{1}{(n + 1) b^2} \sum_{p = 1}^{(n + 1) b} (\partial_i \ell_{p} - \partial_i \mathcal{L}) (\partial_j \ell_{p} - \partial_j \mathcal{L})\\
  &\quad - \frac{b - 1}{(n + 1) b^2 ((n + 1) b - 1)} \sum_{p = 1}^{(n + 1) b} (\partial_i \ell_{p} - \partial_i \mathcal{L}) (\partial_j \ell_{p} - \partial_j \mathcal{L})\\
                            &= \frac{n}{(n + 1) b - 1} \Sigma_{i j}.
\end{align*}
Similarly,
\begin{align*}
  \Eletpi \noise[i j]{k} \noise[j]{k} &= \frac{n}{(n + 1) b - 1} \frac{1}{(n + 1) b} \sum_{p = 1}^{(n + 1) b} (\partial_{i j} \ell_{p} - \partial_{i j} \mathcal{L}) (\partial_j \ell_{p} - \partial_j \mathcal{L})\\
                              &= \frac{n}{2 ((n + 1) b - 1)} \partial_i \Sigma_{j j}.
                                \qedhere
\end{align*}
\end{proof}

\begin{proof}[Proof of \cref{thm:finite-n-nonzero-eps}]
Combine \cref{prop:e-pi-of-res} and \cref{lem:noise-exp-in-cov-matr}.
\end{proof}

\section{Proof of \cref{thm:limits}}\label{sec:proof-of-limits-thm}

\begin{proof}[Proof of \cref{thm:large-n-limits-eps}]
The limit of \(\fbneps{j}\) follows immediately from
\begin{gather*}
\sum_{k=0}^{n-1}\mu_{n,k}(n-k)
\longrightarrow
(1-\beta_1)\sum_{a=1}^{\infty}a\beta_1^a
=
\frac{\beta_1}{1-\beta_1},
\\
\sum_{k=0}^{n-1}\nu_{n,k}(n-k)
\longrightarrow
\frac{\beta_2}{1-\beta_2}.
\end{gather*}

We also need the elementary limits
\begin{align*}
\sum_{k=0}^{n}\nu_{n,k}^2
&\longrightarrow
\frac{1-\beta_2}{1+\beta_2},
\\
\sum_{k=0}^{n}\mu_{n,k}\nu_{n,k}
&\longrightarrow
\frac{(1-\beta_1)(1-\beta_2)}{1-\beta_1\beta_2},
\\
\sum_{k=0}^{n-1}(n-k)\mu_{n,k}\nu_{n,k}
&\longrightarrow
\frac{\beta_1\beta_2(1-\beta_1)(1-\beta_2)}
{(1-\beta_1\beta_2)^2},
\\
\sum_{k=0}^{n-1}(n-k)\nu_{n,k}^2
&\longrightarrow
\frac{\beta_2^2}{(1+\beta_2)^2}.
\end{align*}
Applying these four limits directly to the definition of
\(A^{(n,\epsilon)}_j\) gives \(A^{(n,\epsilon)}_j\to A^{(\infty,\epsilon)}_j\).
Since
\[
\frac{n}{(n + 1)b-1}=\frac{n}{(n+1)b-1} = b^{-1} + o_n(b^{-1}),
\]
we obtain
\[
\mbnoneneps{j}
=
\mbnoneinfeps{j} + o_n(b^{-1}).
\]

Next, using the limits
\begin{align*}
  \sum_{l = 0}^{n - 1} \sum_{k, p = 0}^l \mu_{n, k} \nu_{l, p}^2
  &\longrightarrow
  \frac{\beta_1}{1-\beta_1}\frac{1-\beta_2}{1+\beta_2},
  \\
  \sum_{l = 0}^{n - 1} \sum_{k, p = 0}^l \mu_{n, k} \nu_{l, p}
  &\longrightarrow
  \frac{\beta_1}{1-\beta_1},
  \\
  \sum_{l = 0}^{n - 1} \sum_{k, p = 0}^l \mu_{n, k}\mu_{l,p}\nu_{l,p}
  &\longrightarrow
  \frac{\beta_1}{1-\beta_1}
    \frac{(1-\beta_1)(1-\beta_2)}{1-\beta_1\beta_2},
  \\
    \sum_{l = 0}^{n - 1} \sum_{k, p = 0}^l \nu_{n, k} \nu_{l, p}^2
  &\longrightarrow
  \frac{\beta_2}{1-\beta_2}\frac{1-\beta_2}{1+\beta_2},
  \\
  \sum_{l = 0}^{n - 1} \sum_{k, p = 0}^l \nu_{n, k} \nu_{l, p}
  &\longrightarrow
  \frac{\beta_2}{1-\beta_2},
  \\
  \sum_{l = 0}^{n - 1} \sum_{k, p = 0}^l \nu_{n, k}\mu_{l,p}\nu_{l,p}
  &\longrightarrow
  \frac{\beta_2}{1-\beta_2}
  \frac{(1-\beta_1)(1-\beta_2)}{1-\beta_1\beta_2},
\end{align*}
we see that
\begin{align*}
&\sum_{l=0}^{n-1}\sum_{k,p=0}^{l}
\mu_{n,k}
\prn[\big]{
3g_i^2\nu_{l,p}^2
-(g_i^2+\epsilon)\nu_{l,p}
-2(g_i^2+\epsilon)\mu_{l,p}\nu_{l,p}
}
  \\
&\quad \longrightarrow
\frac{\beta_1}{1-\beta_1} \prn[\bigg]{3g_i^2\frac{1-\beta_2}{1+\beta_2}
-(g_i^2+\epsilon)
-2(g_i^2+\epsilon)
\frac{(1-\beta_1)(1-\beta_2)}{1-\beta_1\beta_2}},
\end{align*}
and
\begin{align*}
&\sum_{l=0}^{n-1}\sum_{k,p=0}^{l}
\nu_{n,k}
\prn[\big]{
3g_i^2\nu_{l,p}^2
-(g_i^2+\epsilon)\nu_{l,p}
-2(g_i^2+\epsilon)\mu_{l,p}\nu_{l,p}
}
  \\
&\quad \longrightarrow
\frac{\beta_2}{1-\beta_2} \prn[\bigg]{3g_i^2\frac{1-\beta_2}{1+\beta_2}
-(g_i^2+\epsilon)
-2(g_i^2+\epsilon)
\frac{(1-\beta_1)(1-\beta_2)}{1-\beta_1\beta_2}}.
\end{align*}
Substituting these limits into \(B^{(n,\epsilon)}_{i,j}\) gives
\[
B^{(n,\epsilon)}_{i,j}
\longrightarrow B^{(\infty,\epsilon)}_{i,j}.
\]
Thus
\[
\mbntwoneps{j}
=
\mbntwoinfeps{j} + o_n(b^{-1}).
\]

For \(\mbnthreeneps{j}\), the required scalar limits are
\[
\sum_{k=0}^{n-1}(n-k)\mu_{n,k}\nu_{n,k}
\longrightarrow
\frac{\beta_1\beta_2(1-\beta_1)(1-\beta_2)}
{(1-\beta_1\beta_2)^2},
\]
\[
\sum_{k=0}^{n-1}(n-k)\nu_{n,k}
\longrightarrow
\frac{\beta_2}{1-\beta_2},
\qquad
\sum_{k=0}^{n-1}(n-k)\nu_{n,k}^2
\longrightarrow
\frac{\beta_2^2}{(1+\beta_2)^2}.
\]
Together with
\[
\frac{n}{2((n + 1) b-1)} = \frac1{2b} + o_n(b^{-1}),
\]
these give
\[
\mbnthreeneps{j}
=
\mbnthreeinfeps{j} + o_n(b^{-1}).
\]

For \(D^{(n,\epsilon)}_{i,j}\), we use the four limits
\begin{align*}
\sum_{l=0}^{n-1}\sum_{k=0}^{l}\mu_{n,k}\mu_{l,k}
&\longrightarrow
\frac{\beta_1}{1+\beta_1},
\\
\sum_{l=0}^{n-1}\sum_{k=0}^{l}\mu_{n,k}\nu_{l,k}
&\longrightarrow
\frac{\beta_1(1-\beta_2)}{1-\beta_1\beta_2},
\\
\sum_{l=0}^{n-1}\sum_{k=0}^{l}\nu_{n,k}\mu_{l,k}
&\longrightarrow
\frac{\beta_2(1-\beta_1)}{1-\beta_1\beta_2},
\\
\sum_{l=0}^{n-1}\sum_{k=0}^{l}\nu_{n,k}\nu_{l,k}
&\longrightarrow
\frac{\beta_2}{1+\beta_2}.
\end{align*}
Substituting them into \(D^{(n,\epsilon)}_{i,j}\) gives
\(D^{(n,\epsilon)}_{i,j}\to D^{(\infty,\epsilon)}_{i,j}\), and the prefactor
\((n)/(2((n + 1)b-1)) = 1/(2b) + o_n(b^{-1})\) gives
\[
\mbnfourneps{j}
=
\mbnfourinfeps{j} + o_n(b^{-1}).
\]

It remains to treat \(E^{(n,\epsilon)}_{i,j}\). We use the following scalar limits:
\begin{align*}
\sum_{l=0}^{n-1}\sum_{k,p=0}^{l}
\mu_{n,k}\nu_{n,p}\mu_{l,p}
&\longrightarrow
\frac{\beta_1\beta_2(1-\beta_1)(1-\beta_2)}
{(1-\beta_1\beta_2)^2},
\\
\sum_{l=0}^{n-1}\sum_{k,p=0}^{l}
\mu_{n,k}\nu_{n,p}\nu_{l,p}
&\longrightarrow
\frac{\beta_1\beta_2(1-\beta_2)}
{(1+\beta_2)(1-\beta_1\beta_2)},
\\
\sum_{l=0}^{n-1}\sum_{k,p=0}^{l}
\nu_{n,k}\mu_{l,p}\mu_{n,p}
&\longrightarrow
\frac{\beta_1\beta_2(1-\beta_1)}
{(1+\beta_1)(1-\beta_1\beta_2)},
\\
\sum_{l=0}^{n-1}\sum_{k,p=0}^{l}
\nu_{n,k}\mu_{l,p}\nu_{n,p}
&\longrightarrow
\frac{\beta_2^2(1-\beta_1)}
{(1+\beta_2)(1-\beta_1\beta_2)},
\\
\sum_{l=0}^{n-1}\sum_{k,p=0}^{l}
\nu_{n,k}\nu_{l,p}\mu_{n,p}
&\longrightarrow
\frac{\beta_1\beta_2(1-\beta_1)(1-\beta_2)}
{(1-\beta_1\beta_2)^2},
\\
\sum_{l=0}^{n-1}\sum_{k,p=0}^{l}
\nu_{n,k}\nu_{l,p}\nu_{n,p}
&\longrightarrow
\frac{\beta_2^2}{(1+\beta_2)^2},
\\
\sum_{l=0}^{n-1}\sum_{k=0}^{l}
\nu_{n,k}\mu_{l,k}
&\longrightarrow
\frac{\beta_2(1-\beta_1)}
{1-\beta_1\beta_2},
\\
\sum_{l=0}^{n-1}\sum_{k=0}^{l}
\nu_{n,k}\nu_{l,k}
&\longrightarrow
\frac{\beta_2}{1+\beta_2}.
\end{align*}

Substituting these limits into $E^{(n,\epsilon)}_{i,j}$,
we get
\begin{align*}
E^{(n,\epsilon)}_{i,j}
\longrightarrow
E^{(\infty,\epsilon)}_{i,j}
:=&\
-\frac{\beta_1\beta_2(1-\beta_1)(1-\beta_2)}
{(1-\beta_1\beta_2)^2}
+\frac{g_i^2}{g_i^2+\epsilon}
\frac{\beta_1\beta_2(1-\beta_2)}
{(1+\beta_2)(1-\beta_1\beta_2)}
\\
&
-\frac{\beta_1\beta_2(1-\beta_1)}
{(1+\beta_1)(1-\beta_1\beta_2)}
+\frac{2g_j^2}{g_j^2+\epsilon}
\frac{\beta_2^2(1-\beta_1)}
{(1+\beta_2)(1-\beta_1\beta_2)}
\\
&
+\frac{g_i^2}{g_i^2+\epsilon}
\frac{\beta_1\beta_2(1-\beta_1)(1-\beta_2)}
{(1-\beta_1\beta_2)^2}
-2\frac{g_i^2}{g_i^2+\epsilon}\frac{g_j^2}{g_j^2+\epsilon}
\frac{\beta_2^2}{(1+\beta_2)^2}
\\
&
+\frac{g_j^2}{g_j^2+\epsilon}
\frac{\beta_2^2(1-\beta_1)}
{(1+\beta_2)(1-\beta_1\beta_2)}
-\frac{g_i^2}{g_i^2+\epsilon}\frac{g_j^2}{g_j^2+\epsilon}
\frac{\beta_2^2}{(1+\beta_2)^2}
\\
&
-\frac{\beta_2(1-\beta_1)}
{1-\beta_1\beta_2}
+\frac{g_i^2}{g_i^2+\epsilon}
  \frac{\beta_2}{1+\beta_2}
\\
  =&\
-\frac{\beta_1\beta_2(1-\beta_1)(1-\beta_2)}
{(1-\beta_1\beta_2)^2}
+\frac{g_i^2}{g_i^2+\epsilon}
\frac{\beta_1\beta_2(1-\beta_2)}
{(1+\beta_2)(1-\beta_1\beta_2)}
\\
&\quad
-\frac{\beta_1\beta_2(1-\beta_1)}
{(1+\beta_1)(1-\beta_1\beta_2)}
+\frac{g_i^2}{g_i^2+\epsilon}
\frac{\beta_1\beta_2(1-\beta_1)(1-\beta_2)}
{(1-\beta_1\beta_2)^2}
\\
&\quad
+3\frac{g_j^2}{g_j^2+\epsilon}
\frac{\beta_2^2(1-\beta_1)}
{(1+\beta_2)(1-\beta_1\beta_2)}
-3\frac{g_i^2}{g_i^2+\epsilon}
\frac{g_j^2}{g_j^2+\epsilon}
\frac{\beta_2^2}{(1+\beta_2)^2}
\\
&\quad
-\frac{\beta_2(1-\beta_1)}
{1-\beta_1\beta_2}
+\frac{g_i^2}{g_i^2+\epsilon}
\frac{\beta_2}{1+\beta_2}.
\end{align*}
Thus
\[
E^{(n,\epsilon)}_{i,j}
=
E^{(\infty,\epsilon)}_{i,j}+o_n(1).
\]
Since
\[
\frac{n}{(n + 1)b-1}
=
\frac1b+o_n(b^{-1}),
\]
we obtain
\begin{align*}
\mbnfiveneps{j}
&=
\frac{n}{(n + 1)b-1}
\frac{g_j}{g_j^2+\epsilon}
\sum_i
\frac{\hess{i}{j}}{\sqrt{g_i^2+\epsilon}}
E^{(n,\epsilon)}_{i,j}
\Sigma_{ij}
\\
&=
\frac1b
\frac{g_j}{g_j^2+\epsilon}
\sum_i
\frac{\hess{i}{j}}{\sqrt{g_i^2+\epsilon}}
E^{(\infty,\epsilon)}_{i,j}
\Sigma_{ij}
+o_n(b^{-1})
\\
&=
\mbnfiveinfeps{j}+o_n(b^{-1}).
\end{align*}

This proves all claimed limits.
\end{proof}

\begin{proof}[Proof of \cref{thm:large-n-zero-eps-simple-version}]
This is obtained directly by setting $\epsilon = 0$ and simplifying the resulting expressions.
The resulting limit is finite if $g_i \neq 0$ in \eqref{eq:mbn-simple}.
\end{proof}

\section{Simplification of Terms}\label{sec:simplification-of-terms}

\subsection{The Terms $\mbnfour{j}(\beta_1, \beta_2)$ and $\mbnfive{j}(\beta_1, \beta_2)$ are Small}
First, the following \lcnamecref{lem:c4-c5-are-small} implies that the terms containing $C_4(\beta_1, \beta_2)$ and $C_5(\beta_1, \beta_2)$ are small compared to other terms and can therefore be neglected.

\begin{lemma}[$C_4(\beta_1, \beta_2)$ and $C_5(\beta_1, \beta_2)$ are small]\label{lem:c4-c5-are-small}
The following bounds hold, with quotients involving \(C_2\) interpreted by continuous extension at removable singularities:
\begin{align}
  \sup_{\beta_2 \in [0.9, 1)}
  \abs{C_4(0.9, \beta_2) / C_1(0.9, \beta_2)}
  &< 1.5 \times 10^{-4},
  \label{eq:c4-c5-small-1}
  \\
  \sup_{\beta_2 \in [0.9, 1)}
  \abs{C_5(0.9, \beta_2) / C_1(0.9, \beta_2)}
  &< 4 \times 10^{-3},
  \label{eq:c4-c5-small-2}
  \\
  \sup_{\beta_1 \in [0.9, 1)}
  \abs{C_4(\beta_1, 0.999) / C_2(\beta_1, 0.999)}
  &< 3 \times 10^{-5},
        \label{eq:c4-c5-small-5}
  \\
  \sup_{\beta_1 \in [0.9, 1)}
  \abs{C_5(\beta_1, 0.999) / C_2(\beta_1, 0.999)}
  &< 5 \times 10^{-4}.
  \label{eq:c4-c5-small-6}
\end{align}
\end{lemma}

\begin{proof}
We first prove \eqref{eq:c4-c5-small-1} and \eqref{eq:c4-c5-small-2}.
Put
\[
\beta_2 = \frac{9}{10}+\frac{t}{10}, \qquad t\in[0,1).
\]
On this interval,
\[
C_4(0.9, \beta_2) \le 0,
\qquad
C_5(0.9, \beta_2) \le 0.
\]
Therefore
\[
\left|\frac{C_4(0.9,\beta_2)}{C_1(0.9,\beta_2)}\right|<\frac{3}{20000}
\]
follows from
\[
\frac{3}{20000}C_1(0.9,\beta_2)+C_4(0.9,\beta_2)>0,
\]
and
\[
\left|\frac{C_5(0.9,\beta_2)}{C_1(0.9,\beta_2)}\right|<\frac{1}{250}
\]
follows from
\[
\frac{1}{250}C_1(0.9,\beta_2)+C_5(0.9,\beta_2)>0.
\]
After substituting \(\beta_2 = 9/10+t/10\), direct simplification gives
\begin{align*}
\frac{3}{20000}C_1(0.9,\beta_2)+C_4(0.9,\beta_2)
&=
\frac{P_1(t)}
{-1140000(t-1)(t+19)^2(9t-19)^2},
\\
\frac{1}{250}C_1(0.9,\beta_2)+C_5(0.9,\beta_2)
&=
\frac{P_2(t)}
{-42750(t-1)(t+19)^2(9t-19)^2},
\end{align*}
where
\begin{align*}
P_1(t)
={}&
-26664498t^5-421710270t^4+1546038360t^3
  \\
&\qquad \qquad -1177371480t^2+3411450t+178896438,
\\
P_2(t)
={}&
4385502t^5-7335270t^4-324386640t^3
+978653520t^2-727613550t+178896438.
\end{align*}
The denominators above are positive for \(t\in[0,1)\). It remains to show that
\(P_1\) and \(P_2\) are positive on \([0,1)\).
Using Sturm's theorem, we can observe that neither \(P_1\) nor \(P_2\) has a zero in \((0,1)\). Since both are positive
at \(0\), they are positive throughout \([0,1]\), concluding the proof of \eqref{eq:c4-c5-small-1} and \eqref{eq:c4-c5-small-2}.

We now treat the quotients involving \(C_2\). Direct algebra gives
\begin{align}
\frac{C_4(\beta_1,\beta_2)}{C_2(\beta_1,\beta_2)}
&=
-\frac{(\beta_1-1)(\beta_1-\beta_2)(\beta_2-1)}
{2(\beta_1+1)\left(\beta_1\beta_2^2-\beta_1\beta_2+\beta_1+\beta_2^2-2\beta_2\right)},
\label{eq:c4-c2-simplified}
\\
\frac{C_5(\beta_1,\beta_2)}{C_2(\beta_1,\beta_2)}
&=
\frac{\beta_2(\beta_1-1)(\beta_2-1)(3\beta_1-2\beta_2+1)}
{(\beta_1+1)(\beta_2+1)\left(\beta_1\beta_2^2-\beta_1\beta_2+\beta_1+\beta_2^2-2\beta_2\right)}.
\label{eq:c5-c2-simplified}
\end{align}
Put
\[
\beta_1=\frac{9}{10}+\frac{t}{10},\qquad t\in[0,1).
\]
From \eqref{eq:c4-c2-simplified},
\begin{align*}
&\left(\frac{3}{100000}\right)^2
-
\left(\frac{C_4(\beta_1,0.999)}{C_2(\beta_1,0.999)}\right)^2
  \\
&\quad =
\frac{
-\left(47002997t^2-153416114t+107011917\right)
\left(52997003t^2-45583886t-8011917\right)
}
{10000000000(t+19)^2(999001t-1008981)^2}.
\end{align*}
The denominator is positive. The first quadratic factor is positive on \([0,1]\):
it is decreasing on \([0,1]\), and its value at \(t=1\) is positive. The second
quadratic factor is negative on \([0,1]\): it is convex and negative at both endpoints.
Thus the numerator is positive, giving
\[
\left|
\frac{C_4(\beta_1,0.999)}{C_2(\beta_1,0.999)}
\right|
<
\frac{3}{100000}
=
3\times 10^{-5}.
\]
Finally, using \eqref{eq:c5-c2-simplified},
\begin{align*}
&\left(\frac{1}{2000}\right)^2
-
\left(\frac{C_5(\beta_1,0.999)}{C_2(\beta_1,0.999)}\right)^2
  \\
&\quad =
\frac{
  \begin{aligned}
    &-\left(3996997001t^2-7914143962t+4316147361\right)\\
    &\qquad \qquad \qquad \times \left(7991002999t^2+63938063962t-72328067361\right)
\end{aligned}
}
{15984004000000(t+19)^2(999001t-1008981)^2}.
\end{align*}
The denominator is positive. The first quadratic factor is positive on \([0,1]\),
because its discriminant is negative and its leading coefficient is positive. The
second quadratic factor is negative on \([0,1]\), since it is increasing there and
its value at \(t=1\) is negative. Hence the numerator is positive, and
\[
\left|
\frac{C_5(\beta_1,0.999)}{C_2(\beta_1,0.999)}
\right|
<
\frac{1}{2000}
=
5\times 10^{-4}.
\]
\Cref{eq:c4-c5-small-5,eq:c4-c5-small-6} are proven.
\end{proof}

\subsection{The Term $\mbnthree{j}(\beta_1,\beta_2)$ is Neutral}\label{sec:term-c_3-neutral-generalization}

Recalling
\begin{equation*}
\Eletpi \noise[i j]{k} \noise[j]{k} = \frac{n}{2 ((n + 1) b - 1)} \partial_i \Sigma_{j j},
\end{equation*}
by \cref{lem:noise-exp-in-cov-matr},
we claim that
\begin{align*}
  \mbnthree{j}(\beta_1,\beta_2)
  &= \frac{1}{b} C_3(\beta_1, \beta_2)
  \frac{\sign g_j}{\abs{g_j}} \sum_i \sign g_i \, \partial_i \Sigma_{j j}\\
  &= 2 C_3(\beta_1, \beta_2) \frac{\sign g_j}{\abs{g_j}} \sum_i \sign g_i \, \Eletpi \noise[i j]{0} \noise[j]{0} + o_n(b^{-1})
\end{align*}
provides neither regularization nor anti-regularization, i.\,e. is neutral.

We start by rewriting
\begin{align*}
  \frac{\sign g_j}{\abs{g_j}} \sum_i \sign g_i \, \Eletpi \noise[i j]{0} \noise[j]{0}
  &= \frac{\sign g_j}{\abs{g_j}} \sum_i \sign g_i \, \Eletpi (\partial_{i j} \mathcal{L}_0 - \partial_{i j} \mathcal{L}) \noise[j]{0}\\
  &= \frac{\sign g_j}{\abs{g_j}} \sum_i \sign g_i \, \Eletpi \partial_{i j} \mathcal{L}_0 \noise[j]{0}.
\end{align*}
In the gradient-dominated (as opposed to noise-dominated) regime, the sign of a mini-batch gradient
component is typically the same as the sign of the full-batch gradient component: $\sign \partial_i \mathcal{L}_0 \approx \sign \partial_i \mathcal{L}$. Then
\begin{align*}
  \frac{\sign g_j}{\abs{g_j}} \sum_i \sign g_i \, \Eletpi \partial_{i j} \mathcal{L}_0 \noise[j]{0}
  &\approx \frac{1}{g_j} \sum_i \Eletpi \prn{\partial_j \abs{\partial_i \mathcal{L}_0}} \noise[j]{0}\\
  &= \Eletpi \frac{\noise[j]{0}}{g_j} \partial_j \sum_i \abs{\partial_i \mathcal{L}_0} = \Eletpi \frac{\noise[j]{0}}{g_j} \partial_j \norm{\nabla \mathcal{L}_0}_1.
\end{align*}
The factor $\noise[j]{0} / g_j$ (noise component relative to gradient component) can be equally likely positive or negative, so there is no preferred choice whether the 1-norm of the gradient is penalized or anti-penalized. Since our interest is the sign of (anti-)penalization, we can interpret this term as neutral for our purposes.

\section{Monotonicity Regions}

\begin{proposition}[Monotonicity of \(C_{\mathrm{total}}(0.9,\beta_2,\lambda)\)]\label{prop:ctotal-monotonicity-b2-full}
For \(\beta_2\in[0.9,1)\), define the auxiliary polynomials
\begin{align*}
P_1(\beta_2)
:=&\
39753\beta_2^6-175145\beta_2^5+233419\beta_2^4
+17703\beta_2^3
-294872\beta_2^2+239758\beta_2-60600,
\\
P_2(\beta_2)
:=&\
357777\beta_2^8-2344581\beta_2^7+5725125\beta_2^6
-5463783\beta_2^5
-1622132\beta_2^4
\\
&
+8722516\beta_2^3-8468652\beta_2^2+3777388\beta_2-683690,
\\
P_3(\beta_2)
:=&\
728523\beta_2^5-2389693\beta_2^4+1633232\beta_2^3
+2272548\beta_2^2
-3534988\beta_2+1289690,
\\
P_4(\beta_2)
:=&\
2185569\beta_2^6-9477825\beta_2^5+12213510\beta_2^4
+2832960\beta_2^3
\\
&
-19805820\beta_2^2+16545956\beta_2-4500960
\end{align*}
and rational functions
\begin{align*}
f(\beta_2)
:=
-\frac{(\beta_2+1)^3(9\beta_2-10)^3}{P_1(\beta_2)},\\
\kappa(\beta_2)
:=
-\frac{(\beta_2+1)^4(9\beta_2-10)^4}{P_2(\beta_2)}.
\end{align*}
Let \(\rho\in(0.9,1)\) be the unique root of
$P_3(\beta_2)=0$.
(Numerically,
$\rho\approx 0.9506620267$.)
Set
\[
\lambda_{\min}:= f(\rho) \approx 0.4945366333,
\qquad
\lambda_{\max}:=f(0.9)=\frac{6859}{13497}\approx 0.5081870045,
\]
and
\[
\kappa_{\min}:=\kappa(0.9)=\frac{130321}{271013}\approx 0.4808662315.
\]

Then the monotonicity of \(C_{\mathrm{total}}(0.9,\beta_2,\lambda)\), as a function of
\(\beta_2\in[0.9,1)\), is as follows.
\begin{assertions}
\item If \(\lambda\le \lambda_{\min}\), then
\(C_{\mathrm{total}}(0.9,\beta_2,\lambda)\) is strictly decreasing on \([0.9,1)\).

\item If \(\lambda_{\min}<\lambda<1/2\), then there are unique points
\[
u_\lambda\in(0.9,\rho),
\qquad
v_\lambda\in(\rho,1)
\]
such that
\[
f(u_\lambda)=f(v_\lambda)=\lambda.
\]
Moreover,
\begin{gather*}
C_{\mathrm{total}}(0.9,\beta_2,\lambda)
\text{ is decreasing on } [0.9,u_\lambda),\\
C_{\mathrm{total}}(0.9,\beta_2,\lambda)
\text{ is increasing on } (u_\lambda,v_\lambda),\\
C_{\mathrm{total}}(0.9,\beta_2,\lambda)
\text{ is decreasing on } (v_\lambda,1).
\end{gather*}

\item If \(1/2\le \lambda<\lambda_{\max}\), then there is a unique point
\[
u_\lambda\in(0.9,\rho)
\]
such that
\[
f(u_\lambda)=\lambda.
\]
Moreover,
\begin{gather*}
C_{\mathrm{total}}(0.9,\beta_2,\lambda)
\text{ is decreasing on } [0.9,u_\lambda),
\\
C_{\mathrm{total}}(0.9,\beta_2,\lambda)
\text{ is increasing on } (u_\lambda,1).
\end{gather*}

\item If \(\lambda\ge\lambda_{\max}\), then
\(C_{\mathrm{total}}(0.9,\beta_2,\lambda)\) is increasing on \([0.9,1)\).
\end{assertions}
\end{proposition}

\begin{proof}
Define
\[
S(\beta_2)
:=
C_1(0.9,\beta_2)+C_2(0.9,\beta_2).
\]
Then
\[
C_{\mathrm{total}}(0.9,\beta_2,\lambda)
=
9-\frac{\beta_2}{1-\beta_2}
+\lambda S(\beta_2).
\]
A direct simplification and differentiation gives
\begin{align*}
S(\beta_2)
&=
-\frac{
3672\beta_2^5-8905\beta_2^4+2677\beta_2^3
+8282\beta_2^2-7428\beta_2+1710
}
{
(\beta_2-1)(\beta_2+1)^2(9\beta_2-10)^2
  },\\
S'(\beta_2)
&=
-\frac{P_1(\beta_2)}
{(\beta_2-1)^2(\beta_2+1)^3(9\beta_2-10)^3}.
\end{align*}

We first record the sign information needed below. Using Sturm's theorem,
we can see that
\(P_1\) has no root in \((0.9,1)\) and is
positive there, while \(P_2\) has no root in \((0.9,1)\) and is negative there.
Moreover, \(P_3\) has exactly one root in \((0.9,1)\), denoted by \(\rho\), and
\(P_4\) has no root in \((0.9,1)\) and is negative there.

Since
\[
(\beta_2-1)^2(\beta_2+1)^3(9\beta_2-10)^3<0
\]
on \([0.9,1)\), and since \(P_1(\beta_2)>0\), we have
\[
S'(\beta_2)>0
\qquad
\text{for } \beta_2\in[0.9,1).
\]

Now
\begin{align*}
\partial_{\beta_2} C_{\mathrm{total}}(0.9,\beta_2,\lambda)
&=
-\frac{1}{(1-\beta_2)^2}
+\lambda S'(\beta_2)
\\
&=
S'(\beta_2)
\left(
\lambda-\frac{1}{(1-\beta_2)^2S'(\beta_2)}
\right)
\\
&=
S'(\beta_2)\left(\lambda-f(\beta_2)\right).
\end{align*}
Because \(S'(\beta_2)>0\), the sign of
\[
\partial_{\beta_2} C_{\mathrm{total}}(0.9,\beta_2,\lambda)
\]
is the sign of
\[
\lambda-f(\beta_2).
\]

We next analyze \(f\). Direct differentiation gives
\[
f'(\beta_2)
=
\frac{
2(\beta_2-1)(\beta_2+1)^2(9\beta_2-10)^2P_3(\beta_2)
}
{
P_1(\beta_2)^2
}.
\]
Since \(P_3\) has a unique root \(\rho\in(0.9,1)\), with
\[
P_3(\beta_2)>0 \text{ on } [0.9,\rho),
\qquad
P_3(\beta_2)<0 \text{ on } (\rho,1),
\]
and since \(\beta_2-1<0\), we get
\[
f'(\beta_2)<0 \text{ on } [0.9,\rho),
\qquad
f'(\beta_2)>0 \text{ on } (\rho,1).
\]
Therefore \(f\) decreases on \([0.9,\rho]\) and increases on \([\rho,1)\). Furthermore,
\[
f(0.9)=\frac{6859}{13497},
\qquad
\lim_{\beta_2\to1^-}f(\beta_2)=\frac12.
\]
Thus \(f(\rho)=\lambda_{\min}\) is the minimum of \(f\), and the largest value of \(f\)
on \([0.9,1)\) is \(f(0.9)=\lambda_{\max}\).

The monotonicity classification follows from the sign rule
\[
\operatorname{sign}
\partial_{\beta_2} C_{\mathrm{total}}(0.9,\beta_2,\lambda)
=
\operatorname{sign}(\lambda-f(\beta_2)),
\]
together with the fact that \(f\) decreases from \(\lambda_{\max}\) to
\(\lambda_{\min}\), then increases from \(\lambda_{\min}\) to \(1/2\).
\end{proof}

\begin{proposition}[Monotonicity of \(C_{\mathrm{total}}(\beta_1,0.999,\lambda)\)]
\label{prop:ctotal-monotonicity-b1-full}
For \(\beta_1\in[0.9,1)\), define a polynomial
\begin{align*}
P_1(\beta_1)
:=&\
2001994001001\beta_1^3
-6011966022994\beta_1^2
+6017956024997\beta_1
-2007984005000
\end{align*}
and rational function
\[
f(\beta_1)
:=
\frac{1999(999\beta_1-1000)^3}{P_1(\beta_1)}.
\]
Set
\begin{align*}
\lambda_{\min}^{(1)}
&:=
f(0.9)
=
\frac{2053460214271}{2062434398111}
  \approx 0.9956487422,\\
\lambda_{\max}^{(1)}
&:=
\lim_{\beta_1\to1^-}f(\beta_1)
=
\frac{1999}{1996}
\approx 1.001503006.
\end{align*}
Then the monotonicity of
$
C_{\mathrm{total}}(\beta_1,0.999,\lambda)
$
as a function of \(\beta_1\in[0.9,1)\) is as follows.

\begin{assertions}
\item If \(\lambda\le \lambda_{\min}^{(1)}\), then
  $C_{\mathrm{total}}(\beta_1,0.999,\lambda)$ is increasing on \([0.9,1)\).

\item If \(\lambda_{\min}^{(1)}<\lambda<\lambda_{\max}^{(1)}\), then there is a unique point
\[
u_\lambda\in(0.9,1)
\]
such that
\[
f(u_\lambda)=\lambda.
\]
Moreover,
\[
C_{\mathrm{total}}(\beta_1,0.999,\lambda)
\text{ is decreasing on }[0.9,u_\lambda),
\]
and
\[
C_{\mathrm{total}}(\beta_1,0.999,\lambda)
\text{ is increasing on }(u_\lambda,1).
\]

\item If \(\lambda\ge \lambda_{\max}^{(1)}\), then
\(C_{\mathrm{total}}(\beta_1,0.999,\lambda)\) is strictly decreasing on \([0.9,1)\).
\end{assertions}
\end{proposition}

\begin{proof}
Write
\[
S(\beta_1)
:=
C_1(\beta_1,0.999)+C_2(\beta_1,0.999).
\]
Then
\[
C_{\mathrm{total}}(\beta_1,0.999,\lambda)
=
\frac{\beta_1}{1-\beta_1}
-\frac{0.999}{1-0.999}
+\lambda S(\beta_1) = \frac{\beta_1}{1-\beta_1}
-999
+\lambda S(\beta_1).
\]

A direct simplification and differentiation gives
\begin{align*}
S(\beta_1)
&=
\frac{
\begin{aligned}
&7973042963014998\beta_1^3
-23931084945018997\beta_1^2\\
&\qquad \qquad \qquad
+23943040997990003\beta_1
-7984999011996000
\end{aligned}
}
{
3996001(\beta_1-1)(999\beta_1-1000)^2
},
\\
S'(\beta_1)
&=
-
\frac{P_1(\beta_1)}
{1999(\beta_1-1)^2(999\beta_1-1000)^3}.
\end{align*}

We first determine the sign of \(S'(\beta_1)\) on \([0.9,1)\). Since
the denominator
\[
1999(\beta_1-1)^2(999\beta_1-1000)^3
\]
is negative for \(\beta_1\in[0.9,1)\),
\[
\operatorname{sign} S'(\beta_1)
=
\operatorname{sign} P_1(\beta_1).
\]
We claim that
\begin{equation}\label{eq:RxtLYY}
P_1(\beta_1)<0
\qquad
\text{for all }\beta_1\in[0.9,1].
\end{equation}
To see this, put
\[
\beta_1=\frac{9}{10}+\frac{t}{10},
\qquad
t\in[0,1].
\]
Then
\begin{align*}
&1000\,P_1\left(\frac{9}{10}+\frac{t}{10}\right)
  \\
&\quad =
2001994001001t^3
-6065822202913t^2
+6126260604023t
-2062434398111.
\end{align*}
Differentiating, one obtains that the right-hand side increases in $t \in [0, 1]$. Since the value at $t = 1$ is negative, \eqref{eq:RxtLYY} is true.
We have proven that
\[
S'(\beta_1)<0
\qquad
\text{for all }\beta_1\in[0.9,1).
\]

Now differentiate \(C_{\mathrm{total}}\):
\[
\partial_{\beta_1}C_{\mathrm{total}}(\beta_1,0.999,\lambda)
=
\frac{1}{(1-\beta_1)^2}
+\lambda S'(\beta_1)
=
S'(\beta_1)
\prn[\big]{
\lambda
- f(\beta_1)
},
\]
where
\[
f(\beta_1)
:=
-\frac{1}{(1-\beta_1)^2S'(\beta_1)} = \frac{1999(999\beta_1-1000)^3}{P_1(\beta_1)}.
\]
Since \(S'(\beta_1)<0\), this implies
\[
\operatorname{sign}
\partial_{\beta_1}C_{\mathrm{total}}(\beta_1,0.999,\lambda)
=
\operatorname{sign}\bigl(f(\beta_1)-\lambda\bigr).
\]

It remains to analyze \(f\). Direct differentiation gives
\[
f'(\beta_1)
=
\frac{
7992002(\beta_1-1)(999\beta_1-1000)^2(6990003\beta_1-6994000)
}
{
P_1(\beta_1)^2
}.
\]
On \([0.9,1)\), we have
\[
\beta_1-1<0,
\qquad
(999\beta_1-1000)^2>0,
\qquad
6990003\beta_1-6994000<0.
\]
Since \(P_1(\beta_1)^2>0\), it follows that
\[
f'(\beta_1)>0
\qquad
\text{for all }\beta_1\in[0.9,1).
\]
Thus \(f\) is strictly increasing on \([0.9,1)\). Moreover,
\[
f(0.9)
=
\frac{2053460214271}{2062434398111}
=
\lambda_{\min}^{(1)},
\]
and
\[
\lim_{\beta_1\to1^-}f(\beta_1)
=
\frac{1999}{1996}
=
\lambda_{\max}^{(1)}.
\]

We now translate this into monotonicity of
\(C_{\mathrm{total}}(\beta_1,0.999,\lambda)\). The sign rule is
\[
\operatorname{sign}
\partial_{\beta_1}C_{\mathrm{total}}(\beta_1,0.999,\lambda)
=
\operatorname{sign}\bigl(f(\beta_1)-\lambda\bigr).
\]

If \(\lambda\le \lambda_{\min}^{(1)}\), then
\[
f(\beta_1)-\lambda\ge0
\qquad
\text{for all }\beta_1\in[0.9,1),
\]
and the inequality is strict for \(\beta_1>0.9\). Hence
$
C_{\mathrm{total}}(\beta_1,0.999,\lambda)
$
is increasing on \([0.9,1)\), and strictly increasing on \((0.9,1)\).

If
\[
\lambda_{\min}^{(1)}<\lambda<\lambda_{\max}^{(1)},
\]
then, because \(f\) is continuous and strictly increasing, there is a unique
\[
u_\lambda\in(0.9,1)
\]
such that
\[
f(u_\lambda)=\lambda.
\]
For \(\beta_1<u_\lambda\), we have \(f(\beta_1)<\lambda\), so
\[
\partial_{\beta_1}C_{\mathrm{total}}(\beta_1,0.999,\lambda)<0.
\]
For \(\beta_1>u_\lambda\), we have \(f(\beta_1)>\lambda\), so
\[
\partial_{\beta_1}C_{\mathrm{total}}(\beta_1,0.999,\lambda)>0.
\]
Therefore
$
C_{\mathrm{total}}(\beta_1,0.999,\lambda)
$
is decreasing on \([0.9,u_\lambda)\) and increasing on \((u_\lambda,1)\).

Finally, if \(\lambda\ge \lambda_{\max}^{(1)}\), then
\[
f(\beta_1)-\lambda<0
\qquad
\text{for all }\beta_1\in[0.9,1),
\]
because \(f(\beta_1)<\lambda_{\max}^{(1)}\) for every \(\beta_1<1\). Hence
\[
\partial_{\beta_1}C_{\mathrm{total}}(\beta_1,0.999,\lambda)<0
\]
throughout \([0.9,1)\), and
$
C_{\mathrm{total}}(\beta_1,0.999,\lambda)
$
is strictly decreasing on \([0.9,1)\). This proves the classification.
\end{proof}


\end{document}